%% file: cmdp_neurips_2026.tex
\documentclass{article}

\PassOptionsToPackage{numbers,sort&compress}{natbib}
\usepackage[main,preprint]{neurips_2026}

\usepackage[utf8]{inputenc}
\usepackage[T1]{fontenc}
\usepackage{hyperref}
\usepackage{url}
\usepackage{booktabs}
\usepackage{amsfonts}
\usepackage{microtype}
\usepackage{xcolor}
\usepackage{graphicx}
\usepackage{subcaption}
\usepackage{amsmath}
\usepackage{amssymb}
\usepackage{mathtools}
\usepackage{amsthm}
\usepackage{multirow}
\usepackage[capitalize,noabbrev]{cleveref}

\newcommand{\thickhline}{\noalign{\hrule height 0.8pt}}

\usepackage{algorithm}
\usepackage{algorithmic}

\theoremstyle{plain}
\newtheorem{theorem}{Theorem}[section]
\newtheorem{proposition}[theorem]{Proposition}
\newtheorem{lemma}[theorem]{Lemma}
\newtheorem{lem}[theorem]{Lemma}
\newtheorem{corollary}[theorem]{Corollary}
\theoremstyle{definition}
\newtheorem{definition}[theorem]{Definition}

\theoremstyle{definition}
\newtheorem{assum}[theorem]{Assumption}
\newtheorem{remark}[theorem]{Remark}

\title{Primal-Dual Representation Learning for Low-Rank Constrained MDPs}

\author{
Chenhao Zhou \quad
Chao Zhang\thanks{Corresponding author.} \quad
Wensong Bai \quad
Hanbin Zhao \quad
Hui Qian\\
College of Computer Science and Technology, Zhejiang University\\
\texttt{zhouchenhao@zju.edu.cn} \quad
\texttt{zczju@zju.edu.cn}
}

\begin{document}

\maketitle

\begin{abstract}
Representation learning has substantially improved the efficiency of unconstrained continuous-state Markov Decision Processes (MDPs), but its extension to Constrained MDPs (CMDPs) is difficult because the transition representation must be learned while constraint violation is controlled. In this paper, we study representation learning for low-rank CMDPs with continuous state spaces, and propose REP-PD-TEG for the soft cumulative-violation setting. In the feature learning period, the algorithm interacts with the environment using a decaying-$\epsilon$-greedy mixture execution policy, and uses the collected samples to learn transition features by maximum likelihood. Based on the learned features, it constructs uncertainty bonuses and plans with an optimistic Lagrangian objective to obtain the return policy.
We prove regret and soft-violation guarantees for the return policies, and quantify the influence of the decaying-$\epsilon$-greedy exploration. In particular, REP-PD-TEG achieves the optimal \(\widetilde O(\sqrt K)\) return policy bound, and gives a \(\widetilde O(K^{3/4})\) overall sample bound when exploratory actions are charged. For the stricter hard cumulative violation, where feasibility must be verified episode by episode, we propose REP-PD-H-TEG. It plans under a bonus-corrected lower-confidence utility certificate with the learned representation, and performs a stabilized dual search to obtain the return policy, yielding regret and hard-violation bounds. Together, these results give the first theoretically guaranteed representation-learning method for low-rank continuous-state CMDPs with both soft cumulative guarantees and certified hard-violation control.
\end{abstract}

\input{content/intro.tex}
\input{content/related.tex}
\input{content/background.tex}
\input{content/soft.tex}
\input{content/hard.tex}
\input{content/conclusion.tex}

\begin{ack}
\paragraph{Funding.} There was no funding in direct support of this work.

\paragraph{Competing interests.} The authors declare no competing interests.
\end{ack}

\bibliographystyle{plainnat}
\bibliography{references}

\clearpage
\appendix
\section*{Appendix Roadmap}
The appendix is organized as follows. Appendix~\ref{app:soft_proof} fixes
notation and records the basic covariance concentration tool. Appendix~\ref{app:aux_lemmas}
collects the MLE, elliptical-potential, simulation, and primal-dual lemmas used
throughout the proofs. Appendix~\ref{app:hard_proof} states the technical
ingredients specific to the hard-violation analysis. Appendix~\ref{app:general_probing}
proves the tertiary-aware decaying-\(\epsilon\) representation-learning result
for REP-PD-TEG, while Appendix~\ref{app:uniform_repaired} gives the \(p=0\)
uniform specialization. Appendix~\ref{app:hard_general_probing} derives the
one-sided hard certificate and proves the decaying-schedule result, while
Appendix~\ref{app:hard_tertiary_roadmap} records its \(p=0\) specialization.
Appendix~\ref{sec:exp} contains the
experimental details.
\input{appendix/soft_proof.tex}
\input{appendix/lemmas.tex}
\input{appendix/hard_proof.tex}
\input{appendix/new_proof.tex}
\input{appendix/uniform_proof.tex}
\input{appendix/hard_general_probing.tex}
\input{appendix/hard_tertiary_roadmap.tex}
\input{content/experiment.tex}

\end{document}

%% file: content/intro.tex
\section{Introduction}
\label{sec:intro}
Constrained Markov Decision Processes (CMDPs) provide a principled framework for constrained reinforcement learning with both reward objectives and safety requirements \cite{ji2023ai,altman2021constrained}. This framework is especially relevant in safety-critical domains, and a growing line of provably constrained RL algorithms has been developed within this framework \cite{efroni2020exploration,singh2020learning,brantley2020constrained,kalagarla2021sample,liu2021learning,bura2022dope}.
These works primarily focus on the tabular setting, where guarantees scale with the cardinality of the state space and become unsuitable for continuous or high-dimensional environments. Linear CMDP methods \cite{amani2021safe,ghosh2022provably,ghosh2024towards,wei2024safe,kitamura2025provably,ding2021provably} partially address this issue by assuming that transitions are linearly represented by known low-dimensional features, yielding state-dimension-independent bounds. However, existing tabular/linear CMDP methods all rely on a priori knowledge of the latent feature representations. In the absence of such prior feature knowledge, their guarantees no longer apply, leaving CMDPs with unknown features insufficiently addressed.

% Existing approaches fracture this challenge—either assuming known low-dimensional features \cite{ghosh2022provably,wei2024safe} or restricting to tabular settings \cite{liu2021learning}—leaving a critical gap for safe RL in unknown, complex environments.
%改成和前面更搭的
Representation learning has been effective for discovering low-dimensional transition features in unconstrained MDPs \cite{agarwal2020flambe,uehara2021representation,ren2022latent,ren2022spectral,zhang2023provable,modi2024model,ni2022representation}. Extending this idea to CMDPs is more challenging because representation error now affects both reward optimization and constraint evaluation. The learner must identify latent dynamics from online data, use the learned representation to guide exploration, and still control cumulative constraint violation \cite{agarwal2020flambe,efroni2020exploration}. This creates a coupling between representation coverage and constraint-aware policy optimization that is absent from unconstrained low-rank RL and is not handled by existing CMDP results with known features.

In this paper, we develop primal-dual representation-learning algorithms for low-rank continuous-state CMDPs with unknown transition features. 
During feature learning, the algorithms interact with the environment via an $\epsilon$-greedy mixture \emph{execution policy} and use the collected samples to learn transition features by maximum-likelihood estimation (MLE).
For soft violation criterion, where only the total constraint shortfall across episodes is charged, our algorithm, REP-PD-TEG, uses the learned features to construct uncertainty bonuses, and plan with the bonus-augmented Lagrangian objective that couples reward maximization, utility optimization, and exploration.
The algorithm outputs the \emph{return policies} with a primal-dual update, and admits both return-policy and execution-policy guarantees. For the stricter per-episode hard-violation metric, where each episode's shortfall is charged separately, we develop REP-PD-H-TEG, which keeps the representation-learning procedure but plans under a tight lower-confidence utility certificate and conducts a stabilized dual search to obtain the return policy.
Specifically, we make the following contributions:
\begin{itemize}
    \item We propose the first theoretically guaranteed representation-learning method for low-rank continuous-state CMDPs with unknown features. Starting from the occupancy of the current return policy, it uses $\epsilon$-greedy auxiliary probes for action coverage, fits transition features by MLE, and constructs a merged-covariance uncertainty bonus. We quantify the effect of decaying-$\epsilon$ exploration for both return and execution policies.
    \item For the aggregate soft-violation criterion,
    we propose REP-PD-TEG, which updates the return policy through an optimistic Lagrangian whose reward and dual-weighted utility terms are both bonus-augmented based on the learned features.    
    We show that REP-PD-TEG achieves return policy regret and soft violation of order \(\tilde O(K^{1/2+p})\) sample complexity under the schedule \(\epsilon_k=k^{-p}\), where \(\epsilon_k\) is the exploration level in episode \(k\). 
    %random-exploration probability这个词和前面对齐 说清楚为什么可以p=0，return policy不考虑epsilon的影响
    In particular, this recovers the optimal \(\tilde O(\sqrt K)\) return policy rate when \(p=0\), which corresponds to uniform auxiliary probing over actions during the whole procedure. We also prove an execution policy overall bound: executing the $\epsilon_k$-greedy exploration adds \(O(K^{1-p})\) complexity, and the balanced choice \(p=1/4\) yields a \(\tilde O(K^{3/4})\) execution policy regret and soft violation up to problem-dependent factors.%前后名词对齐
    \item For the stricter per-episode hard-violation criterion,
    we propose REP-PD-H-TEG, which uses the learned features to conduct a stabilized dual search and certifies feasibility through a bonus-corrected lower-confidence utility certificate.
    For return policies, we obtain regret of order \(\tilde O(K^{3/4+p})\) and hard violation of order \(\tilde O(K^{1/2+p})\); taking \(p=0\) again gives  uniform probing and the clean \(\tilde O(\sqrt K)\) hard-violation dependence. For execution policies, we show that the cost of $\epsilon_k$-greedy exploration is also \(O(K^{1-p})\) and choosing \(p=1/8\) gives a \(\tilde O(K^{7/8})\) regret and  hard violation up to problem-dependent factors.
\end{itemize}
Appendix~\ref{sec:exp} provides the empirical study of REP-PD-TEG on a simulated job-scheduling task. Besides, we also implement an approximate neural network  version of REP-PD-TEG on the Safety Gym task to validate the effectiveness of our representation learning procedure for a complex state-action space.

\begin{table*}[tb]
   
    \centering
    \setlength{\tabcolsep}{2pt}
     \caption{
    Comparison of soft-violation bounds. The two REP-PD-TEG rows report the complexity of the best return policy specialization and the best execution policy specialization, respectively.
    }\label{table:algorithms}
    \begin{tabular}{|c|c|c|c|c|}
    \thickhline
    MDP &Method & Known Feature & Regret & Violation \\\hline
    \multirow[c]{2}{*}{Tabular} 
    &\citet{efroni2020exploration} & Yes   & $\tilde{O}\left( \sqrt{|\mathcal S|^3 |\mathcal A| H^4 K}\right)$ & $\tilde{O}\left( \sqrt{|\mathcal S|^3 |\mathcal A| H^4 K}\right)$   \\
    \cline{2-5}
    & \citet{liu2021learning} & Yes   & $\tilde{O}\left( \sqrt{|\mathcal S|^3 |\mathcal A| H^6 K}\right)$ & 0   \\  \thickhline
    \multirow[c]{3}{*}[-1.3ex]{Linear} 
    & \citet{ding2021provably} & Yes  &$\tilde{O}\left(\sqrt{d^2 H^6 K}\right)$  & $\tilde{O}\left(\sqrt{d^2 H^6 K}\right)$  \\ 
    \cline{2-5}
    & \citet{ghosh2022provably} & Yes  &$\tilde{O}\left(\sqrt{d^3 H^3 K}\right)$  & $\tilde{O}\left(\sqrt{d^3 H^3 K}\right)$  \\ 
    \cline{2-5}
    & \citet{ghosh2022provably} & Yes  &$\tilde{O}\left(\sqrt{d^3 H^5 K}\right)$  & 0  \\ 
    \thickhline
    \multirow[c]{2}{*}[-1.3ex]{Low-rank}
       &\multirow[c]{2}{*}[-1.3ex]{\textbf{REP-PD-TEG}}
       &\multirow[c]{2}{*}[-1.3ex]{\textbf{No}}
       & \begin{tabular}[c]{@{}c@{}}
         $\displaystyle\tilde O\!\left(\!|\mathcal A| d^2\sqrt K(1-\gamma)^{-\!2}\!\right)$\\
         \textcolor{blue}{[return policy]}
         \end{tabular}
       & \begin{tabular}[c]{@{}c@{}}
         $\displaystyle\tilde O\!\left(\!|\mathcal A| d^2\sqrt K(1-\gamma)^{-\!2}\!\right)$\\
         \textcolor{blue}{[return policy]}
         \end{tabular}\\
       \cline{4-5}
       &
       &
       & \begin{tabular}[c]{@{}c@{}}
         $\displaystyle\tilde O\!\left(\!|\mathcal A| d^2K^{3/4}(1-\gamma)^{-\!2}\!\right)$\\
         \textcolor{blue}{[execution policy]}
         \end{tabular}
       & \begin{tabular}[c]{@{}c@{}}
         $\displaystyle\tilde O\!\left(\!|\mathcal A| d^2K^{3/4}(1\!-\!\gamma)^{-\!2}\!\right)$\\
         \textcolor{blue}{[execution policy]}
         \end{tabular}\\
    \thickhline
    \end{tabular}
    \end{table*}

\begin{table*}[tb]
    \setlength{\tabcolsep}{3pt}
    \centering
    \small
    \caption{
    Comparison of hard-violation bounds. The two REP-PD-H-TEG rows report the complexity of the best return policy  specialization and the balanced execution policy  specialization, respectively.
    }\label{table:hard_algorithms}
    \begin{tabular}{|c|c|c|c|c|c|}
    \thickhline
    MDP &Method & \shortstack{Known\\Feature} & \shortstack{Known Safe\\Action Set} & Regret & Violation \\\hline
    \multirow{2}{*}{Tabular}
    & \citet{efroni2020exploration} & Yes  & Yes & $\tilde{O}\left( \sqrt{|\mathcal S|^3|\mathcal A|H^4K}\right)$ &  $\tilde{O}\left( \sqrt{|\mathcal S|^3|\mathcal A|H^4K}\right)$  \\
    \cline{2-6}
    & \citet{bura2022dope} & Yes  & Yes & $\tilde{O}\left( \sqrt{|\mathcal S|^2|\mathcal A|H^6K}\right)$ &  0  \\ \thickhline
    \multirow[c]{4}{*}[-1.3ex]{Linear}
    & \citet{amani2021safe} & Yes & Yes & $\tilde{O}\left(\sqrt{d^3 H^4 K}\right)$ &  0  \\
    
    \cline{2-6}
    &\citet{kitamura2025provably} & Yes & Yes & $\tilde{O}\left( \sqrt{d^5 H^8 K}\right)$ & 0 \\

    \cline{2-6}
    & \citet{wei2024safe} & Yes  & No & $\tilde{O}\left(\sqrt{d^3 H^4 K}\right)$ & $\tilde{O}\left(\sqrt{dH^2 K}\right)$  \\
    \cline{2-6}
    & \citet{ghosh2024towards} & Yes & No & $\tilde{O}\left(\sqrt{d^3 H^4 K}\right)$ & $\tilde{O}\left(\sqrt{d^3 H^4 K}\right)$  \\
    \thickhline
    \multirow[c]{2}{*}[-1.4ex]{Low-rank}
    &\multirow[c]{2}{*}[-1.4ex]{\textbf{REP-PD-H-TEG}}
    & \multirow[c]{2}{*}[-1.4ex]{\textbf{No}}
    & \multirow[c]{2}{*}[-1.4ex]{\textbf{No}}
    & \begin{tabular}[c]{@{}c@{}}
      $\displaystyle\tilde{O}\!\left(\!|\mathcal A| d^2\!K^{3/4}\!(1\!-\!\gamma)^{-\!2}\!\right)$\\
      \textcolor{blue}{[return policy]}
      \end{tabular}
    & \begin{tabular}[c]{@{}c@{}}
      $\displaystyle\tilde O\!\left(\!|\mathcal A| d^2\!\sqrt K(1\!-\!\gamma)^{-\!2}\!\right)$\\
      \textcolor{blue}{[return policy]}
      \end{tabular}\\
    \cline{5-6}
    &
    &
    &
    & \begin{tabular}[c]{@{}c@{}}
      $\displaystyle\tilde{O}\!\left(\!|\mathcal A|d^2\!K^{7/8}\!(1\!-\!\gamma)^{-\!2}\!\right)$\\
      \textcolor{blue}{[execution policy]}
      \end{tabular}
    & \begin{tabular}[c]{@{}c@{}}
      $\displaystyle\tilde O\!\left(\!|\mathcal A| d^2\!K^{7/8}\!(1\!-\!\gamma)^{-\!2}\!\right)$\\
      \textcolor{blue}{[execution policy]}
      \end{tabular}\\
    \thickhline
    \end{tabular}
\end{table*}

\paragraph{Remark on the probing schedule.}
In tabular and linear CMDPs, the state or feature representation is known, so
the analyzed policy is the execution policy. In low-rank CMDPs, the transition
features must be learned online, and auxiliary $\epsilon_k$-greedy probing is needed to obtain
coverage over the state-action space for representation learning. This creates a distinction between the
return-policy sequence and the execution policy. The latter is especially relevant in constrained RL because exploratory behavior can incur real training-time costs---for example, robot collisions or excessive contact forces---even when the final return policy is safe.
The two REP-PD-TEG rows in Table~\ref{table:algorithms} report, respectively, the sample complexity of the best return policy specialization \((p=0)\) and the best execution policy specialization \((p=1/4)\), where $p$ denotes the probing strength with $\epsilon_k = k^{-p}$. Similarly, the two REP-PD-H-TEG rows in Table~\ref{table:hard_algorithms} show the best return policy complexity with \(p=0\) and the balanced execution policy complexity under the hard specialization \(p=1/8\), respectively.

%% file: content/related.tex
\subsection{Related Work}
\paragraph{Provably Constrained RL}
CMDPs provide a rigorous framework for safe RL, balancing reward maximization with cumulative cost constraints \cite{altman2021constrained,gu2024review}. 
Several practical constrained RL methods have been developed on this foundation, including policy-gradient, Lagrangian, shielding, and model-based approaches \cite{achiam2017constrained,ray2019benchmarking,as2022constrained,dai2023augmented,huang2023safedreamer,banerjee2024dynamic}.
On the theoretical side, there has been a growing body of work on provably safe reinforcement learning. Many works \cite{efroni2020exploration,singh2020learning,brantley2020constrained,kalagarla2021sample,liu2021learning,bura2022dope,ding2023last,kitamura2024policy} consider the tabular setting. Among the best known regret and constraint-violation guarantees are of order \(\tilde O(\sqrt{K})\) \cite{efroni2020exploration,liu2021learning}. However, their polynomial dependence on the cardinality of state space $|\mathcal{S}|$ makes them unsuitable for large or continuous state spaces.

Recent advances in linear CMDPs \cite{neu2020unifying,amani2021safe,yang2021reduction,ding2021provably,ghosh2022provably,ghosh2024towards,wei2024safe,kitamura2025provably} achieve dimension-independent sample complexity by assuming known low-dimensional features \(\phi^\star(s,a)\).
They enforce safety through confidence intervals, optimistic bonuses, primal-dual updates, pessimistic constraints, or safe baselines.
However, their reliance on predefined latent features, assumed to be known a priori or handcrafted, limits applicability when transition structures are unknown and must be learned from data.  Our setting instead has unknown latent features that must be learned online. To our knowledge, the only prior work on low-rank CMDPs is \cite{huang2022safe}, which studies the sample complexity of identifying a near-optimal feasible policy after training. In contrast, our work analyzes the entire learning trajectory and provides regret and constraint-violation guarantees.

\paragraph{Representation Learning for Unconstrained MDPs}
A growing body of research has focused on theoretically grounded representation learning in RL, aiming to derive low-dimensional features that efficiently capture environment dynamics.  For example, works such as \cite{du2019provably, misra2020kinematic, zhang2022efficient} investigated representation learning in block MDPs, leveraging hierarchical or modular state decompositions to improve sample efficiency. \cite{ren2022free} studied representations under a Gaussian noise model, establishing guarantees for feature learning in noisy environments.
Low-Rank MDPs have emerged as a pivotal framework for representation learning in structured environments. Within this regime, model-based approaches \cite{agarwal2020flambe, uehara2021representation, zhang2022making, ren2022latent, zhang2023provable} learn latent transition features by leveraging model classes of transition probabilities, achieving sample-efficient exploration under realizability assumptions. In contrast, model-free methods like \cite{modi2024model} eschew explicit transition models and show efficiency under minimal reachability conditions.
The framework has recently expanded into reward-free exploration \cite{huang2022safe,cheng2023improved} and multi-agent Markov games \cite{ni2022representation, qiu2022contrastive, zhou2024solving, zhou2026toward}.
However, these works do not address online safety constraints in CMDPs. 
Consequently, this work gives the first representation-learning framework for low-rank CMDPs that bridges feature discovery and provable online safety.

%% file: content/background.tex
\section{Background}
We consider an episodic constrained MDP, denoted by $(\mathcal{S},\mathcal{A},P^*,r,g,\gamma,d_1)$ where $\mathcal{S}$ is the state space, $\mathcal{A}$ is the action space, $P^*$ is the transition function, $r$ is the reward function, $g$ is the utility function, $\gamma$ is the discounted factor, and $d_1$ is the initial distribution. Following prior work \cite{jiang2017contextual,sun2019model,uehara2021representation}, we assume trajectory reward is normalized, i.e., for any trajectory $\{s_h,a_h\}_{h=1}^\infty$, we have $\sum_{h=1}^\infty\gamma^hr(s_h,a_h)\in[0,1]$ and $\sum_{h=1}^\infty\gamma^hg(s_h,a_h)\in[0,1]$.%check这个定义

In an infinite-horizon discounted CMDP, each episode starts with state $s_1\sim d_1$. Then at each step $h$, the agent observes state $s_h\in \mathcal{S}$, picks an action $a_h\in \mathcal{A}$, and receives a reward $r(s_h,a_h)$ and a utility $g(s_h,a_h)$. The MDP then transitions to $s_{h+1}$ which is drawn from $P^*(\cdot|s_h,a_h)$. In this paper, we assume access to reward and utility function feedback, but does not require access to the transition model.
%The policy-space of the agent is denoted by $\Delta (\mathcal{A}|\mathcal{S})$, where $\Delta(\mathcal{A})$ represents the probability simplex over the action space. 
For any $s\in \mathcal{S}$, $\pi(a|s)$ denotes the probability of selecting action $a\in \mathcal{A}$  when the state is $s$ under policy $\pi$. 
Let $d_{P^*}^\pi(s)$ be the state occupancy measure
\begin{align*}
    d_{P^*}^\pi(s)& =(1-\gamma)\sum_{h=1}^\infty \gamma^{h-1} P^\star_\pi(s_h=s|s_1\sim d_1),
\end{align*}
where $P^*_\pi(s_h=s|s_1\sim d_1)$ is the probability that $s_h=s$, after starting at state $s_1\sim d_1$ and following $\pi$ thereafter.
Also, let $d_{P^*}^\pi(s,a)=\pi(a|s)d_{P^*}^\pi(s)$ be the state-action occupation measure.
We define the expected value function for the discounted reward under policy $\pi$ starting from state $s$ as:
\begin{align}\label{eqn:vr}
   V^{\pi}_{P^*,r}(s)= \mathbb{E}_{P^*,\pi}\left[\sum_{h=1}^{\infty}\gamma^{h}r(s_h,a_h)|s_1=s\right],
\end{align}
where $\mathbb{E}$ is taken with respect to the policy $\pi$ and the transition function $P^*$. Similarly, we define the action-value function for the reward as:
\begin{align}\label{eqn:qr}
    Q^{\pi}_{P^*,r}(s,a)=\mathbb{E}_{P^*,\pi}\left[\sum_{h=1}^{\infty}\gamma^{h}r(s_h,a_h)|s_1=s,a_1=a\right].
\end{align}
For simplicity, when the context is clear, we sometimes abbreviate $V^{\pi}_{P^{\star}, r}$ and $Q^{\pi}_{P^{\star}, r}$ as $V^{\pi}_{r}$ and $Q^{\pi}_{r}$, respectively. We also define the value function for the utility $V^{\pi}_{g}(s)$ and the action-value function for the utility $Q^{\pi}_{g}(s,a)$ analogously by replacing $r$ with $g$ in \eqref{eqn:vr} and \eqref{eqn:qr}, respectively.  We denote $V^{\pi}_j(s)$, and $Q^{\pi}_{j}(s,a)$ for $j=r,g$. 
\begin{definition}
For brevity, we denote $P^*V^{\pi}_{j}(s,a)=\mathbb{E}_{s^{\prime}\sim P^*(\cdot|s,a)}V^{\pi}_{j}(s^{\prime})$ for $j=r,g$.
\end{definition}
Using this notation, the Bellman's equation associated with the policy $\pi$ becomes
\begin{align*}
    Q^{\pi}_{j}(s,a)=(j+\gamma P^*V^{\pi}_{j})(s,a),\qquad j\in\{r,g\}.
\end{align*}
Note that $V^{\pi}_{j}(s)=\langle\pi(\cdot|s),Q^{\pi}_{j}(s,\cdot)\rangle_{\mathcal{A}}$, where $\langle\pi(\cdot|s),Q^{\pi}_{j}(s,\cdot)\rangle_{\mathcal{A}}=\sum_{a\in \mathcal{A}}\pi_h(a|s)Q^{\pi}_{j}(s,a)$.

The objective of the learning agent is to find an optimal policy that maximizes the expected reward while satisfying a safety constraint on the expected utility:
\begin{align}\label{eq:cmdp}
& \text{max}_{\pi} V^{\pi}_{r}(s_1),\quad
& \text{subject to } V^{\pi}_{g}(s_1)\geq b.
\end{align}
where $b\in (0,1]$ is a given threshold.
While we focus on a single constraint for simplicity, our method  naturally extends to handle multiple constraints. We denote the optimal policy as $\pi^{*}$ which solves the above optimization problem. %Since $\pi^{*}$ is obtained by having complete information, it is also denoted as {\em the best policy in the hindsight}.

In the absence of prior knowledge about the constraint, the agent must balance exploration and exploitation while ensuring that the policy satisfies the constraint. To this end, we allow the policy to occasionally violate the constraint and aim to minimize the regret and the total constraint violations over $K$ episodes. This approach is consistent with existing literature on constrained RL \cite{efroni2020exploration,ding2021provably,ding2020natural}. 

To formally measure the performance of the policy sequence $\{\pi_k\}_1^K$, we define the regret as $$\mathrm{Regret}(K)=\sum_{k=1}^{K}V^{\pi^*}_{r}(s_1)-V^{\pi_k}_{r}(s_1),$$ and the soft/hard constraint violations  as 
\begin{align*}
\mathrm{Violation}_{\mathrm{soft}}(K)&=\left[\sum_{k=1}^{K}(b-V^{\pi_k}_{g}(s_1))\right]_{+},
\mathrm{Violation}_{\mathrm{hard}}(K)
&=\sum_{k=1}^{K}\left[b-V^{\pi_k}_{g}(s_1)\right]_{+},
\end{align*}
where $[z]_{+}=\max\{z,0\}$. The regret measures the gap between the total reward value by following the optimal policy $\pi^{*}$, and the total reward value obtained by following agent's policy $\pi_k$ at episode $k$ over $K$ episodes. The soft constraint violation quantifies the cumulative difference between the desired constraint threshold and the achieved utility, allowing for both positive and negative deviations over time. In contrast, the hard constraint violation is a stricter metric that measures only the cumulative positive deviations. As shown in \cite{ghosh2024towards}, even zero soft violation may lead to hard constraint violation that grows linearly with the number of episodes, while soft constraint violation can only grow sub-linearly with the number of episodes if the hard constraint violation grows sub-linearly.

These metrics capture cumulative notions of safety rather than strict zero-violation safe exploration. In particular, soft violation is the standard online-CMDP notion for budget-style or long-run average constraints, while hard violation is a stricter cumulative metric that counts only positive deviations. Throughout the paper, our guarantees concern cumulative violation control; they do not imply per-episode safety during learning.

\textbf{Dual problem and Slater's Condition.}
We then introduce some additional notation that will be used throughout this paper.
\begin{definition}\label{defn:lagrange_value}
For any policy $\pi$ and Lagrangian multiplier $Y$, we define the composite value function and Q-function as $V^{\pi,Y}(s)=V^{\pi}_{r}(s)+Y V_g^{\pi}(s)$ and $Q^{\pi,Y}(s,a)=Q^{\pi}_{r}(s,a)+YQ^{\pi}_{g}(s,a)$.
% \begin{align*}
% V^{\pi,Y}(s)=V^{\pi}_{r}(s)+Y V_g^{\pi}(s), \qquad \text{and}\qquad
%     Q^{\pi,Y}(s,a)=Q^{\pi}_{r}(s,a)+YQ^{\pi}_{g}(s,a).
% \end{align*}
\end{definition}

We can cast problem (\ref{eq:cmdp}) as a saddle point problem $\max_{\pi}\min_{Y}\mathcal{L}(\pi,Y)$, with Lagrangian $\mathcal{L}(\pi,Y)=V_{r}^{\pi}(s_1)+Y(V_{g}^{\pi}(s_1)-b)=V^{\pi,Y}-Yb$. Although the Lagrangian is non-concave in $\pi$ \cite{agarwal2021theory}, the strong duality holds \cite{paternain2019safe}. Hence, there exists an optimal Lagrangian multiplier $Y^{*}$ such that $\max_{\pi}\mathcal{L}(\pi,Y^*)$  recovers the optimal reward value function. %Thus, we focus on obtaining the optimal dual value and the optimal composite value function $V^{\pi,Y^{*}}_1(\cdot)$ corresponding to the dual variable $Y^{*}$.

We assume the following Slater's condition in this paper.
\begin{assum}[Slater's Condition]\label{assum:slater}
There exists $\theta>0$, and a policy $\bar{\pi}$, such that $V^{\bar{\pi}}_{g}(s_1)\geq b+\theta$.
\end{assum}

%\xingyu{$\xi$ should be $2H/\gamma$ since you cannot know the V value and we only can use $H$ to bound it. This is the same as Ding's paper, and hence the result should have an additional $H$. Please update the main results and the table.}

The Slater's condition is mild in practice and widely used in prior work \cite{ding2021provably,efroni2020exploration,qiu2020upper}; it also implies that the optimal dual variable is bounded.
\begin{lem}[Boundedness of $Y^{*}$]\label{lem:boundY}
The optimal dual-variable $Y^{*}\leq \dfrac{V^{\pi^*}_{r}(s_1)-V^{\bar{\pi}}_{r}(s_1)}{\theta}\leq \dfrac{1}{\theta}$. 
\end{lem}
% Now, we introduce a notation which we use throughout this section.
% \begin{definition}
% We set $\zeta=\min\{\gamma/2,\dfrac{\tilde{\cO}(\sqrt{d^3HT})}{K}\}$
% \end{definition}

\textbf{Low-rank CMDP.} In this paper, we study low-rank CMDPs,  which are characterized by a low-rank decomposition of the transition function. Specifically, we assume the following:
\begin{assum}[low-rank CMDP]\label{assum:linearmdp}
The transition function of the CMDP $P^*:\mathcal{S}\times\mathcal{A}\to\Delta(\mathcal{S})$ admits a low rank decomposition with rank $d\in\mathbb{N}$ if there exists two embedding functions $\phi^*:\mathcal{S}\times\mathcal{A}\to\mathbb{R}^d,\mu^*:\mathcal{S}\to\mathbb{R}^d$ such that
\begin{align*}
\forall s,s^\prime\in\mathcal{S},a\in\mathcal{A}:P^*(s^{\prime}|s,a)=\langle\phi^*(s,a),\mu^*(s^{\prime})\rangle.
\end{align*}
For normalization, we assume $\|\phi^*(s,a)\|\leq1$ for all $(s,a)$ and for any functions $f:\mathcal{S}\to[0,1]$, $\|\int \mu^*(s)f(s) \,ds \|_2\leq\sqrt{d}$. 
\end{assum}
\begin{remark}
    The low-rank CMDP assumption, characterized by the existence of low-dimensional embedding functions $\phi^*(s,a)$ and $\mu^*(s^\prime)$, is widely adopted structural assumption in RL \cite{agarwal2020flambe,uehara2021representation,ren2022spectral}.  This assumption encompasses several common settings as special cases, including tabular CMDPs, linear CMDPs, and block CMDPs.
\end{remark}

\textbf{Stochastic-oracle feedback model.}
Following prior low-rank representation-learning work, we study an occupancy-based stochastic-oracle setting. The learner may query states distributed as $d_{P^\star}^{\pi}$ together with auxiliary transitions that are used only for representation learning and bonus construction. Practically, a discounted-occupancy state can be sampled by resetting to the initial distribution and rolling in under \(\pi\) for a geometrically distributed number of steps. Here \(K\) counts outer oracle rounds, not raw environment transitions; implementing \(K\) queries this way uses \(O(K/(1-\gamma)+3K)\) transitions in expectation. Our regret and violation bounds do not charge the complete roll-in trajectories. Likewise, the execution-policy corollaries transfer values to the stationary $\epsilon$-mixture policy rather than giving an end-to-end guarantee for the entire three-transition probing trajectory.

\textbf{Function approximation setup and MLE oracle.}
Since the transition function $P^*$ is unknown, we employ a function class $\mathcal{M}=\{(\mu,\phi):\mu\in\Theta ,\phi\in\Gamma \}$ to characterize $\mu^\star$ and $\phi^\star$, and by extension, to approximate $P^*$. Our formal guarantees use a finite realizable class, $|\mathcal M|<\infty$ and $P^\star\in\mathcal M$; this restricts the candidate representations rather than the cardinality of the state space. Finite realizability and logarithmic dependence on $|\mathcal M|$ are standard in low-rank representation-learning analyses \cite{agarwal2020flambe,uehara2021representation}.
To select the appropriate transition $\hat{P}$ from the function class $\mathcal{M}$, we assume access to a supervised learning-style maximum-likelihood estimation (MLE) oracle.

\begin{definition}[MLE Oracle]
Given a model class $\mathcal{M}$ and a dataset $\mathcal{D}$ in the form of $(s, a, s^\prime)$, the MLE oracle returns the maximum likelihood estimator $\hat{P}=(\hat{\mu},\hat{\phi})=\arg\max_{(\mu,\phi)\in\mathcal{M}}\mathbb{E}_{\mathcal{D}}\ln(\mu(s^\prime)^\top\phi(s,a))$.
\end{definition}

% We want to remark that the planning procedure for known linear MDPs with potentially nonlinear rewards, which can be done in polynomial time \cite{jin2020provably}. (See Section 5.1 in \cite{agarwal2020flambe} for details of how to implement this procedure for a known linear MDP with polynomial computation complexity.)

%% file: content/soft.tex
\section{Soft Violation via Cumulative Primal-Dual Control}
\label{sec:rep-pd}
\label{sec:online}

We first study low-rank CMDPs under the soft cumulative-violation metric, where
constraint deficits and surpluses are aggregated before taking the positive
part. The learner must identify an unknown low-rank transition representation
while updating policies through a primal-dual procedure: return-policy data may
lack action coverage, whereas always-uniform actions are less policy-adaptive.
REP-PD-TEG (\textbf{Rep}resentation-learning
\textbf{P}rimal-\textbf{D}ual with \textbf{T}ertiary decaying-\(\boldsymbol{\epsilon}\)-
\textbf{G}reedy probing) resolves this tension by keeping
the probing kernel close to the previous return policy while retaining a
uniform action floor; the tertiary transition supplies the coverage needed for
optimism. Unless stated otherwise, regret and soft violation are evaluated on
the return policies \(\{\pi_k\}_{k=1}^K\); auxiliary tuples are used only for
model estimation and bonus construction. We later quantify the cost of charging
the execution policy.

\subsection{REP-PD-TEG Algorithm}
\label{subsec:rep-pd-decay}

Algorithm~\ref{alg:rep-pd-teg} gives REP-PD-TEG.
At episode \(k\), the algorithm interacts with the environment using the execution policy
\(q_k(\cdot\mid s)=q'_k(\cdot\mid s)=q''_k(\cdot\mid s)
=(1-\epsilon_k)\pi_{k-1}(\cdot\mid s)+\epsilon_k U(\mathcal A)\),
where \(U(\mathcal A)\) is the uniform distribution over the action space and
\(\epsilon_1\ge \epsilon_2\ge \cdots \ge \epsilon_K>0\).
The oracle returns one tertiary tuple
\((s,a,s',a',s'',a'',\bar s)\): it first samples
\(s\sim d^{\pi_{k-1}}_{P^\star}\), then alternates actions from
\(q_k,q'_k,q''_k\) with transitions from \(P^\star\).
The resulting three transitions are stored in separate replay buffers denoted as $D_k$, $D'_k$ and $D''_k$. 
%In our proposed method, MLE uses all the three stream datasets to learn the feature,
%whereas the covariance mixture bonus uses a merged primary/secondary covariance;
%the tertiary stream supplies the next-state coverage needed in the proof.
After updating the replay buffers, REP-PD-TEG fits
\(\widehat P_k=(\widehat\mu_k,\widehat\phi_k)\) by MLE on $D_k+D'_k+D''_k$.
The third stream is required by the present \emph{merged-covariance} design, not by MLE alone. Since the covariance contains both $D_k$ and $D'_k$, the one-step-back proof produces propagated quadratic terms under both primary and secondary designs: $D'_k$ controls the former, while the one-step pushforward $D''_k$ supplies an MLE-controlled denominator for the latter. We therefore do not place $D''_k$ in the covariance, which would create another propagated level.
The learned features are then utilized to calculate the empirical covariance matrix $\widehat\Sigma^{\mathrm{mix}}_k$, which quantifies the uncertainty associated with these features. Based on the empirical covariance matrix, an exploration bonus, $\widehat b_k(x,u)$, is
constructed as a scaled Mahalanobis norm of the learned feature \(\widehat\phi_k\) with
respect to $\widehat\Sigma^{\mathrm{mix}}_k$.
This bonus induces uncertainty-aware planning under the learned model by augmenting both the reward and utility value estimates, after which the return policy and dual variable are updated according to the final two steps of Algorithm~\ref{alg:rep-pd-teg}.

The dual variable is projected into $[0,\xi]$ and we set $\xi=2/\theta$ in the analysis.
Reward optimism promotes exploration, while utility optimism is valid
for soft violation because episode-level deficits and surpluses are aggregated.
Thus one episode consists of auxiliary tertiary sampling, representation
learning, bonus construction, optimistic primal update, and projected dual
update.

\begin{algorithm}[t]
\caption{REP-PD-TEG}
\label{alg:rep-pd-teg}
\label{alg:online_teg}
\label{alg:online_algorithm}
\begin{algorithmic}[1]
\STATE {\bf Input:} Parameters \(\{\lambda_k\}_{k=1}^K\), \(\{\alpha_k\}_{k=1}^K\), \(\eta\), \(\xi\), \(\{\epsilon_k\}_{k=1}^K\), model class \(\mathcal M=\{(\mu,\phi):\mu\in\Theta,\phi\in\Gamma\}\), number of episodes \(K\).
\STATE Initialize return policy \(\pi_0(\cdot\mid s)=U(\mathcal A)\), \(Y_1=0\), and \(D_0=D'_0=D''_0=\emptyset\).
\FOR{\(k=1,\ldots,K\)}
    \STATE Set execution policy $q_k=q'_k=q''_k=(1-\epsilon_k)\pi_{k-1}+\epsilon_k U(\mathcal A)$.
    \STATE Query one auxiliary tertiary tuple \((s,a,s',a',s'',a'',\bar s)\), where $
    s\sim d^{\pi_{k-1}}_{P^\star}$, $a\sim q_k(\cdot\mid s)$, $s'\sim P^\star(\cdot\mid s,a)$, $a'\sim q'_k(\cdot\mid s')$, $s''\sim P^\star(\cdot\mid s',a')$, $a''\sim q''_k(\cdot\mid s'')$, $\bar s\sim P^\star(\cdot\mid s'',a'')$.
    \STATE Update $D_k=D_{k-1}\cup\{(s,a,s')\}$, $
    D'_k=D'_{k-1}\cup\{(s',a',s'')\}$, $D''_k=D''_{k-1}\cup\{(s'',a'',\bar s)\}$.
    \STATE Fit the low-rank transition model by MLE:
    \[
    \widehat P_k=(\widehat\mu_k,\widehat\phi_k)
    \in
    \arg\max_{(\mu,\phi)\in\mathcal M}
    \mathbb E_{D_k+D'_k+D''_k}
    \big[\log \mu(x')^\top \phi(x,u)\big].
    \]
    \STATE Form the merged empirical covariance
    \[
    \widehat\Sigma^{\mathrm{mix}}_k
    =
    \sum_{(x,u,x')\in D_k}
    \widehat\phi_k(x,u)\widehat\phi_k(x,u)^\top
    +
    \sum_{(x,u,x')\in D'_k}
    \widehat\phi_k(x,u)\widehat\phi_k(x,u)^\top
    +\lambda_k I.
    \]
    \STATE Set the covariance mixture bonus
    \[
    \widehat b_k(x,u)
    =
    \min\left\{
    \alpha_k
    \sqrt{
    \widehat\phi_k(x,u)^\top
    (\widehat\Sigma^{\mathrm{mix}}_k)^{-1}
    \widehat\phi_k(x,u)
    },
    \bar b
    \right\}.
    \]
    \STATE Update the return policy $
    \pi_k
    \in
    \arg\max_{\pi}
    \left(
    Q^{\pi}_{\widehat P_k,r+\widehat b_k}
    +
    Y_k Q^{\pi}_{\widehat P_k,g+\widehat b_k}
    \right)$.
    \STATE Update the dual variable $
    Y_{k+1}
    =
    \Pi_{[0,\xi]}
    \left(
    Y_k+\eta\big(b-
    V^{\pi_k}_{\widehat P_k,g+\widehat b_k}(s_1)
    \big)
    \right)$.
\ENDFOR
\STATE {\bf Return} \(\pi_1,\ldots,\pi_K\).
\end{algorithmic}
\end{algorithm}

Algorithm~\ref{alg:rep-pd-teg} uses exact MLE and exact optimistic-Lagrangian planning to isolate the statistical challenge of learning unknown representations. The proof exposes a conditional interface for other learners: a representation method may replace MLE if it establishes the same prediction event under all three probing designs, with any excess prediction error added to the representation-error term. Likewise, full-planning suboptimality enters through the cumulative planning term. A raw MLE objective gap need not imply the required prediction event, so the present theorems do not claim guarantees for arbitrary approximate oracles, misspecified models, or infinite classes. Spectral representation learning \cite{ren2022spectral} is an example of a non-MLE alternative, but its three-design event would have to be verified separately here.

\subsection{Polynomial-Schedule Guarantees}
\label{subsec:rep-pd-decay-theory}

We state the guarantee for polynomial schedules
\(\epsilon_k=k^{-p}\), which make the exploration--deployment tradeoff
explicit. The return-policy bound captures the statistical cost of learning
the representation from auxiliary probing data, while the execution-policy
statement adds the standard performance-difference cost of actually deploying
the mixed probing policy.

\begin{theorem}[Return policy and execution policy guarantees for polynomial probing]
\label{thm:rep-pd-decay}
\label{the:regret}
Suppose Assumptions~\ref{assum:slater} and~\ref{assum:linearmdp} hold and the finite model class is realizable, \(|\mathcal M|<\infty\) and \(P^\star\in\mathcal M\).
Let \(\epsilon_k=k^{-p}\) for some \(p\in[0,1/2)\).
Run Algorithm~\ref{alg:rep-pd-teg} with
\(\alpha_k=\widetilde O(\sqrt{(|\mathcal A|k^p+d^2)\gamma\log(|\mathcal M|k/\delta)})\),
\(\lambda_k=\widetilde\Theta(((|\mathcal A|k^p+d^2)/d)\log(|\mathcal M|k/\delta))\),
\(\bar b\ge\max_{k\le K}\alpha_k/\sqrt{\lambda_k}\),
\(\eta=2(1-\gamma)/(\theta\bar b\sqrt K)\),
\(\xi=2/\theta\).
Then, with probability at least \(1-\delta\),
\[
\mathrm{Regret}^{\mathrm{return}}(K)\le
\widetilde O\!\left(\frac{|\mathcal A|\bar b d^2K^{1/2+p}}{\theta(1-\gamma)^2}\right),
\qquad
\mathrm{Violation}^{\mathrm{return}}_{\mathrm{soft}}(K)\le
\widetilde O\!\left(\frac{|\mathcal A|\bar b d^2K^{1/2+p}}{(1-\gamma)^2}\right).
\]
Moreover, for the execution policy
\(\widetilde\pi_k=(1-\epsilon_k)\pi_k+\epsilon_kU(\mathcal A)\),
\[
\begin{aligned}
\mathrm{Regret}^{\mathrm{exec}}(K)
&\le \mathrm{Regret}^{\mathrm{return}}(K)
  +\frac{2}{1-\gamma}\sum_{k=1}^K k^{-p},\\
\mathrm{Violation}^{\mathrm{exec}}_{\mathrm{soft}}(K)
&\le \mathrm{Violation}^{\mathrm{return}}_{\mathrm{soft}}(K)
  +\frac{2}{1-\gamma}\sum_{k=1}^K k^{-p}.
\end{aligned}
\]
\end{theorem}

The proof is given in Appendix~\ref{app:general_probing}. Generally, our analyses combine a tuple-level MLE event, a pathwise-to-averaged bridge, tertiary-aware optimism for the merged bonus, and a cumulative overestimation bound for the dual update. The schedule parameter \(p\) trades coverage for deployment cost: smaller \(p\) improves the return-policy rate, while larger \(p\) reduces the cost of executing the exploratory mixture. Since \(\sum_{k=1}^K k^{-p}\le K^{1-p}/(1-p)\), the return-policy term scales as \(\widetilde O(K^{1/2+p})\), whereas execution adds \(O(K^{1-p}/(1-\gamma))\).

\begin{corollary}[Best return policy and execution policy rates]
\label{cor:rep-pd-uniform}
\label{cor:uniform-soft}
\label{cor:soft_uniform}
\label{cor:rep-pd-balanced-exec}
\label{cor:poly-exec}
\label{cor:soft_exec}
Under the same setting as Theorem~\ref{thm:rep-pd-decay}, the following two
specializations hold.
\begin{itemize}
    \item First, if \(p=0\), equivalently \(\epsilon_k\equiv1\), then the probing kernel
in Algorithm~\ref{alg:rep-pd-teg} becomes uniform:
\(q_k=q'_k=q''_k=U(\mathcal A)\). The return policies satisfy
\[
\mathrm{Regret}^{\mathrm{return}}(K)\le
\widetilde O\!\left(\frac{|\mathcal A|\bar b d^2\sqrt K}{\theta(1-\gamma)^2}\right),
\quad
\mathrm{Violation}^{\mathrm{return}}_{\mathrm{soft}}(K)\le
\widetilde O\!\left(\frac{|\mathcal A|\bar b d^2\sqrt K}{(1-\gamma)^2}\right).
\]
This is the best return policy specialization, and reduces to uniform tertiary probing.
\item
Second, if \(p=1/4\), then the return policies satisfy
\[
\mathrm{Regret}^{\mathrm{return}}(K)\le
\widetilde O\!\left(\frac{|\mathcal A|\bar b d^2K^{3/4}}{\theta(1-\gamma)^2}\right),
\quad
\mathrm{Violation}^{\mathrm{return}}_{\mathrm{soft}}(K)\le
\widetilde O\!\left(\frac{|\mathcal A|\bar b d^2K^{3/4}}{(1-\gamma)^2}\right).
\]
The execution policy gap is also \(O(K^{3/4}/(1-\gamma))\), and the overall sample complexities of the execution policies become
\[
\mathrm{Regret}^{\mathrm{exec}}(K)\le
\widetilde O(K^{3/4}),
\mathrm{Violation}^{\mathrm{exec}}_{\mathrm{soft}}(K)
\le
\widetilde O(K^{3/4})
\]
up to logarithmic and problem-dependent factors.
\end{itemize}

\end{corollary}

Corollary~\ref{cor:rep-pd-uniform}
highlights the schedule tradeoff: \(p=0\) is statistically sharp for return
policies, while \(p=1/4\) balances representation learning with the cost of
using the execution policy. Thus the return-policy theorem isolates
representation-learning cost, and the execution policy corollary quantifies the
extra price of deploying the probes as real behavior.

%% file: content/hard.tex
\section{Hard Violation via Per-Episode Certification}
\label{sec:rep-pd-hard}

We now consider the hard cumulative-violation metric.
Unlike soft violation, hard violation does not allow constraint surplus in one episode to offset constraint deficit in another.
Therefore the optimistic utility estimate used by the soft primal-dual update is insufficient for certifying feasibility at the episode level: an optimistic estimate may exceed the threshold \(b\) even when the true utility remains below it.
REP-PD-H-TEG (the \textbf{H}ard-violation version of REP-PD-TEG) therefore keeps the
same representation-learning procedure but changes the policy decision:
policy search remains optimistic, while a tight feasibility test on the lower-confidence utility of the policy is explicitly conducted.

\subsection{Lower-Certiﬁcate Planning in REP-PD-H-TEG}
\label{subsec:rep-pd-hard-decay}
As shown in Algorithm~\ref{alg:rep-pd-hard-teg} in
Appendix~\ref{app:hard_general_probing},
REP-PD-H-TEG uses the same tertiary data collection, MLE representation update,
merged covariance, and bonus construction as REP-PD-TEG. The hard-specific
change lies in how the policy is selected after the bonus is built. In REP-PD-TEG, the optimistic utility value can be used in the dual update because
violation is aggregated across episodes. In the hard violation setting, however, the
policy executed in the current episode must be certified feasible on its own.
Thus REP-PD-H-TEG introduces a lower-confidence utility certificate.  Write
\(\Gamma_k:=|\mathcal A|/\epsilon_k\) and
\(\zeta_k:=\ln(|\mathcal M|k/\delta)/k\), and choose a sufficiently large
universal calibration constant \(C_{\rm cert}\).  The one-sided statistical
correction and the certificate are
\[
\ell_k:=
C_{\rm cert}\sqrt{\frac{\Gamma_k\zeta_k}{1-\gamma}},
\qquad
\underline{\mathcal U}_k(\pi):=
\left[V^{\pi}_{\widehat P_k,g-\widehat b_k}(s_1)
-\ell_k\right]_{+},
\]
where \([x]_+=\max\{x,0\}\).  On the representation-learning good event, the
one-step-back argument gives, simultaneously for every queried policy,
\[
\underline{\mathcal U}_k(\pi)
\le V^{\pi}_{P^\star,g}(s_1).
\]
Hence \(\underline{\mathcal U}_k(\pi)\ge b\) implies true feasibility.  The
corresponding two-sided width is needed only in the analysis of the selected
policy and is
\[
w_k(\pi):=
2\ell_k+2V^\pi_{\widehat P_k,\widehat b_k}(s_1),
\qquad
V^\pi_{P^\star,g}(s_1)
\le \underline{\mathcal U}_k(\pi)+w_k(\pi).
\]
Separating \(\ell_k\) from \(w_k(\pi)\) is essential: the former makes the
stopping rule conservative, while the latter accounts for the pessimistic
bonus and is summed only along the returned policy sequence.

Instead of choosing a single policy at the current dual variable, REP-PD-H-TEG performs a capped dual search within the
episode in the planning step. For each queried dual level \(Y\), it solves the optimistic Lagrangian
planning problem
\[
\widehat\pi_{k,Y}\in\arg\max_\pi
\left(
Q^\pi_{\widehat P_k,r+\widehat b_k}
+
YQ^\pi_{\widehat P_k,g+\widehat b_k}
\right),
\]
and then forms a statewise softmax candidate from the corresponding optimistic
reward-plus-utility scores. The bonus makes planning
optimistic and uncertainty-aware, exactly as in REP-PD-TEG; the stopping rule is
different. REP-PD-H-TEG deploys the first candidate whose lower-confidence
certificate satisfies \(\underline{\mathcal U}_k(\pi_{k,Y})\ge b\). If no candidate is
certified before the cap \(Y=K^{1/4}\), the algorithm deploys the final
capped-dual candidate. Thus optimistic bonus-augmented values are used to
search for good policies, while the pessimistic bonus-based certificate decides
whether a candidate is safe enough to execute.

\noindent\begin{minipage}{\linewidth}
\small
\textbf{Capped certificate search in one outer round.}
\begin{algorithmic}[1]
\STATE For each grid level $Y\le Y_{\max}$, compute the optimistic candidate $\pi_{k,Y}$.
\STATE Evaluate $\underline{\mathcal U}_k(\pi_{k,Y})$; if it is at least $b$, deploy $\pi_{k,Y}$ and stop.
\STATE Otherwise increase $Y$ and repeat the optimistic-plan/certificate test.
\STATE If no earlier candidate passes, deploy the candidate at $Y_{\max}$.
\end{algorithmic}
\end{minipage}

\subsection{Hard-Violation Guarantees}
\label{subsec:rep-pd-hard-theory}

We state the sample complexity result of REP-PD-H-TEG in this section. The proof decomposes
into a front-end regret bound, a lower-certificate reduction, and a
Slater-induced search-success argument; the full proof is given in Appendices~\ref{app:hard_general_probing} and~\ref{app:hard_tertiary_roadmap}.

\begin{center}
\small
\renewcommand{\arraystretch}{1.08}
\begin{tabular}{p{0.25\linewidth}p{0.47\linewidth}p{0.16\linewidth}}
\toprule
\textbf{Assumption/interface} & \textbf{Role in the proof} & \textbf{Results}\\
\midrule
Low rank + finite realizability & MLE prediction event and one-step-back control & Soft + hard\\
Three-stream coverage & Controls both propagated terms from the merged covariance & Soft + hard\\
Slater margin $\theta>0$ & Bounds the dual variable and cap-induced deficit & Soft + hard\\
Positive score gap & Controls the hard-search softmax approximation & Hard only\\
\bottomrule
\end{tabular}
\end{center}
The positive score gap is not used to learn the representation or to establish the lower-certificate inequality; it is used by the current hard proof to control the softmax approximation along the capped dual search.

\begin{theorem}[Return policy and execution policy hard guarantees for REP-PD-H-TEG]
\label{thm:hard-poly-main}
\label{cor:hard-decay-poly}
\label{cor:hard-poly}
\label{cor:hard_body_poly}
Suppose Assumptions~\ref{assum:slater}, \ref{assum:linearmdp}, and
\ref{ass:action_gap} hold, the finite model class is realizable
\((|\mathcal M|<\infty,\ P^\star\in\mathcal M)\), and fix
\(\delta\in(0,1)\).
Run Algorithm~\ref{alg:rep-pd-hard-teg} with \(\epsilon_k=k^{-p}\) for some
\(p\in[0,1/2)\), and set
\(\alpha_k=\widetilde\Theta(\sqrt{(|\mathcal A|k^p+d^2)\gamma\log(|\mathcal M|k/\delta)})\)
with a sufficiently large leading constant,
\(\lambda_k=\widetilde\Theta(((|\mathcal A|k^p+d^2)/d)\log(|\mathcal M|k/\delta))\),
\(\bar b\ge\max_{k\le K}\alpha_k/\sqrt{\lambda_k}\),
\(\ell_k=C_{\rm cert}\sqrt{\Gamma_k\zeta_k/(1-\gamma)}\) for a sufficiently
large universal \(C_{\rm cert}\), \(Y_{\max}=K^{1/4}\),
\(\tau_h=\Theta(1/((1-\gamma)\sqrt K\log|\mathcal A|))\),
\(\eta_h=\Theta(\tau_h(1-\gamma)/K^{3/4})\).
Then, with probability at least \(1-\delta\), the return policies satisfy
\[
\begin{aligned}
\mathrm{Regret}^{\mathrm{return}}(K)
&\le
\widetilde O\!\left(
\left[
\frac{|\mathcal A|\bar b d^2}{(1-\gamma)^2}
+\sqrt{\frac{|\mathcal A|}{1-\gamma}}
\right]K^{3/4+p}\right),
\\
\mathrm{Violation}^{\mathrm{return}}_{\mathrm{hard}}(K)
&\le
\widetilde O\!\left(
\frac{1}{\theta}
\left[
\frac{|\mathcal A|\bar b d^2}{(1-\gamma)^2}
+\sqrt{\frac{|\mathcal A|}{1-\gamma}}
\right]K^{1/2+p}
\right)
+O(\theta^{-4}).
\end{aligned}
\]
The hard-violation bound is sublinear for \(p<1/2\), while the regret bound is
sublinear only when \(p<1/4\).
Moreover, for the execution policy
\(\widetilde\pi_k=(1-\epsilon_k)\pi_k+\epsilon_kU(\mathcal A)\), we have
\(\mathrm{Regret}^{\mathrm{exec}}(K)\le \mathrm{Regret}^{\mathrm{return}}(K)+
\frac{2}{1-\gamma}\sum_{k=1}^K k^{-p}\) and
\(\mathrm{Violation}^{\mathrm{exec}}_{\mathrm{hard}}(K)\le
\mathrm{Violation}^{\mathrm{return}}_{\mathrm{hard}}(K)+
\frac{2}{1-\gamma}\sum_{k=1}^K k^{-p}\).
\end{theorem}

The \(K^{1/2+p}\) hard-violation term is the representation-learning price, and
the extra \(K^{1/4}\) in regret comes from capped in-episode dual search; the
execution policy bounds follow similar performance-difference argument from
Section~\ref{sec:rep-pd}.

\begin{corollary}[Best return policy and execution policy hard rates]
\label{cor:hard-pzero}
\label{cor:hard-uniform}
\label{cor:hard-balanced-exec}
Under the same setting as Theorem~\ref{thm:hard-poly-main}, the following two
specializations hold.
\begin{itemize}
    \item First, if \(p=0\), equivalently \(\epsilon_k\equiv1\), then REP-PD-H-TEG uses
\(q_k=q_k'=q_k''=U(\mathcal A)\) and the return policies satisfy
\[
\begin{aligned}
\mathrm{Regret}^{\mathrm{return}}(K)
&\le
\widetilde O\!\left(
\left[
\frac{|\mathcal A|\bar b d^2}{(1-\gamma)^2}
+\sqrt{\frac{|\mathcal A|}{1-\gamma}}
\right]K^{3/4}\right),
\\
\mathrm{Violation}^{\mathrm{return}}_{\mathrm{hard}}(K)
&\le
\widetilde O\!\left(
\frac{1}{\theta}
\left[
\frac{|\mathcal A|\bar b d^2}{(1-\gamma)^2}
+\sqrt{\frac{|\mathcal A|}{1-\gamma}}
\right]\sqrt K
\right)
+O(\theta^{-4}).
\end{aligned}
\]
\item Second, if \(p=1/8\), then the return policies satisfy
\[
\begin{aligned}
\mathrm{Regret}^{\mathrm{return}}(K)
&\le
\widetilde O\!\left(
\left[
\frac{|\mathcal A|\bar b d^2}{(1-\gamma)^2}
+\sqrt{\frac{|\mathcal A|}{1-\gamma}}
\right]K^{7/8}\right),
\\
\mathrm{Violation}^{\mathrm{return}}_{\mathrm{hard}}(K)
&\le
\widetilde O\!\left(
\frac{1}{\theta}
\left[
\frac{|\mathcal A|\bar b d^2}{(1-\gamma)^2}
+\sqrt{\frac{|\mathcal A|}{1-\gamma}}
\right]K^{5/8}\right)+O(\theta^{-4}).
\end{aligned}
\]
For \(p=1/8\), the execution-policy gap is \(O(K^{7/8}/(1-\gamma))\); hence both execution-policy regret and hard violation are \(\widetilde O(K^{7/8})\) up to problem-dependent factors, with exploratory actions determining the balanced rate.
\end{itemize}

\end{corollary}

Thus \(p=0\) is sharp for return policies, whereas \(p=1/8\) balances the hard-regret front end \(K^{3/4+p}\) with the execution gap \(K^{1-p}\).

%% file: content/conclusion.tex
\section{Conclusion}\label{sec:conc}
We develop representation-learning algorithms for low-rank continuous-state CMDPs with unknown transition features. REP-PD-TEG combines auxiliary probing, covariance bonuses, and optimistic primal-dual planning for soft-violation control; REP-PD-H-TEG adds a lower-confidence utility certificate and stabilized dual search for hard-violation control. Both provide return- and execution-policy guarantees.

\paragraph{Limitations.}
The theorems assume finite realizability and exact MLE/planning; alternative learners require the same three-design prediction guarantee. The hard result additionally uses a positive optimistic-score gap. The occupancy oracle is implementable with resets and geometric roll-ins, but our bounds do not charge full roll-in trajectories. Finally, the experiments illustrate the mechanism rather than exhaustively evaluating hard certification.

%% file: appendix/soft_proof.tex
\section{Notation and Auxiliary Concentration}\label{app:soft_proof}

This appendix fixes the occupancy notation used throughout the proof and
recalls the empirical-to-population covariance comparison that underlies the
uncertainty bonus. Later appendices replace the uniform probing distribution
below by the tertiary decaying-\(\epsilon\) design, but the same comparison
principle is repeatedly used.

We use \(d_{P^\star}^{\pi}\) for the normalized discounted state occupancy and
\(d_{P^\star}^{\pi}(s,a)=d_{P^\star}^{\pi}(s)\pi(a\mid s)\) for the
state-action occupancy. For the uniform probing baseline, let
\(\rho_k(s)=k^{-1}\sum_{i=0}^{k-1}d_{P^\star}^{\pi_i}(s)\), and let
\(\rho'_k\) be the marginal distribution of \(s'\) generated by first drawing
\(s\sim\rho_k\), then \(a\sim U(\mathcal A)\), and finally
\(s'\sim P^\star(\cdot\mid s,a)\).
For a feature map \(\phi\), define
\[
\Sigma_{\rho_k\times U(\mathcal A),\phi}
:=
k\mathbb E_{s\sim\rho_k,a\sim U(\mathcal A)}
[\phi(s,a)\phi(s,a)^\top]+\lambda_k I .
\]

\begin{lemma}[Concentration of the bonus term]\label{lem:con}
Let
\[
\hat\Sigma_{k,\phi}:=
\sum_{i=0}^{k-1}
\phi(s^{(i)},a^{(i)})\phi(s^{(i)},a^{(i)})^\top+\lambda_k I ,
\qquad
\lambda_k=\Theta(d\log(k|\Phi|/\delta)).
\]
Then, with probability at least \(1-\delta\), simultaneously for all
\(k\in\mathbb K^+\), all \(\phi\in\Phi\), and all \((s,a)\),
\[
c_1
\|\phi(s,a)\|_{\Sigma^{-1}_{\rho_k\times U(\mathcal A),\phi}}
\le
\|\phi(s,a)\|_{\hat\Sigma^{-1}_{k,\phi}}
\le
c_2
\|\phi(s,a)\|_{\Sigma^{-1}_{\rho_k\times U(\mathcal A),\phi}},
\]
where \(c_1,c_2>0\) are universal constants.
\end{lemma}

\begin{proof}
This is the standard empirical-to-population covariance comparison for the
uniform probing design; see Lemma 11 of \citet{uehara2021representation}. In
the repaired analysis we apply the same concentration argument to block
increments generated by the predictable auxiliary streams.
\end{proof}

%% file: appendix/lemmas.tex
\section{Auxiliary Lemmas}\label{app:aux_lemmas}
This appendix collects standard auxiliary tools used by both the soft and hard
analyses. We first recall the MLE guarantee for low-rank transition models,
then state the elliptical-potential and simulation lemmas used to aggregate
model error and convert value differences into regret and violation bounds.
The MLE statement is a direct consequence of Theorem 21 of
\citet{agarwal2020flambe}, specialized to the notation of our algorithms.

\begin{lemma}[MLE guarantee]\label{lem:mle}
For a fixed episode $k$, with probability $1-\delta$, 
\begin{align*}
   \mathbb{E}_{s \sim \{0.5\rho_k+0.5\rho'_k\},a\sim U(\mathcal{A}) }[ \|\hat P_k(\cdot \!\mid\! s,\!a)\!-\!P^{\star}(\cdot\! \mid \!s,\!a)\|^2_1 ] \lesssim  \zeta,
\end{align*}
where $\zeta := \frac{\ln(|\mathcal{M}|/\delta)}{k}$.
As a straightforward corollary, with probability $1-\delta$, $\forall k\in \mathbb{K}^{+},$
\begin{align*}
   \mathbb{E}_{s\! \sim\! \{0.5\rho_k+0.5\rho'_k\}\!,a\!\sim\! U(\mathcal{A}) }[ \|\!\hat P_k(\cdot \!\mid\! s,\!a)\!-\!P^{\star}(\cdot \!\mid\! s,\!a)\|^2_1 ]\! \lesssim \! 0.5\zeta_k ,
\end{align*}
where $\zeta_k := \frac{\ln(|\mathcal{M}|k/\delta)}{k}$.
\end{lemma}

The next two lemmas are the standard determinant-trace inequalities used in the
elliptical-potential argument; see Lemma G.2 of
\citet{agarwal2020flambe}.

\begin{lemma}\label{lem:reduction}
Consider the following process. For $k=1,\cdots,K$, $M_k=M_{k-1}+G_k$ with $M_0=\lambda_0 I$ and $G_k$ being a positive semidefinite matrix with eigenvalues upper-bounded by $1$. We have that:
\begin{align*}
    2\ln \det (M_K)-2\ln \det(\lambda_0 I)\geq \sum_{k=1}^K\text{Tr}(G_k M^{-1}_{k-1}). 
\end{align*}
\end{lemma}

\begin{lemma}[Potential function lemma]\label{lem:potential}
Suppose $\text{Tr}(G_k)\leq B^2$. 
\begin{align*}
      2\ln \det (M_K)-2\ln \det(\lambda_0 I)\leq d\ln\left(1+\frac{KB^2}{d\lambda_0}\right). 
\end{align*}
\end{lemma}

\begin{lemma}[Simulation lemma]\label{lem:performance}
For any $c\in\{r,g\}$, two CMDPs $(P',c+b)$ and $(P,c)$, and any
policy $\pi$, we have
\begin{align*}
V^\pi_{P',c+b}-V^\pi_{P,c}
&=\frac{1}{1-\gamma}\mathbb E_{d^\pi_{P'}}
\!\left[
b(s,a)+\gamma\mathbb E_{P'}Q^\pi_{P,c}(s',\pi)
-\gamma\mathbb E_{P}Q^\pi_{P,c}(s',\pi)
\right],\\
V^\pi_{P',c+b}-V^\pi_{P,c}
&=\frac{1}{1-\gamma}\mathbb E_{d^\pi_{P}}
\!\left[
b(s,a)+\gamma\mathbb E_{P'}Q^\pi_{P,c+b}(s',\pi)
-\gamma\mathbb E_{P}Q^\pi_{P',c+b}(s',\pi)
\right].
\end{align*}
\end{lemma}

Finally, we use the following primal-dual reduction, which is Lemma 9 of
\citet{ding2021provably} in our notation.
\begin{lem}\label{lem:convex}
Let $Y^*$ be the optimal dual variable, and $C\geq 2Y^*$, then, if 
\begin{align*}
V_{r}^{\pi^*}(s_1)-V_{r}^{\tilde{\pi}}(s_1)+C[b-V_{g}^{\tilde{\pi}}(s_1)]_{+}\leq \delta
\end{align*}
then
\begin{align*}
[b-V_{g}^{\tilde{\pi}}(s_1)]_{+}\leq \frac{2\delta}{C}.
\end{align*}
\end{lem}

%% file: appendix/hard_proof.tex
\section{Hard-Constraint Technical Lemmas}\label{app:hard_proof}

This appendix collects the hard-constraint technical ingredients that are reused
by the lower-certificate analyses in Appendices~\ref{app:hard_tertiary_roadmap}
and~\ref{app:hard_general_probing}. The hard-violation guarantees in this draft
are stated through lower-confidence certificates and Slater-induced control of
cap episodes; no optimistic feasibility check is used as a hard certificate.

\begin{assum}[Positive action gap]\label{ass:action_gap}
For every episode \(k\), queried dual level \(Y\), and state \(s\), the
maximizer of the optimistic Lagrangian action score used by
Algorithm~\ref{alg:rep-pd-hard-teg} is separated from every suboptimal action
by at least \(\delta_{\min}>0\).  The constant is uniform over \(k,Y,s\).
\end{assum}

\begin{lemma}[Softmax sensitivity]\label{lem:eta_diff}
Let \(\pi_k^Y\) and \(\pi_k^{Y+\eta}\) be the softmax policies produced by two
neighboring dual levels \(Y\) and \(Y+\eta\) with temperature \(\tau\). Then
\[
\bigl|
V_g^{\pi_k^Y}(s_1)-V_g^{\pi_k^{Y+\eta}}(s_1)
\bigr|
\le
O\!\left(\frac{\eta}{\tau(1-\gamma)}\right).
\]
\end{lemma}

\begin{proof}
Let \(\Delta=\sup_s|V_g^{\pi_k^Y}(s)-V_g^{\pi_k^{Y+\eta}}(s)|\). Comparing the
two Bellman equations gives
\[
\Delta
\le
\gamma\Delta+
\sup_s
\|\pi_k^Y(\cdot\mid s)-\pi_k^{Y+\eta}(\cdot\mid s)\|_1 .
\]
The softmax map is \(O(1/\tau)\)-Lipschitz in its input. Increasing the dual
level by \(\eta\) changes the softmax score by \(O(\eta)\) under the normalized
utility scale, and therefore
\[
\sup_s
\|\pi_k^Y(\cdot\mid s)-\pi_k^{Y+\eta}(\cdot\mid s)\|_1
\le O(\eta/\tau).
\]
Rearranging the Bellman inequality proves the claim.
\end{proof}

%% file: appendix/new_proof.tex
\section{Decaying \texorpdfstring{$\epsilon$}{epsilon}-Greedy Probing: Tertiary-Aware Analysis}
\label{app:general_probing}

This appendix proves the decaying-\(\epsilon\) tertiary-aware result used in
Section~\ref{sec:rep-pd}. The proof is organized around three steps:

\begin{enumerate}
    \item the naive two-level decaying-$\epsilon$ probing rule is not enough to
    close the optimism step;
    \item a minimal tertiary-aware repair restores the missing denominator;
    \item for this repaired algorithm, one can prove a regret guarantee for the
    \emph{return policy sequence} $\{\pi_k\}_{k=1}^K$.
\end{enumerate}

The guarantees in this appendix are for the return policies of
Algorithm~\ref{alg:rep-pd-teg}. The execution-policy transfer in
Section~\ref{sec:rep-pd} is obtained afterward through a performance-difference
argument. The same tertiary-aware front end also supports soft-violation
control once the cumulative optimistic overestimation of the utility is bounded
explicitly. The hard-violation analysis in later appendices uses the same
representation-learning front end, but replaces optimistic feasibility checks
with a lower-confidence certificate.

\subsection{Why the Naive Two-Level Variant Is Not Enough}

The natural first attempt is the two-level decaying-$\epsilon$ rule
\[
q_i(\cdot\mid s)=q_i'(\cdot\mid s)
=(1-\epsilon_i)\pi_{i-1}(\cdot\mid s)+\epsilon_i U(\mathcal A).
\]
This already improves the sampling rule relative to pure uniform probing, but
it does not fully solve the coverage problem needed by the optimism proof.

\begin{proposition}[Two-level obstruction]
\label{prop:bridge_obstruction}
There is no universal constant $C<\infty$ such that the pointwise comparison
\[
\|\hat \phi_k(s,a)\|_{(\Sigma_k^{\mathrm{pri}})^{-1}}
\le
C\,
\|\hat \phi_k(s,a)\|_{(\Sigma_k^{\mathrm{mix}})^{-1}},
\qquad \forall (s,a),
\]
holds for every episode $k$ under the naive two-level decaying-$\epsilon$
probing rule.
\end{proposition}

\begin{proof}
Consider a deterministic two-state MDP with states $s_0,s_1$ and two actions
$a_{\mathrm{stay}},a_{\mathrm{jump}}$. Let the deployed policy always choose
$a_{\mathrm{stay}}$ at $s_0$, let $s_1$ be absorbing, and let
$a_{\mathrm{jump}}$ be the only action that can move the process from $s_0$ to
$s_1$. Choose a bounded feature map such that
\[
\hat \phi_k(s_0,a)=e_1,
\qquad
\hat \phi_k(s_1,a)=e_2,
\qquad \forall a\in\mathcal A.
\]
Then all primary samples stay at $s_0$, so
\[
\Sigma_k^{\mathrm{pri}}=ke_1e_1^\top+\lambda_k I.
\]
However, the secondary stream reaches $s_1$ with probability of order
$\epsilon_k/|\mathcal A|$, hence
\[
\Sigma_k^{\mathrm{mix}}
\succeq
ke_1e_1^\top
+
\frac{k\epsilon_k}{|\mathcal A|}e_2e_2^\top
+
\lambda_k I.
\]
Evaluating the two norms at any point with feature vector $e_2$ gives
\[
\|e_2\|_{(\Sigma_k^{\mathrm{pri}})^{-1}}
=
\frac{1}{\sqrt{\lambda_k}},
\qquad
\|e_2\|_{(\Sigma_k^{\mathrm{mix}})^{-1}}
\le
\frac{1}{
\sqrt{\lambda_k+k\epsilon_k/|\mathcal A|}
}.
\]
Therefore
\[
\frac{\|e_2\|_{(\Sigma_k^{\mathrm{pri}})^{-1}}}
{\|e_2\|_{(\Sigma_k^{\mathrm{mix}})^{-1}}}
\ge
\sqrt{1+\frac{k\epsilon_k}{|\mathcal A|\lambda_k}},
\]
which rules out a universal constant bridge.
\end{proof}

Proposition~\ref{prop:bridge_obstruction} explains why the clean theorem is not
for the naive two-level rule. The minimal repair for the present merged-covariance design is to add one more probe level. This obstruction does not rule out a primary-only alternative that fits on $D+D'$ but constructs its covariance from $D$ alone, as in the architecture of \citet{uehara2021representation}. Such a design removes the secondary covariance term but requires a different potential argument; the current regret aggregation cannot be transferred to it directly.

\subsection{The Tertiary-Aware Repaired Algorithm}

For each round $i\in[K]$, define the primary, secondary, and tertiary state
marginals by
\[
\rho_i(s):=d_{P^\star}^{\pi_{i-1}}(s),
\]
\[
\rho_i'(s')
:=
\mathbb E_{(s,a)\sim \rho_i\times q_i}
\bigl[P^\star(s'\mid s,a)\bigr],
\]
\[
\rho_i''(s'')
:=
\mathbb E_{(s,a)\sim \rho_i'\times q_i'}
\bigl[P^\star(s''\mid s,a)\bigr].
\]
The corresponding probing distributions are
\[
\nu_i(s,a)
:=\rho_i(s)q_i(a\mid s),\quad
\nu_i'(s,a)
:=\rho_i'(s)q_i'(a\mid s),\quad
\nu_i''(s,a)
:=\rho_i''(s)q_i''(a\mid s).
\]
We write
\[
\bar \nu_k^{\mathrm{pri}}
:=
\frac{1}{k}\sum_{i=1}^k \nu_i,
\qquad
\bar \nu_k^{\mathrm{sec}}
:=
\frac{1}{k}\sum_{i=1}^k \nu_i',
\qquad
\bar \nu_k^{\mathrm{ter}}
:=
\frac{1}{k}\sum_{i=1}^k \nu_i'',
\]
and
\[
\bar \nu_k^{\mathrm{all}}
:=
\frac{1}{3}
\bigl(
\bar \nu_k^{\mathrm{pri}}
+
\bar \nu_k^{\mathrm{sec}}
+
\bar \nu_k^{\mathrm{ter}}
\bigr).
\]

The repaired probing kernels are
\[
q_k(\cdot\mid s)=q_k'(\cdot\mid s)=q_k''(\cdot\mid s)
=(1-\epsilon_k)\pi_{k-1}(\cdot\mid s)+\epsilon_k U(\mathcal A),
\]
where
\[
\epsilon_1\ge \epsilon_2\ge \cdots \ge \epsilon_K>0.
\]
We also write
\[
\Gamma_k:=\frac{|\mathcal A|}{\epsilon_k},
\qquad
\Lambda_k:=\frac{1}{1-\epsilon_k+\epsilon_k/|\mathcal A|}.
\]

The coverage coefficients used below are domination constants between the
target distributions appearing in the proof and the three empirical probing
streams.  For two state-action distributions \(\eta,\nu\), define
\[
\mathsf D(\eta\|\nu)
:=
\inf\left\{
c\ge0:
\mathbb E_{\eta}[g^2]\le c\,\mathbb E_{\nu}[g^2]\ \text{for all }g
\right\}.
\]
When densities exist, this is the essential supremum of \(d\eta/d\nu\).  For a
state-action distribution \(\nu\), let \(\mathsf T_\pi\nu\) denote the
one-step pushforward obtained by drawing
\((\tilde s,\tilde a)\sim\nu\), \(s\sim P^\star(\cdot\mid\tilde s,\tilde a)\),
and \(a\sim\pi(\cdot\mid s)\).  With this notation,
\[
\begin{aligned}
\mathfrak C_k^{\mathrm{init}}(\pi)
&:=\mathsf D\!\left((1-\gamma)d_1\otimes\pi
\;\middle\|\;\bar\nu_k^{\mathrm{pri}}\right),\\
\mathfrak C_k^{\mathrm{next}}(\pi)
&:=\mathsf D\!\left(\mathsf T_\pi\bar\nu_k^{\mathrm{pri}}
\;\middle\|\;\bar\nu_k^{\mathrm{sec}}\right),\\
\mathfrak C_{k,1}^{\mathrm{tri,next}}(\pi)
&:=\mathfrak C_k^{\mathrm{next}}(\pi),\\
\mathfrak C_{k,2}^{\mathrm{tri,next}}(\pi)
&:=\mathsf D\!\left(\mathsf T_\pi\bar\nu_k^{\mathrm{sec}}
\;\middle\|\;\bar\nu_k^{\mathrm{ter}}\right),\\
\mathfrak C_k^{\mathrm{tri,next}}(\pi)
&:=\max\!\left\{
\mathfrak C_{k,1}^{\mathrm{tri,next}}(\pi),
\mathfrak C_{k,2}^{\mathrm{tri,next}}(\pi)
\right\}.
\end{aligned}
\]
Thus \(\mathfrak C_k^{\mathrm{init}}\) compares the initial discounted mass to
the primary stream, \(\mathfrak C_k^{\mathrm{next}}\) compares the
primary-to-secondary one-step design, and
\(\mathfrak C_k^{\mathrm{tri,next}}\) is the maximum of the two one-step
comparisons needed for the tertiary one-step-back argument.

The algorithm analyzed in this appendix is the main-text
Algorithm~\ref{alg:rep-pd-teg}. The repair relative to the earlier
two-stream protocol is entirely in the auxiliary sampling design: it adds one
more probe-action / transition pair, while keeping the bonus covariance merged
over the primary and secondary tuples.

\subsection{Supporting Lemmas}

We now record the small collection of lemmas needed for the final theorem.

\begin{lemma}[Three-stream MLE event at tuple level]
\label{lem:gen_mle}
Under realizability $P^\star\in\mathcal M$, with probability at least
$1-\delta$, simultaneously for all $k\in[K]$,
\[
\mathbb E_{(s,a)\sim \widetilde \nu_k^{\mathrm{all}}}
\left[
\|\hat P_k(\cdot\mid s,a)-P^\star(\cdot\mid s,a)\|_1^2
\right]
\lesssim
\frac{\ln(|\mathcal M|k/\delta)}{k}
=:
\zeta_k.
\]
Here
\[
\widetilde \nu_k^{\mathrm{sec}}
:=
\frac{1}{k}\sum_{i=1}^k \delta_{s_i'}q_i'(\cdot\mid s_i'),
\qquad
\widetilde \nu_k^{\mathrm{ter}}
:=
\frac{1}{k}\sum_{i=1}^k \delta_{s_i''}q_i''(\cdot\mid s_i''),
\]
and
\[
\widetilde \nu_k^{\mathrm{all}}
:=
\frac{1}{3}
\bigl(
\bar \nu_k^{\mathrm{pri}}
+
\widetilde \nu_k^{\mathrm{sec}}
+
\widetilde \nu_k^{\mathrm{ter}}
\bigr).
\]
\end{lemma}

\begin{proof}
For episode $i$, write the realized tertiary tuple as
\((s_i,a_i,s_i',a_i',s_i'',a_i'',\bar s_i)\), where
\(\bar s_i\sim P^\star(\cdot\mid s_i'',a_i'')\) is the next state of the
third probe transition.  Define
\[
x_{3i-2}:=(s_i,a_i),
\qquad
y_{3i-2}:=s_i',
\]
\[
x_{3i-1}:=(s_i',a_i'),
\qquad
y_{3i-1}:=s_i'',
\]
and
\[
x_{3i}:=(s_i'',a_i''),
\qquad
y_{3i}:=\bar s_i.
\]
Let $\mathcal F_{3i-3}$ be the sigma-field generated by all data collected
before episode $i$ starts. Recursively define
\[
\mathcal F_{3i-2}
:=
\sigma(\mathcal F_{3i-3},x_{3i-2},y_{3i-2}),\quad
\mathcal F_{3i-1}
:=
\sigma(\mathcal F_{3i-2},x_{3i-1},y_{3i-1}),
\]
and
\[
\mathcal F_{3i}
:=
\sigma(\mathcal F_{3i-1},x_{3i},y_{3i}).
\]
Then the predictable covariate designs are
\[
D_{3i-2}
:=
\mathrm{Law}(x_{3i-2}\mid \mathcal F_{3i-3})
=
\nu_i,
\]
\[
D_{3i-1}
:=
\mathrm{Law}(x_{3i-1}\mid \mathcal F_{3i-2})
=
\delta_{s_i'}q_i'(\cdot\mid s_i'),
\]
\[
D_{3i}
:=
\mathrm{Law}(x_{3i}\mid \mathcal F_{3i-1})
=
\delta_{s_i''}q_i''(\cdot\mid s_i'').
\]
Hence the full covariate sequence
\[
x_1,x_2,\ldots,x_{3k}
\]
is a predictable design sequence with respect to the tuple-level filtration
$\{\mathcal F_t\}_{t=0}^{3k}$ in the sense of
\citep[Theorem 21]{agarwal2020flambe}. Moreover, the responses satisfy
\[
y_t\sim P^\star(\cdot\mid x_t)
\qquad\text{for every }t\le 3k,
\]
and realizability gives $P^\star\in\mathcal M$.

Apply Theorem 21 of \cite{agarwal2020flambe} to the first $3k$ samples, with
confidence level $\delta/k^2$. This yields, for every fixed $k$,
\[
\frac{1}{3k}
\sum_{t=1}^{3k}
\mathbb E_{x\sim D_t}
\Bigl[
\|\hat P_k(\cdot\mid x)-P^\star(\cdot\mid x)\|_1^2
\Bigr]
\lesssim
\frac{\ln(|\mathcal M|k/\delta)}{k},
\]
where $D_t$ denotes the predictable design of the $t$-th tuple. By the
definitions above,
\[
\frac{1}{3k}\sum_{t=1}^{3k} D_t
=
\frac{1}{3}
\bigl(
\bar \nu_k^{\mathrm{pri}}
+
\widetilde \nu_k^{\mathrm{sec}}
+
\widetilde \nu_k^{\mathrm{ter}}
\bigr)
=
\widetilde \nu_k^{\mathrm{all}},
\]
as a random probability measure. Therefore the left-hand side is precisely
\[
\mathbb E_{(s,a)\sim \widetilde \nu_k^{\mathrm{all}}}
\left[
\|\hat P_k(\cdot\mid s,a)-P^\star(\cdot\mid s,a)\|_1^2
\right].
\]
Finally, a union bound over all $k\in[K]$ turns the fixed-$k$ statement into
the simultaneous event claimed in the lemma, with the extra
$\ln k$ absorbed into $\zeta_k$.
\end{proof}

\begin{remark}[Route choice: pathwise-to-averaged bridge]
\label{rem:pathwise_bridge}
Lemma~\ref{lem:gen_mle} is the formally correct tuple-level statement: the
secondary and tertiary designs are the pathwise predictable kernels
$\delta_{s_i'}q_i'(\cdot\mid s_i')$ and
$\delta_{s_i''}q_i''(\cdot\mid s_i'')$, not the averaged covariates
$\nu_i',\nu_i''$ directly. For the remainder of this appendix, we keep the
coverage coefficients and theorem statements in terms of the deterministic
averaged designs $\bar \nu_k^{\mathrm{pri}}$, $\bar \nu_k^{\mathrm{sec}}$, and
$\bar \nu_k^{\mathrm{ter}}$, rather than rewriting every later lemma in
pathwise notation. We therefore choose the
\emph{pathwise-to-averaged bridge} route: whenever the tuple-level MLE event is
used under $\bar \nu_k^{\mathrm{sec}}$ or $\bar \nu_k^{\mathrm{ter}}$, we apply
the bridge lemma below to the relevant squared mismatch class. This ensures
simultaneous control under the random pathwise designs
$\widetilde \nu_k^{\mathrm{sec}},\widetilde \nu_k^{\mathrm{ter}}$ and their
conditional means
$\bar \nu_k^{\mathrm{sec}},\bar \nu_k^{\mathrm{ter}}$. The covariance
comparison lemma below already avoids this issue by working at the episode-block
level. The only bridge actually needed later in this appendix is for the squared
MLE-mismatch class, which we state next.
\end{remark}

\begin{lemma}[Pathwise-to-averaged bridge for the MLE mismatch class]
\label{lem:pathwise_bridge_mle}
For each candidate model $P\in\mathcal M$, define
\[
h_P(s,a)
:=
\|P(\cdot\mid s,a)-P^\star(\cdot\mid s,a)\|_1^2.
\]
Then there is a universal constant $C<\infty$ such that, with probability at
least $1-\delta$, simultaneously for all $k\in[K]$ and all $P\in\mathcal M$,
\[
\mathbb E_{\bar \nu_k^{\mathrm{sec}}}[h_P]
\le
2\mathbb E_{\widetilde \nu_k^{\mathrm{sec}}}[h_P]
+
C\frac{\ln(|\mathcal M|k/\delta)}{k},
\]
and
\[
\mathbb E_{\bar \nu_k^{\mathrm{ter}}}[h_P]
\le
2\mathbb E_{\widetilde \nu_k^{\mathrm{ter}}}[h_P]
+
C\frac{\ln(|\mathcal M|k/\delta)}{k}.
\]
Consequently, on the intersection of this event with
Lemma~\ref{lem:gen_mle}, we have simultaneously for all $k\in[K]$,
\[
\mathbb E_{\bar \nu_k^{\mathrm{sec}}}[f_k^2]
\lesssim
\zeta_k,
\qquad
\mathbb E_{\bar \nu_k^{\mathrm{ter}}}[f_k^2]
\lesssim
\zeta_k.
\]
In particular, this lemma is used later only to transfer the internal
MLE-mismatch terms from the pathwise designs to the averaged designs. The
generic test function $g$ in the one-step-back and optimism lemmas is handled
directly under the deterministic averaged measures
$\bar \nu_k^{\mathrm{sec}},\bar \nu_k^{\mathrm{ter}}$ through the coverage
coefficients, so no stronger uniform bridge over an adaptive class is invoked.
\end{lemma}

\begin{proof}
Fix a candidate model $P\in\mathcal M$ and write $h:=h_P$. Since
$\|P(\cdot\mid s,a)-P^\star(\cdot\mid s,a)\|_1\le 2$, we have
\[
0\le h(s,a)\le 4
\qquad\forall (s,a).
\]

We first treat the secondary stream. For each episode $i$, define
\[
X_i^{\mathrm{sec}}(h)
:=
h(s_i',a_i')
-
\mathbb E_{\bar \nu_i^{\mathrm{sec}}}[h]
=
h(s_i',a_i')
-
\mathbb E_{\nu_i'}[h].
\]
Let $\mathcal G_i$ denote the sigma-field generated by all tuples up to the end
of episode $i$. Because $(s_i',a_i')$ has conditional mean design $\nu_i'$
given $\mathcal G_{i-1}$, we have
\[
\mathbb E[X_i^{\mathrm{sec}}(h)\mid \mathcal G_{i-1}]=0.
\]
Moreover,
\[
|X_i^{\mathrm{sec}}(h)|\le 4,
\]
and, using $h^2\le 4h$ on $[0,4]$,
\[
\mathbb E\!\left[
\bigl(X_i^{\mathrm{sec}}(h)\bigr)^2
\mid \mathcal G_{i-1}
\right]
\le
\mathbb E\!\left[
h(s_i',a_i')^2
\mid \mathcal G_{i-1}
\right]
\le
4\mathbb E[h(s_i',a_i')\mid \mathcal G_{i-1}]
=
4\mathbb E_{\nu_i'}[h].
\]
Hence the martingale $S_k^{\mathrm{sec}}(h):=\sum_{i=1}^k X_i^{\mathrm{sec}}(h)$
has predictable quadratic variation bounded by
\[
V_k^{\mathrm{sec}}(h)
\le 4\sum_{i=1}^k \mathbb E_{\nu_i'}[h]
=4k\,\mathbb E_{\bar \nu_k^{\mathrm{sec}}}[h].
\]

Apply Freedman's inequality with confidence level
\(\delta_{k,P}^{\mathrm{sec}}:=\delta/(4|\mathcal M|k^2)\), and write
\(\ell_k^{\mathrm{sec}}:=\ln(1/\delta_{k,P}^{\mathrm{sec}})\).
Then, with probability at least $1-\delta_{k,P}^{\mathrm{sec}}$,
\[
-S_k^{\mathrm{sec}}(h)
\le
\sqrt{
2V_k^{\mathrm{sec}}(h)\ell_k^{\mathrm{sec}}
}
+
\frac{4}{3}\ell_k^{\mathrm{sec}}.
\]
Dividing by $k$ and recalling that
\[
\frac{-S_k^{\mathrm{sec}}(h)}{k}
=
\mathbb E_{\bar \nu_k^{\mathrm{sec}}}[h]
-
\mathbb E_{\widetilde \nu_k^{\mathrm{sec}}}[h],
\]
we obtain
\[
\mathbb E_{\bar \nu_k^{\mathrm{sec}}}[h]
-
\mathbb E_{\widetilde \nu_k^{\mathrm{sec}}}[h]
\le
\sqrt{
\frac{
8\mathbb E_{\bar \nu_k^{\mathrm{sec}}}[h]\ell_k^{\mathrm{sec}}
}{k}
}
+
\frac{4}{3}
\frac{\ell_k^{\mathrm{sec}}}{k}.
\]
Using $2\sqrt{ab}\le \frac{a}{2}+2b$ with
\(a=\mathbb E_{\bar \nu_k^{\mathrm{sec}}}[h]\) and
\(b=8\ell_k^{\mathrm{sec}}/k\)
gives
\[
\sqrt{
\frac{
8\mathbb E_{\bar \nu_k^{\mathrm{sec}}}[h]\ell_k^{\mathrm{sec}}
}{k}
}
\le
\frac{1}{2}\mathbb E_{\bar \nu_k^{\mathrm{sec}}}[h]
+
16\frac{\ell_k^{\mathrm{sec}}}{k}.
\]
Therefore
\[
\mathbb E_{\bar \nu_k^{\mathrm{sec}}}[h]
\le
2\mathbb E_{\widetilde \nu_k^{\mathrm{sec}}}[h]
+
\frac{104}{3}\frac{\ell_k^{\mathrm{sec}}}{k}.
\]
Since
\(\ell_k^{\mathrm{sec}}=\ln(4|\mathcal M|k^2/\delta)
\lesssim \ln(|\mathcal M|k/\delta)\),
this yields the desired secondary bound after adjusting the universal
constant $C$.

The tertiary bound is identical. Indeed, defining
\[
X_i^{\mathrm{ter}}(h)
:=
h(s_i'',a_i'')
-
\mathbb E_{\nu_i''}[h],
\]
we again have
\[
\mathbb E[X_i^{\mathrm{ter}}(h)\mid \mathcal G_{i-1}]=0,
\qquad
|X_i^{\mathrm{ter}}(h)|\le 4,
\]
and
\[
\mathbb E\!\left[
\bigl(X_i^{\mathrm{ter}}(h)\bigr)^2
\mid \mathcal G_{i-1}
\right]
\le
4\mathbb E_{\nu_i''}[h].
\]
Running the same Freedman argument and taking a union bound over the secondary
and tertiary streams, over all $P\in\mathcal M$, and over all $k\in[K]$, gives
the first two displays in the lemma.

Finally, set $P=\hat P_k$. Since $\hat P_k\in\mathcal M$, we may plug
$h_{\hat P_k}=f_k^2$ into the simultaneous event. By Lemma~\ref{lem:gen_mle},
\[
\mathbb E_{\widetilde \nu_k^{\mathrm{all}}}[f_k^2]
\lesssim
\zeta_k.
\]
Because
\[
\widetilde \nu_k^{\mathrm{all}}
=
\frac{1}{3}
\bigl(
\bar \nu_k^{\mathrm{pri}}
+
\widetilde \nu_k^{\mathrm{sec}}
+
\widetilde \nu_k^{\mathrm{ter}}
\bigr),
\]
and all three component expectations are nonnegative,
\[
\mathbb E_{\widetilde \nu_k^{\mathrm{sec}}}[f_k^2]
\le
3\mathbb E_{\widetilde \nu_k^{\mathrm{all}}}[f_k^2]
\lesssim
\zeta_k,
\]
and likewise
\[
\mathbb E_{\widetilde \nu_k^{\mathrm{ter}}}[f_k^2]
\lesssim
\zeta_k.
\]
Substituting these into the secondary and tertiary bridge inequalities and
absorbing the extra
$\ln(|\mathcal M|k/\delta)/k$ term into $\zeta_k$ proves the final display.
\end{proof}

\begin{lemma}[Primary-covariance deployed-to-probe transfer]
\label{lem:gen_useful}
Fix an episode $k$ and define
\[
\Sigma_{k,\star}^{\mathrm{pri}}
:=
k\mathbb E_{\bar \nu_k^{\mathrm{pri}}}
[\phi^\star(s,a)\phi^\star(s,a)^\top]
+
\lambda_k I.
\]
Then for every policy $\pi$ and every measurable
$g:\mathcal S\times\mathcal A\to\mathbb R$ with $\|g\|_\infty\le B$,
\begin{align*}
\left|
\mathbb E_{d_{P^\star}^\pi}[g]
\right|
\lesssim&
\sqrt{\gamma}\,
\mathbb E_{d_{P^\star}^\pi}
\|\phi^\star(\tilde s,\tilde a)\|_{(\Sigma_{k,\star}^{\mathrm{pri}})^{-1}}
\sqrt{
k\mathfrak C_k^{\mathrm{next}}(\pi)\,
\mathbb E_{\bar \nu_k^{\mathrm{sec}}}[g^2]
+
\lambda_k d B^2
}
\\
&+
\sqrt{
(1-\gamma)\mathfrak C_k^{\mathrm{init}}(\pi)\,
\mathbb E_{\bar \nu_k^{\mathrm{pri}}}[g^2]
}.
\end{align*}
\end{lemma}

\begin{proof}
Define
\[
m_g(s):=\mathbb E_{a\sim \pi(\cdot\mid s)}[g(s,a)],\quad
T_\star^\pi g(\tilde s,\tilde a)
:=\mathbb E_{s\sim P^\star(\cdot\mid \tilde s,\tilde a)}[m_g(s)].
\]
Then $|m_g(s)|\le B$, and the discounted-occupancy recursion gives
\begin{align}
\mathbb E_{d_{P^\star}^{\pi}}[g]
&=\gamma\mathbb E_{d_{P^\star}^{\pi}}[T_\star^\pi g]
+(1-\gamma)\mathbb E_{s\sim d_1}[m_g(s)].
\label{eq:gen_useful_recursion}
\end{align}
The initialization term is bounded by Cauchy-Schwarz and the definition of
$\mathfrak C_k^{\mathrm{init}}(\pi)$:
\[
(1-\gamma)\left|\mathbb E_{s\sim d_1}[m_g(s)]\right|
\le
\sqrt{
(1-\gamma)\mathfrak C_k^{\mathrm{init}}(\pi)\,
\mathbb E_{\bar \nu_k^{\mathrm{pri}}}[g^2]
}.
\]
Next write \(v_g^\star:=\int \mu^\star(s)m_g(s)\,ds\). Since
\(T_\star^\pi g(\tilde s,\tilde a)
=\phi^\star(\tilde s,\tilde a)^\top v_g^\star\), Cauchy-Schwarz gives
\[
\mathbb E_{d_{P^\star}^{\pi}}[T_\star^\pi g]
\le
\mathbb E_{d_{P^\star}^{\pi}}
\|\phi^\star(\tilde s,\tilde a)\|_{(\Sigma_{k,\star}^{\mathrm{pri}})^{-1}}
\cdot
\|v_g^\star\|_{\Sigma_{k,\star}^{\mathrm{pri}}}.
\]
Moreover, \(\|v_g^\star\|_2\le \sqrt d\,B\), and
\[
\|v_g^\star\|_{\Sigma_{k,\star}^{\mathrm{pri}}}^2
=
k\mathbb E_{\bar \nu_k^{\mathrm{pri}}}
\bigl[(T_\star^\pi g)^2\bigr]
+\lambda_k\|v_g^\star\|_2^2.
\]
By Jensen and the definition of $\mathfrak C_k^{\mathrm{next}}(\pi)$,
\[
\mathbb E_{\bar \nu_k^{\mathrm{pri}}}
\bigl[(T_\star^\pi g)^2\bigr]
\le
\mathfrak C_k^{\mathrm{next}}(\pi)\,
\mathbb E_{\bar \nu_k^{\mathrm{sec}}}[g^2].
\]
Substituting these bounds into \eqref{eq:gen_useful_recursion} gives the claim.
\end{proof}

\begin{lemma}[Merged covariance comparison]
\label{lem:gen_con}
For each episode $k$, define
\[
\Sigma_k^{\mathrm{mix}}
:=k\mathbb E_{\bar \nu_k^{\mathrm{pri}}}[\hat \phi_k\hat \phi_k^\top]
+k\mathbb E_{\bar \nu_k^{\mathrm{sec}}}[\hat \phi_k\hat \phi_k^\top]
+\lambda_k I.
\]
There exist universal constants $0<c_1\le c_2<\infty$ such that, with
probability at least $1-\delta$, simultaneously for all $k\in[K]$ and all
$(s,a)$,
\[
c_1\|\hat \phi_k(s,a)\|_{(\Sigma_k^{\mathrm{mix}})^{-1}}
\le \|\hat \phi_k(s,a)\|_{(\hat \Sigma_k^{\mathrm{mix}})^{-1}}
\le c_2
\|\hat \phi_k(s,a)\|_{(\Sigma_k^{\mathrm{mix}})^{-1}}.
\]
\end{lemma}

\begin{proof}
Fix an episode $k$ and first work with an arbitrary candidate feature map
$\phi\in\Gamma$. Define the episode-block matrix
\[
H_i(\phi):=
\phi(s_i,a_i)\phi(s_i,a_i)^\top
+\phi(s_i',a_i')\phi(s_i',a_i')^\top,
\qquad i\in[k].
\]
Let $\mathcal G_i$ be the sigma-field generated by all tuples collected up to
the end of episode $i$, with $\mathcal G_0$ the history before episode $1$.
Define the predictable episode-level conditional mean
\(\bar H_i(\phi):=\mathbb E[H_i(\phi)\mid \mathcal G_{i-1}]\).
By construction, conditioned on $\mathcal G_{i-1}$ the primary tuple is sampled
from $\nu_i$, while the secondary tuple has conditional mean design $\nu_i'$
after marginalizing over the primary transition and the secondary probe action.
Therefore
\[
\begin{aligned}
\bar H_i(\phi)
&=\mathbb E_{\nu_i}[\phi\phi^\top]
+\mathbb E_{\nu_i'}[\phi\phi^\top],\qquad
\sum_{i=1}^{k}\bar H_i(\phi)
=k\mathbb E_{\bar \nu_k^{\mathrm{pri}}}[\phi\phi^\top]
+k\mathbb E_{\bar \nu_k^{\mathrm{sec}}}[\phi\phi^\top].
\end{aligned}
\]

Now define
\[
\Sigma_k^{\mathrm{mix}}(\phi)
:=\sum_{i=1}^{k}\bar H_i(\phi)+\lambda_k I,
\qquad
\hat \Sigma_k^{\mathrm{mix}}(\phi)
:=\sum_{i=1}^{k}H_i(\phi)+\lambda_k I.
\]
Under the bounded-feature assumption used throughout the paper,
$H_i(\phi)$ is positive semidefinite with operator norm at most $2$. Hence the
matrix martingale
\[
M_t(\phi):=\sum_{u=1}^t\bigl(H_u(\phi)-\bar H_u(\phi)\bigr)
\]
satisfies the standard bounded-increment conditions for episode-block matrix
concentration with respect to $\{\mathcal G_i\}_{i=0}^{k}$. In particular, the
same self-normalized matrix martingale argument used in the proof of
Lemma~\ref{lem:con} applies to these block increments. After choosing
$\lambda_k\gtrsim d\ln(|\mathcal M|k/\delta)$, there is an event of
probability at least $1-\delta$ on which
\[
\left\|
\bigl(\Sigma_k^{\mathrm{mix}}(\phi)\bigr)^{-1/2}
M_{k}(\phi)
\bigl(\Sigma_k^{\mathrm{mix}}(\phi)\bigr)^{-1/2}
\right\|_{\mathrm{op}}
\le
\frac{1}{2},
\]
simultaneously for all $k\in[K]$ and all $\phi\in\Gamma$. Equivalently,
\[
-\frac{1}{2}\Sigma_k^{\mathrm{mix}}(\phi)
\preceq
M_{k}(\phi)
\preceq
\frac{1}{2}\Sigma_k^{\mathrm{mix}}(\phi).
\]
Since
\(\hat \Sigma_k^{\mathrm{mix}}(\phi)
=\Sigma_k^{\mathrm{mix}}(\phi)+M_k(\phi)\), this immediately yields
\[
\frac{1}{2}\Sigma_k^{\mathrm{mix}}(\phi)
\preceq
\hat \Sigma_k^{\mathrm{mix}}(\phi)
\preceq
\frac{3}{2}\Sigma_k^{\mathrm{mix}}(\phi).
\]
Specializing to $\phi=\hat \phi_k$ and inverting the PSD inequalities gives
\[
\begin{aligned}
\frac{1}{2}\Sigma_k^{\mathrm{mix}}
&\preceq
\hat \Sigma_k^{\mathrm{mix}}
\preceq
\frac{3}{2}\Sigma_k^{\mathrm{mix}},\\
\frac{2}{3}(\Sigma_k^{\mathrm{mix}})^{-1}
&\preceq
(\hat \Sigma_k^{\mathrm{mix}})^{-1}
\preceq
2(\Sigma_k^{\mathrm{mix}})^{-1}.
\end{aligned}
\]
Applying the resulting quadratic-form comparison to
$\hat \phi_k(s,a)$ proves
\[
\begin{aligned}
\sqrt{\frac{2}{3}}
\|\hat \phi_k(s,a)\|_{(\Sigma_k^{\mathrm{mix}})^{-1}}
&\le
\|\hat \phi_k(s,a)\|_{(\hat \Sigma_k^{\mathrm{mix}})^{-1}}
\le
\sqrt{2}\,
\|\hat \phi_k(s,a)\|_{(\Sigma_k^{\mathrm{mix}})^{-1}}.
\end{aligned}
\]
Thus the claim holds with any universal constants
$c_1\le \sqrt{2/3}$ and $c_2\ge \sqrt 2$.
\end{proof}

\begin{proposition}[Concrete tertiary coverage bound]
\label{prop:tri_next_bound}
Under Algorithm~\ref{alg:rep-pd-teg}, for every episode
$k$ and every policy $\pi$,
\[
\mathfrak C_{k,1}^{\mathrm{tri,next}}(\pi)\le \Gamma_k,
\qquad
\mathfrak C_{k,2}^{\mathrm{tri,next}}(\pi)\le \Gamma_k,
\qquad
\mathfrak C_k^{\mathrm{tri,next}}(\pi)\le \Gamma_k.
\]
\end{proposition}

\begin{proof}
Because the schedule is nonincreasing,
\[
\epsilon_i\ge \epsilon_k
\qquad\forall i\le k.
\]
Hence every probing kernel among the first $k$ episodes satisfies
\[
q_i(a\mid s)\ge \frac{\epsilon_k}{|\mathcal A|},
\qquad
q_i'(a\mid s)\ge \frac{\epsilon_k}{|\mathcal A|},
\qquad
q_i''(a\mid s)\ge \frac{\epsilon_k}{|\mathcal A|}.
\]
For the first tertiary coefficient, let
\[
\mu_{k,1}^{\pi}(s,a)
:=
\mathbb E_{\substack{(\tilde s,\tilde a)\sim \bar \nu_k^{\mathrm{pri}}\\
s\sim P^\star(\cdot\mid \tilde s,\tilde a)}}
\bigl[
\pi(a\mid s)\mathbf 1\{s\in ds\}
\bigr]
\]
denote the pushforward state-action measure appearing in the definition of
$\mathfrak C_{k,1}^{\mathrm{tri,next}}(\pi)$. Since
\[
\bar \nu_k^{\mathrm{sec}}(s,a)
=
\frac{1}{k}
\sum_{i=1}^k
\rho_i'(s)q_i'(a\mid s)
\ge
\frac{\epsilon_k}{|\mathcal A|}
\cdot
\frac{1}{k}\sum_{i=1}^k \rho_i'(s),
\]
we have, pointwise,
\[
\mu_{k,1}^{\pi}(s,a)
\le
\frac{|\mathcal A|}{\epsilon_k}
\bar \nu_k^{\mathrm{sec}}(s,a)
=
\Gamma_k \bar \nu_k^{\mathrm{sec}}(s,a).
\]
Therefore for every nonnegative measurable $g^2$,
\[
\mathbb E_{\mu_{k,1}^{\pi}}[g^2]
\le
\Gamma_k
\mathbb E_{\bar \nu_k^{\mathrm{sec}}}[g^2],
\]
which proves
\[
\mathfrak C_{k,1}^{\mathrm{tri,next}}(\pi)\le \Gamma_k.
\]

The proof for $\mathfrak C_{k,2}^{\mathrm{tri,next}}(\pi)$ is identical. If
\[
\mu_{k,2}^{\pi}(s,a)
:=
\mathbb E_{\substack{(\tilde s,\tilde a)\sim \bar \nu_k^{\mathrm{sec}}\\
s\sim P^\star(\cdot\mid \tilde s,\tilde a)}}
\bigl[
\pi(a\mid s)\mathbf 1\{s\in ds\}
\bigr],
\]
then
\[
\bar \nu_k^{\mathrm{ter}}(s,a)
\ge
\frac{\epsilon_k}{|\mathcal A|}
\cdot
\frac{1}{k}\sum_{i=1}^k \rho_i''(s),
\]
while by construction
\[
\mu_{k,2}^{\pi}(s,a)
=
\frac{1}{k}\sum_{i=1}^k
\rho_i''(s)\pi(a\mid s).
\]
Using $\pi(a\mid s)\le 1$ and the lower bound on $\bar \nu_k^{\mathrm{ter}}$,
we obtain the pointwise domination
\[
\mu_{k,2}^{\pi}(s,a)
\le
\Gamma_k \bar \nu_k^{\mathrm{ter}}(s,a).
\]
Therefore, for every nonnegative measurable $g^2$,
\[
\mathbb E_{\mu_{k,2}^{\pi}}[g^2]
\le
\Gamma_k
\mathbb E_{\bar \nu_k^{\mathrm{ter}}}[g^2],
\]
which proves
\[
\mathfrak C_{k,2}^{\mathrm{tri,next}}(\pi)\le \Gamma_k.
\]
Taking the maximum of the two coefficients completes the proof.
\end{proof}

\begin{lemma}[Tertiary-aware mixed one-step-back inequality]
\label{lem:mix_useful2}
Fix an episode $k$ and define
\(f_k(s,a):=\|\hat P_k(\cdot\mid s,a)-P^\star(\cdot\mid s,a)\|_1\) and
\[
\Sigma_k^{\mathrm{mix}}
:=k\mathbb E_{\bar \nu_k^{\mathrm{pri}}}[\hat \phi_k\hat \phi_k^\top]
+k\mathbb E_{\bar \nu_k^{\mathrm{sec}}}[\hat \phi_k\hat \phi_k^\top]
+\lambda_k I.
\]
Also define
\[
\begin{aligned}
M_k^{\mathrm{mix}}(\pi)
&:=\mathbb E_{d_{\hat P_k}^{\pi}}
\Bigl[
\|\hat \phi_k(\tilde s,\tilde a)\|_{(\Sigma_k^{\mathrm{mix}})^{-1}}
\Bigr],\qquad
G_{k,23}(g)
:=\mathbb E_{\bar \nu_k^{\mathrm{sec}}}[g^2]
+\mathbb E_{\bar \nu_k^{\mathrm{ter}}}[g^2].
\end{aligned}
\]
Then for every policy $\pi$ and every measurable
$g:\mathcal S\times\mathcal A\to\mathbb R$ with $\|g\|_\infty\le B$,
\begin{align*}
\left|
\mathbb E_{d_{\hat P_k}^{\pi}}[g]
\right|
\lesssim&
\sqrt{\gamma}\,
M_k^{\mathrm{mix}}(\pi)
\sqrt{
k\mathfrak C_k^{\mathrm{tri,next}}(\pi)\,G_{k,23}(g)
+
\lambda_k d B^2
+
kB^2\zeta_k
}
\\
&+
\sqrt{
(1-\gamma)\mathfrak C_k^{\mathrm{init}}(\pi)\,
\mathbb E_{\bar \nu_k^{\mathrm{pri}}}[g^2]
},
\end{align*}
on the events of Lemmas~\ref{lem:gen_mle} and \ref{lem:pathwise_bridge_mle}.
\end{lemma}

\begin{proof}
Define
\[
\begin{aligned}
m_g(s)&:=\mathbb E_{a\sim \pi(\cdot\mid s)}[g(s,a)],\\
T_{\hat P_k}^{\pi}g(\tilde s,\tilde a)
&:=\mathbb E_{s\sim \hat P_k(\cdot\mid \tilde s,\tilde a)}[m_g(s)],\\
T_\star^{\pi}g(\tilde s,\tilde a)
&:=\mathbb E_{s\sim P^\star(\cdot\mid \tilde s,\tilde a)}[m_g(s)].
\end{aligned}
\]
Since $\|g\|_\infty\le B$, we also have $\|m_g\|_\infty\le B$. The discounted
occupancy recursion under $\hat P_k$ gives
\[
\mathbb E_{d_{\hat P_k}^{\pi}}[g]
=\gamma\mathbb E_{d_{\hat P_k}^{\pi}}[T_{\hat P_k}^{\pi}g]
+(1-\gamma)\mathbb E_{s\sim d_1}[m_g(s)].
\]
We treat the two terms on the right separately.

For the initialization term, write
\[
\mathbb E_{s\sim d_1}[m_g(s)]
=
\sum_{s,a}
d_1(s)\pi(a\mid s)g(s,a).
\]
By the definition of $\mathfrak C_k^{\mathrm{init}}(\pi)$,
\[
(1-\gamma)d_1(s)\pi(a\mid s)
\le
\mathfrak C_k^{\mathrm{init}}(\pi)\,
\bar \nu_k^{\mathrm{pri}}(s,a)
\qquad
\forall (s,a)
\]
on the support of $\bar \nu_k^{\mathrm{pri}}$. Therefore Cauchy-Schwarz gives
\[
(1-\gamma)\left|\mathbb E_{s\sim d_1}[m_g(s)]\right|
\le
\sqrt{
(1-\gamma)\mathfrak C_k^{\mathrm{init}}(\pi)\,
\mathbb E_{\bar \nu_k^{\mathrm{pri}}}[g^2]
}.
\]

For the propagated term, define
\(v_g^{\hat P}:=\int \hat \mu_k(s)m_g(s)\,ds\). Because
\(T_{\hat P_k}^{\pi}g(\tilde s,\tilde a)
=\hat \phi_k(\tilde s,\tilde a)^\top v_g^{\hat P}\), Cauchy-Schwarz under
$(\Sigma_k^{\mathrm{mix}})^{-1}$ yields
\[
\mathbb E_{d_{\hat P_k}^{\pi}}[T_{\hat P_k}^{\pi}g]
\le
M_k^{\mathrm{mix}}(\pi)\,
\|v_g^{\hat P}\|_{\Sigma_k^{\mathrm{mix}}}.
\]
Moreover, \(\|v_g^{\hat P}\|_2\le \sqrt d\,B\), so
\[
\|v_g^{\hat P}\|_{\Sigma_k^{\mathrm{mix}}}^2
=
k\mathbb E_{\bar \nu_k^{\mathrm{pri}}}
\bigl[(T_{\hat P_k}^{\pi}g)^2\bigr]
+k\mathbb E_{\bar \nu_k^{\mathrm{sec}}}
\bigl[(T_{\hat P_k}^{\pi}g)^2\bigr]
+\lambda_k\|v_g^{\hat P}\|_2^2.
\]

We now control the two quadratic terms. By definition of $f_k$,
\[
\bigl|
T_{\hat P_k}^{\pi}g(\tilde s,\tilde a)-T_\star^{\pi}g(\tilde s,\tilde a)
\bigr|
\le
B f_k(\tilde s,\tilde a).
\]
Hence, using $(a+b)^2\le 2a^2+2b^2$,
\[
\begin{aligned}
\mathbb E_{\bar \nu_k^{\mathrm{pri}}}
\bigl[(T_{\hat P_k}^{\pi}g)^2\bigr]
&\lesssim
\mathbb E_{\bar \nu_k^{\mathrm{pri}}}
\bigl[(T_\star^{\pi}g)^2\bigr]
+B^2\mathbb E_{\bar \nu_k^{\mathrm{pri}}}[f_k^2],\\
\mathbb E_{\bar \nu_k^{\mathrm{sec}}}
\bigl[(T_{\hat P_k}^{\pi}g)^2\bigr]
&\lesssim
\mathbb E_{\bar \nu_k^{\mathrm{sec}}}
\bigl[(T_\star^{\pi}g)^2\bigr]
+B^2\mathbb E_{\bar \nu_k^{\mathrm{sec}}}[f_k^2].
\end{aligned}
\]
On the events of Lemmas~\ref{lem:gen_mle} and \ref{lem:pathwise_bridge_mle},
\[
\mathbb E_{\bar \nu_k^{\mathrm{pri}}}[f_k^2]
\lesssim \zeta_k,
\qquad
\mathbb E_{\bar \nu_k^{\mathrm{sec}}}[f_k^2]
\lesssim \zeta_k,
\]
because the primary term is one component of the pathwise average
$\widetilde \nu_k^{\mathrm{all}}$, while the secondary term is provided directly
by Lemma~\ref{lem:pathwise_bridge_mle}.
This is the only place where a pathwise-to-averaged transfer is used: the
generic $g$-dependence remains entirely inside
$G_{k,23}(g)=\mathbb E_{\bar \nu_k^{\mathrm{sec}}}[g^2]
+
\mathbb E_{\bar \nu_k^{\mathrm{ter}}}[g^2]$ and is handled through the
deterministic coverage coefficients.

Next, by Jensen's inequality,
\[
\begin{aligned}
\mathbb E_{\bar \nu_k^{\mathrm{pri}}}
\bigl[(T_\star^{\pi}g)^2\bigr]
&\le \mathbb E_{\mu_{k,1}^{\pi}}[g^2]
\le
\mathfrak C_{k,1}^{\mathrm{tri,next}}(\pi)
\mathbb E_{\bar \nu_k^{\mathrm{sec}}}[g^2],\\
\mathbb E_{\bar \nu_k^{\mathrm{sec}}}
\bigl[(T_\star^{\pi}g)^2\bigr]
&\le \mathfrak C_{k,2}^{\mathrm{tri,next}}(\pi)
\mathbb E_{\bar \nu_k^{\mathrm{ter}}}[g^2].
\end{aligned}
\]
Combining the last four displays and using
\[
\mathfrak C_k^{\mathrm{tri,next}}(\pi)
=
\max\Bigl\{
\mathfrak C_{k,1}^{\mathrm{tri,next}}(\pi),
\mathfrak C_{k,2}^{\mathrm{tri,next}}(\pi)
\Bigr\},
\]
we obtain
\[
\|v_g^{\hat P}\|_{\Sigma_k^{\mathrm{mix}}}^2
\lesssim
k\mathfrak C_k^{\mathrm{tri,next}}(\pi)\,G_{k,23}(g)
+
\lambda_k dB^2
+
kB^2\zeta_k.
\]
Substituting this and the initialization bound into the discounted-occupancy
recursion proves the claimed inequality.
\end{proof}

\begin{lemma}[Direct optimism for the repaired merged bonus]
\label{lem:target_mix_optimism}
Fix an episode $k$ and a policy $\pi$. Define
\[
b_{k,\mathrm{mix,pop}}(s,a)
:=
\alpha_k^{\mathrm{mix}}
\|\hat \phi_k(s,a)\|_{(\Sigma_k^{\mathrm{mix}})^{-1}},
\]
\[
b_{k,\mathrm{mix,emp}}(s,a)
:=
\alpha_k^{\mathrm{emp}}
\|\hat \phi_k(s,a)\|_{(\hat \Sigma_k^{\mathrm{mix}})^{-1}}.
\]
Assume the events of Lemmas~\ref{lem:gen_mle},
\ref{lem:pathwise_bridge_mle}, and the covariance comparison event of
Lemma~\ref{lem:gen_con} all hold.
If
\[
\alpha_k^{\mathrm{mix}}
\gtrsim
\sqrt{
\gamma\Bigl(
k\mathfrak C_k^{\mathrm{tri,next}}(\pi)\zeta_k
+
\lambda_k d
+
k\zeta_k
\Bigr)
},
\]
and
\[
\alpha_k^{\mathrm{emp}}\ge \alpha_k^{\mathrm{mix}}/c_1,
\]
then
\[
V^\pi_{\hat P_k,r+b_{k,\mathrm{mix,emp}}}
-V^\pi_{P^\star,r}
\gtrsim
-\sqrt{
\frac{\mathfrak C_k^{\mathrm{init}}(\pi)\zeta_k}{1-\gamma}
}.
\]
Under the nonbinding-cap condition $\hat b_k=b_{k,\mathrm{mix,emp}}$, the same
bound holds with $b_{k,\mathrm{mix,emp}}$ replaced by the actual algorithmic
bonus $\hat b_k$.
\end{lemma}

\begin{proof}
Apply the first form of the simulation lemma
(Lemma~\ref{lem:performance}) to the pair
\((\hat P_k,r+b_{k,\mathrm{mix,pop}})\) and \((P^\star,r)\).
This gives
\[
\begin{aligned}
&(1-\gamma)
\bigl(
V^\pi_{\hat P_k,r+b_{k,\mathrm{mix,pop}}}\!
-V^\pi_{P^\star,r}
\bigr)
&=
\mathbb E_{d_{\hat P_k}^{\pi}}
\Bigl[
b_{k,\mathrm{mix,pop}}(s,a)
\gamma\mathbb E_{\hat P_k}[V^\pi_{P^\star,r}(s')]
-\gamma\mathbb E_{P^\star}[V^\pi_{P^\star,r}(s')]
\Bigr].
\end{aligned}
\]
Because the total discounted reward is normalized to lie in $[0,1]$, we have
\[
\|V^\pi_{P^\star,r}\|_\infty\le 1,
\qquad
\left|
\mathbb E_{\hat P_k}[V^\pi_{P^\star,r}(s')]
-
\mathbb E_{P^\star}[V^\pi_{P^\star,r}(s')]
\right|
\le
f_k(s,a).
\]
Therefore
\[
\begin{aligned}
&(1-\gamma)
\bigl(
V^\pi_{\hat P_k,r+b_{k,\mathrm{mix,pop}}}
-V^\pi_{P^\star,r}
\bigr)
&\ge
\mathbb E_{d_{\hat P_k}^{\pi}}
\bigl[
b_{k,\mathrm{mix,pop}}(s,a)
\bigr]
-
\gamma
\mathbb E_{d_{\hat P_k}^{\pi}}[f_k].
\end{aligned}
\]

For brevity, define
\[
\Omega_k(\pi):=
k\mathfrak C_k^{\mathrm{tri,next}}(\pi)\zeta_k
+\lambda_k d+k\zeta_k.
\]

Now apply Lemma~\ref{lem:mix_useful2} with $g=f_k$ and $B=2$. On the events
assumed in the lemma,
\[
\begin{aligned}
\mathbb E_{d_{\hat P_k}^{\pi}}[f_k]
\lesssim&
\sqrt{\gamma}\,M_k^{\mathrm{mix}}(\pi)\sqrt{\Omega_k(\pi)}
+
\sqrt{
(1-\gamma)\mathfrak C_k^{\mathrm{init}}(\pi)\zeta_k
}.
\end{aligned}
\]
By the choice of $\alpha_k^{\mathrm{mix}}$,
\(\alpha_k^{\mathrm{mix}}\gtrsim \sqrt{\gamma\Omega_k(\pi)}\),
so the first term on the right is absorbed by
\(\mathbb E_{d_{\hat P_k}^{\pi}}[b_{k,\mathrm{mix,pop}}]
=\alpha_k^{\mathrm{mix}}M_k^{\mathrm{mix}}(\pi)\).
Consequently,
\[
V^\pi_{\hat P_k,r+b_{k,\mathrm{mix,pop}}}
-V^\pi_{P^\star,r}
\gtrsim
-\sqrt{
\frac{\mathfrak C_k^{\mathrm{init}}(\pi)\zeta_k}{1-\gamma}
}.
\]

Next, Lemma~\ref{lem:gen_con} gives
\[
\|\hat \phi_k(s,a)\|_{(\hat \Sigma_k^{\mathrm{mix}})^{-1}}
\ge
c_1
\|\hat \phi_k(s,a)\|_{(\Sigma_k^{\mathrm{mix}})^{-1}},
\]
and therefore \(\alpha_k^{\mathrm{emp}}\ge \alpha_k^{\mathrm{mix}}/c_1\)
implies the pointwise domination
\(b_{k,\mathrm{mix,emp}}(s,a)\ge b_{k,\mathrm{mix,pop}}(s,a)\).
Monotonicity of the simulation lemma in the bonus term then yields
\[
V^\pi_{\hat P_k,r+b_{k,\mathrm{mix,emp}}}
-V^\pi_{P^\star,r}
\gtrsim
-\sqrt{
\frac{\mathfrak C_k^{\mathrm{init}}(\pi)\zeta_k}{1-\gamma}
}.
\]
Finally, under the nonbinding-cap condition,
\(\hat b_k=b_{k,\mathrm{mix,emp}}\), so the same bound holds for the actual
algorithmic bonus.
\end{proof}

\subsection{Final Aggregation}
\label{app:tri_final_aggregation}

We now combine the optimism step with the deployed-to-probe transfer and the
elliptical-potential argument.

\begin{lemma}[Tertiary-aware primal-dual aggregation]
\label{lem:gen_reg_and_vio}
Fix an episode $k$ and define
\[
\begin{aligned}
\Delta_k
&:=\bigl(
V_{P^\star,r}^{\pi^\star}(s_1)
-V_{P^\star,r}^{\pi_k}(s_1)
\bigr)
+
Y_k
\bigl(
V_{P^\star,g}^{\pi^\star}(s_1)
-V_{P^\star,g}^{\pi_k}(s_1)
\bigr),\\
A_k
&:=\mathbb E_{d_{P^\star}^{\pi_k}}
\Bigl[
\|\phi^\star(\tilde s,\tilde a)\|_{(\Sigma_{k,\star}^{\mathrm{pri}})^{-1}}
\Bigr],\\
f_k(s,a)
&:=\|\hat P_k(\cdot\mid s,a)-P^\star(\cdot\mid s,a)\|_1.
\end{aligned}
\]
Assume, for every episode $k$,
\begin{enumerate}
    \item the cap is nonbinding:
    \[
    \bar b\ge \max_{1\le j\le K}\frac{\alpha_j^{\mathrm{emp}}}{\sqrt{\lambda_j}};
    \]
    \item the events of Lemmas~\ref{lem:gen_mle} and
    \ref{lem:pathwise_bridge_mle} hold;
    \item Lemma~\ref{lem:target_mix_optimism} holds with $\pi=\pi^\star$ for
    both the reward objective $r$ and the utility objective $g$.
\end{enumerate}
Then, under Algorithm~\ref{alg:rep-pd-teg},
\begin{align*}
\Delta_k
\lesssim&
\frac{1+\xi}{1-\gamma}\,
A_k
\sqrt{
\gamma d\Gamma_k(\alpha_k^{\mathrm{emp}})^2
+
\gamma\lambda_k d\bar b^2
}
+
\frac{(1+\xi)\bar b}{(1-\gamma)^2}\,
A_k
\sqrt{
\gamma\bigl(k\Gamma_k\zeta_k+\lambda_k d\bigr)
}
\\
&+
\frac{1+\xi}{1-\gamma}\,
\alpha_k^{\mathrm{emp}}
\sqrt{
\frac{\Gamma_k d(1-\gamma)}{k}
}
+
\frac{(1+\xi)\bar b}{(1-\gamma)^2}
\sqrt{
\Gamma_k\zeta_k(1-\gamma)
}
+
(1+\xi)
\sqrt{
\frac{\Gamma_k\zeta_k}{1-\gamma}
}.
\end{align*}
Consequently,
\begin{align*}
\sum_{k=1}^K\Delta_k
\lesssim&
\frac{1+\xi}{1-\gamma}
\sum_{k=1}^K
A_k
\sqrt{
\gamma d\Gamma_k(\alpha_k^{\mathrm{emp}})^2
+
\gamma\lambda_k d\bar b^2
}+
\frac{(1+\xi)\bar b}{(1-\gamma)^2}
\sum_{k=1}^K
A_k
\sqrt{
\gamma\bigl(k\Gamma_k\zeta_k+\lambda_k d\bigr)
}
\\
&+
\frac{1+\xi}{1-\gamma}
\sum_{k=1}^K
\alpha_k^{\mathrm{emp}}
\sqrt{
\frac{\Gamma_k d(1-\gamma)}{k}
}\\
&+
\frac{(1+\xi)\bar b}{(1-\gamma)^2}
\sum_{k=1}^K
\sqrt{
\Gamma_k\zeta_k(1-\gamma)
}+
(1+\xi)
\sum_{k=1}^K
\sqrt{
\frac{\Gamma_k\zeta_k}{1-\gamma}
}.
\end{align*}
\end{lemma}

\begin{proof}
Fix an episode $k$. By the definition of $\Delta_k$,
\[
\Delta_k
=
\bigl(
V_{P^\star,r}^{\pi^\star}-V_{P^\star,r}^{\pi_k}
\bigr)
+
Y_k
\bigl(
V_{P^\star,g}^{\pi^\star}-V_{P^\star,g}^{\pi_k}
\bigr).
\]
Because $0\le Y_k\le \xi$ and Lemma~\ref{lem:target_mix_optimism} holds at
$\pi^\star$ for both the reward and utility objectives,
\[
\begin{aligned}
V_{P^\star,r}^{\pi^\star}
&\le V_{\hat P_k,r+\hat b_k}^{\pi^\star}
+
\sqrt{
\frac{\mathfrak C_k^{\mathrm{init}}(\pi^\star)\zeta_k}{1-\gamma}
},\qquad
V_{P^\star,g}^{\pi^\star}
\le V_{\hat P_k,g+\hat b_k}^{\pi^\star}
+
\sqrt{
\frac{\mathfrak C_k^{\mathrm{init}}(\pi^\star)\zeta_k}{1-\gamma}
}.
\end{aligned}
\]
Since every primary probing distribution has the uniform floor
$\epsilon_k/|\mathcal A|$ and
\(\rho_i(s)=d_{P^\star}^{\pi_{i-1}}(s)\ge (1-\gamma)d_1(s)\), we have
\(\bar \nu_k^{\mathrm{pri}}(s,a)
\ge (1-\gamma)d_1(s)\epsilon_k/|\mathcal A|\), hence
\(\mathfrak C_k^{\mathrm{init}}(\pi^\star)\le \Gamma_k\).
Therefore
\[
\Delta_k
\le
\widetilde \Delta_k
+
(1+\xi)
\sqrt{
\frac{\Gamma_k\zeta_k}{1-\gamma}
},
\]
where
\[
\widetilde \Delta_k
:=
\bigl(
V_{\hat P_k,r+\hat b_k}^{\pi^\star}-V_{P^\star,r}^{\pi_k}
\bigr)
+
Y_k
\bigl(
V_{\hat P_k,g+\hat b_k}^{\pi^\star}-V_{P^\star,g}^{\pi_k}
\bigr).
\]

By the definition of $\pi_k$ as the maximizer of the optimistic empirical
Lagrangian,
\[
V_{\hat P_k,r+\hat b_k}^{\pi^\star}
+
Y_kV_{\hat P_k,g+\hat b_k}^{\pi^\star}
\le
V_{\hat P_k,r+\hat b_k}^{\pi_k}
+
Y_kV_{\hat P_k,g+\hat b_k}^{\pi_k}.
\]
Hence
\[
\widetilde \Delta_k
\le
\bigl(
V_{\hat P_k,r+\hat b_k}^{\pi_k}-V_{P^\star,r}^{\pi_k}
\bigr)
+
Y_k
\bigl(
V_{\hat P_k,g+\hat b_k}^{\pi_k}-V_{P^\star,g}^{\pi_k}
\bigr).
\]

Define
\[
\begin{aligned}
W_{k,r}&:=V_{\hat P_k,r+\hat b_k}^{\pi_k},
&\quad
W_{k,g}&:=V_{\hat P_k,g+\hat b_k}^{\pi_k},\\
\Delta_{\hat P_k,P^\star}W(s,a)
&:=\mathbb E_{\hat P_k}[W(s')]
-\mathbb E_{P^\star}[W(s')].
\end{aligned}
\]

Apply the second form of the simulation lemma
(Lemma~\ref{lem:performance}) to each bracket. For the reward term,
\[
\begin{aligned}
V_{\hat P_k,r+\hat b_k}^{\pi_k}-V_{P^\star,r}^{\pi_k}
&=\frac{1}{1-\gamma}
\mathbb E_{d_{P^\star}^{\pi_k}}
\Bigl[
\hat b_k(s,a)
+\gamma\,\Delta_{\hat P_k,P^\star}W_{k,r}(s,a)
\Bigr].
\end{aligned}
\]
The analogous identity holds for the utility term. Since
$\|\hat b_k\|_\infty\le \bar b$, the normalized reward and utility assumptions
imply, up to universal constants,
\[
\begin{aligned}
\|V_{\hat P_k,r+\hat b_k}^{\pi_k}\|_\infty
&\lesssim
\frac{\bar b}{1-\gamma},
\qquad
\|V_{\hat P_k,g+\hat b_k}^{\pi_k}\|_\infty
\lesssim
\frac{\bar b}{1-\gamma}.
\end{aligned}
\]
Thus
\[
\widetilde \Delta_k
\lesssim
\frac{1+\xi}{1-\gamma}
\mathbb E_{d_{P^\star}^{\pi_k}}[\hat b_k]
+
\frac{(1+\xi)\bar b}{(1-\gamma)^2}
\mathbb E_{d_{P^\star}^{\pi_k}}[f_k].
\]
Combining this with the previous reduction gives
\[
\Delta_k
\lesssim
\frac{1+\xi}{1-\gamma}
\mathbb E_{d_{P^\star}^{\pi_k}}[\hat b_k]
+
\frac{(1+\xi)\bar b}{(1-\gamma)^2}
\mathbb E_{d_{P^\star}^{\pi_k}}[f_k]
+
(1+\xi)
\sqrt{
\frac{\Gamma_k\zeta_k}{1-\gamma}
}.
\]

We now bound the two expectations separately.

For the bonus term, Lemma~\ref{lem:gen_con} implies
\[
\hat b_k(s,a)
\lesssim
\alpha_k^{\mathrm{emp}}
\|\hat \phi_k(s,a)\|_{(\Sigma_k^{\mathrm{mix}})^{-1}}.
\]
Moreover,
\[
k\mathbb E_{\bar \nu_k^{\mathrm{pri}}}
\Bigl[
\|\hat \phi_k\|_{(\Sigma_k^{\mathrm{mix}})^{-1}}^2
\Bigr]
\le d,
\]
and the same bound holds with $\bar \nu_k^{\mathrm{pri}}$ replaced by
$\bar \nu_k^{\mathrm{sec}}$, because each term is upper-bounded by
\[
\mathrm{Tr}
\Bigl(
\bigl(
k\mathbb E_{\bar \nu_k^{\mathrm{pri}}}[\hat \phi_k\hat \phi_k^\top]
+
k\mathbb E_{\bar \nu_k^{\mathrm{sec}}}[\hat \phi_k\hat \phi_k^\top]
\bigr)
(\Sigma_k^{\mathrm{mix}})^{-1}
\Bigr)
\le d.
\]
Applying Lemma~\ref{lem:gen_useful} with $g=\hat b_k$ and $B=\bar b$, and using
\(\mathfrak C_k^{\mathrm{next}}(\pi_k)\le \Gamma_k\) and
\(\mathfrak C_k^{\mathrm{init}}(\pi_k)\le \Gamma_k\),
gives
\[
\mathbb E_{d_{P^\star}^{\pi_k}}[\hat b_k]
\lesssim
A_k
\sqrt{
\gamma d\Gamma_k(\alpha_k^{\mathrm{emp}})^2
+
\gamma\lambda_k d\bar b^2
}
+
\alpha_k^{\mathrm{emp}}
\sqrt{
\frac{\Gamma_k d(1-\gamma)}{k}
}.
\]

For the model-error term, apply Lemma~\ref{lem:gen_useful} with $g=f_k$ and
$B=2$. By Lemmas~\ref{lem:gen_mle} and \ref{lem:pathwise_bridge_mle},
\[
\mathbb E_{\bar \nu_k^{\mathrm{pri}}}[f_k^2]
\lesssim \zeta_k,
\qquad
\mathbb E_{\bar \nu_k^{\mathrm{sec}}}[f_k^2]
\lesssim \zeta_k,
\]
and again
\(\mathfrak C_k^{\mathrm{next}}(\pi_k)\le \Gamma_k\) and
\(\mathfrak C_k^{\mathrm{init}}(\pi_k)\le \Gamma_k\).
Hence
\[
\mathbb E_{d_{P^\star}^{\pi_k}}[f_k]
\lesssim
A_k
\sqrt{
\gamma\bigl(k\Gamma_k\zeta_k+\lambda_k d\bigr)
}
+
\sqrt{
\Gamma_k\zeta_k(1-\gamma)
}.
\]

Substituting the last two displays into the bound on $\Delta_k$ yields the
episode-wise inequality claimed in the lemma. Summing over $k=1,\ldots,K$
proves the cumulative bound.
\end{proof}

\begin{lemma}[Tertiary-aware deployed-to-probe potential reduction]
\label{lem:gen_potential}
Under Algorithm~\ref{alg:rep-pd-teg},
\[
\Psi_K:=\sum_{k=1}^K A_k
\lesssim
\frac{1}{\sqrt{\lambda_1}}
+
\sqrt{
d\ln\!\left(1+\frac{K}{d\lambda_1}\right)
\sum_{k=2}^K \Lambda_k
}.
\]
In particular, since $\Lambda_k\le \Gamma_k$,
\[
\Psi_K
\lesssim
\frac{1}{\sqrt{\lambda_1}}
+
\sqrt{
d\ln\!\left(1+\frac{K}{d\lambda_1}\right)
\sum_{k=2}^K \Gamma_k
}.
\]
\end{lemma}

\begin{proof}
The first episode is trivial:
\[
A_1
\le
\frac{1}{\sqrt{\lambda_1}},
\]
because $(\Sigma_{1,\star}^{\mathrm{pri}})^{-1}\preceq \lambda_1^{-1}I$ and
$\|\phi^\star(s,a)\|_2\le 1$.

Now fix $k\le K-1$. The relation
\[
q_{k+1}(a\mid s)
=(1-\epsilon_{k+1})\pi_k(a\mid s)+\epsilon_{k+1}U(a)
\]
implies the pointwise ratio bound
\[
\frac{\pi_k(a\mid s)}{q_{k+1}(a\mid s)}
\le
\Lambda_{k+1}.
\]
Therefore
\[
A_k
\le
\sqrt{\Lambda_{k+1}}
\sqrt{
\mathbb E_{(s,a)\sim \nu_{k+1}}
\Bigl[
\|\phi^\star(s,a)\|^2_{(\Sigma_{k,\star}^{\mathrm{pri}})^{-1}}
\Bigr]
}.
\]
Indeed,
\[
\begin{aligned}
A_k
&=
\mathbb E_{(s,a)\sim d_{P^\star}^{\pi_k}}
\Bigl[
\|\phi^\star(s,a)\|_{(\Sigma_{k,\star}^{\mathrm{pri}})^{-1}}
\Bigr]
\\
&\le
\sqrt{
\mathbb E_{d_{P^\star}^{\pi_k}}
\Bigl[
\frac{d_{P^\star}^{\pi_k}(s,a)}{\nu_{k+1}(s,a)}
\Bigr]
}
\sqrt{
\mathbb E_{(s,a)\sim \nu_{k+1}}
\Bigl[
\|\phi^\star(s,a)\|^2_{(\Sigma_{k,\star}^{\mathrm{pri}})^{-1}}
\Bigr]
}
\\
&\le
\sqrt{\Lambda_{k+1}}
\sqrt{
\mathbb E_{(s,a)\sim \nu_{k+1}}
\Bigl[
\|\phi^\star(s,a)\|^2_{(\Sigma_{k,\star}^{\mathrm{pri}})^{-1}}
\Bigr]
}.
\end{aligned}
\]

For each $k\le K-1$, define
\[
G_{k+1}
:=
\mathbb E_{(s,a)\sim \nu_{k+1}}
\bigl[
\phi^\star(s,a)\phi^\star(s,a)^\top
\bigr].
\]
Then
\[
\Sigma_{k,\star}^{\mathrm{pri}}
=
\lambda_k I+\sum_{i=1}^k G_i.
\]
Since $\lambda_{k+1}\ge \lambda_k$, we may apply
Lemma~\ref{lem:reduction} to the PSD increments $G_2,\ldots,G_K$ and obtain
\[
\sum_{k=1}^{K-1}
\mathrm{Tr}
\bigl(
G_{k+1}(\Sigma_{k,\star}^{\mathrm{pri}})^{-1}
\bigr)
\lesssim
\ln\det(\Sigma_{K,\star}^{\mathrm{pri}})
-
\ln\det(\lambda_1 I).
\]
Because
\[
\mathbb E_{(s,a)\sim \nu_{k+1}}
\Bigl[
\|\phi^\star(s,a)\|^2_{(\Sigma_{k,\star}^{\mathrm{pri}})^{-1}}
\Bigr]
=
\mathrm{Tr}
\bigl(
G_{k+1}(\Sigma_{k,\star}^{\mathrm{pri}})^{-1}
\bigr),
\]
we deduce
\[
\sum_{k=1}^{K-1}
\mathbb E_{(s,a)\sim \nu_{k+1}}
\Bigl[
\|\phi^\star(s,a)\|^2_{(\Sigma_{k,\star}^{\mathrm{pri}})^{-1}}
\Bigr]
\lesssim
\ln\det(\Sigma_{K,\star}^{\mathrm{pri}})
-
\ln\det(\lambda_1 I).
\]
Lemma~\ref{lem:potential} and the bound $\|\phi^\star(s,a)\|_2\le 1$ imply
\[
\ln\det(\Sigma_{K,\star}^{\mathrm{pri}})
-
\ln\det(\lambda_1 I)
\lesssim
d\ln\!\left(1+\frac{K}{d\lambda_1}\right).
\]
Finally, Cauchy-Schwarz gives
\[
\begin{aligned}
\sum_{k=1}^{K-1}A_k
&\le
\sqrt{
\sum_{k=2}^K \Lambda_k
}
\sqrt{
\sum_{k=1}^{K-1}
\mathbb E_{(s,a)\sim \nu_{k+1}}
\Bigl[
\|\phi^\star(s,a)\|^2_{(\Sigma_{k,\star}^{\mathrm{pri}})^{-1}}
\Bigr]
}
\lesssim
\sqrt{
d\ln\!\left(1+\frac{K}{d\lambda_1}\right)
\sum_{k=2}^K \Lambda_k
}.
\end{aligned}
\]
Adding the bound on $A_1$ proves the result.
\end{proof}

\begin{lemma}[Summable optimistic overestimation of the utility]
\label{lem:tri_overestimation}
Fix an episode $k$ and define
\[
\hat U_k:=V_{\hat P_k,g+\hat b_k}^{\pi_k}(s_1),
\qquad
U_k:=V_{P^\star,g}^{\pi_k}(s_1).
\]
Assume Algorithm~\ref{alg:rep-pd-teg}, the nonbinding-cap
condition, and the events of Lemmas~\ref{lem:gen_mle},
\ref{lem:pathwise_bridge_mle}, and \ref{lem:gen_con}. Then, episode-wise,
\begin{align*}
\bigl[\hat U_k-U_k\bigr]_+
\lesssim\;&
\frac{1}{1-\gamma}
\,A_k
\alpha_k^{\mathrm{emp}}
\sqrt{\gamma d\Gamma_k}
+
\frac{\bar b}{1-\gamma}
\,A_k
\sqrt{\gamma\lambda_k d}
\\
&+
\frac{\bar b}{(1-\gamma)^2}
\,A_k
\sqrt{
\gamma\bigl(k\Gamma_k\zeta_k+\lambda_k d\bigr)
}
\\
&+
\frac{1}{1-\gamma}
\alpha_k^{\mathrm{emp}}
\sqrt{
\frac{\Gamma_k d(1-\gamma)}{k}
}
+
\frac{\bar b}{(1-\gamma)^2}
\sqrt{
\Gamma_k\zeta_k(1-\gamma)
}.
\end{align*}
Consequently, under the tuning of
Proposition~\ref{prop:tri_theorem_level},
\[
\begin{aligned}
\sum_{k=1}^K \bigl[\hat U_k-U_k\bigr]_+
\le
\tilde{\mathcal O}\!\left(
\frac{\bar b\,d^2}{(1-\gamma)^2}
\right)
\Biggl[
\sqrt{
\Gamma_K \sum_{k=2}^K \Lambda_k
}
+
\sum_{k=1}^K \frac{\Gamma_k}{\sqrt{k}}
\Biggr].
\end{aligned}
\]
\end{lemma}

\begin{proof}
Apply the second form of the simulation lemma
(Lemma~\ref{lem:performance}) to the utility objective:
\[
\begin{aligned}
\hat U_k-U_k
&=
\frac{1}{1-\gamma}
\mathbb E_{d_{P^\star}^{\pi_k}}
\!\left[
\hat b_k(s,a)
+\gamma\Delta_{\hat P_k,P^\star}W_{k,g}(s,a)
\right],
\end{aligned}
\]
where
\[
\begin{aligned}
W_{k,g}
&:=V_{\hat P_k,g+\hat b_k}^{\pi_k},\quad
\Delta_{\hat P_k,P^\star}W(s,a)
:=
\mathbb E_{\hat P_k}[W(s')]
-
\mathbb E_{P^\star}[W(s')].
\end{aligned}
\]
Since $\|\hat b_k\|_\infty\le \bar b$ and the utility is normalized to lie in
$[0,1]$, we have
\(\|W_{k,g}\|_\infty\lesssim \bar b/(1-\gamma)\).
Therefore
\[
\begin{aligned}
[\hat U_k-U_k]_+
\le&
\frac{1}{1-\gamma}
\mathbb E_{d_{P^\star}^{\pi_k}}[\hat b_k]
+
\frac{\bar b}{(1-\gamma)^2}
\mathbb E_{d_{P^\star}^{\pi_k}}[f_k].
\end{aligned}
\]
The bonus term is controlled exactly as in the proof of
Lemma~\ref{lem:gen_reg_and_vio}:
\[
\mathbb E_{d_{P^\star}^{\pi_k}}[\hat b_k]
\lesssim
A_k
\sqrt{
\gamma d\Gamma_k(\alpha_k^{\mathrm{emp}})^2
+
\gamma\lambda_k d\bar b^2
}
+
\alpha_k^{\mathrm{emp}}
\sqrt{
\frac{\Gamma_k d(1-\gamma)}{k}
},
\]
and the model-mismatch term obeys
\[
\mathbb E_{d_{P^\star}^{\pi_k}}[f_k]
\lesssim
A_k
\sqrt{
\gamma\bigl(k\Gamma_k\zeta_k+\lambda_k d\bigr)
}
+
\sqrt{
\Gamma_k\zeta_k(1-\gamma)
}.
\]
Substituting the last two displays proves the episode-wise bound.

For the cumulative statement, one substitutes the same choices of
$\alpha_k^{\mathrm{emp}},\lambda_k,\zeta_k$ as in
Proposition~\ref{prop:tri_theorem_level}. The two $A_k$-weighted sums are then
of order
\[
\tilde{\mathcal O}\!\left(
\frac{\bar b\,d^2}{(1-\gamma)^2}
\sqrt{\Gamma_K}\,\Psi_K
\right),
\qquad
\Psi_K:=\sum_{k=1}^K A_k,
\]
while the two remaining sums are of order
\[
\tilde{\mathcal O}\!\left(
\frac{\bar b\,d^2}{(1-\gamma)^2}
\sum_{k=1}^K \frac{\Gamma_k}{\sqrt{k}}
\right).
\]
Finally Lemma~\ref{lem:gen_potential} gives
\[
\sqrt{\Gamma_K}\,\Psi_K
\le
\tilde{\mathcal O}\!\left(
\sqrt{\Gamma_K \sum_{k=2}^K \Lambda_k}
\right),
\]
which proves the cumulative bound.
\end{proof}

\begin{proposition}[Theorem-level regret and soft-violation dependence]
\label{prop:tri_theorem_level}
Assume Algorithm~\ref{alg:rep-pd-teg} and suppose the
high-probability events required by
Lemmas~\ref{lem:target_mix_optimism}, \ref{lem:gen_reg_and_vio}, and
\ref{lem:gen_potential} hold simultaneously. Choose
\[
\alpha_k^{\mathrm{mix}}
=
\alpha_k^{\mathrm{emp}}
=
\tilde{\mathcal O}\!\left(
\sqrt{
(\Gamma_k+d^2)\gamma\ln(|\mathcal M|k/\delta)
}
\right),
\]
\[
\lambda_k
=
\tilde{\Theta}\!\left(
\frac{\Gamma_k+d^2}{d}\,
\ln(|\mathcal M|k/\delta)
\right),
\]
and choose
\[
\eta
=
\frac{2(1-\gamma)}{\theta\bar b\sqrt K},
\qquad
\xi=\tilde{\Theta}(1/\theta),
\]
and assume the cap is nonbinding:
\[
\bar b
\ge
\max_{1\le j\le K}
\frac{\alpha_j^{\mathrm{emp}}}{\sqrt{\lambda_j}}.
\]
For a nonnegative sequence $z_2,\ldots,z_K$, define
\[
\Xi_K(z)
:=
\sqrt{
\Gamma_K \sum_{k=2}^K z_k
}
+
\sum_{k=1}^K \frac{\Gamma_k}{\sqrt{k}}.
\]
Then
\[
\sum_{k=1}^K \Delta_k
\le
\tilde{\mathcal O}\!\left(
\frac{(1+\xi)\bar b\,d^2}{(1-\gamma)^2}
\Xi_K(\Lambda)
\right).
\]
Consequently, taking the comparator dual level to be $Y=0$ yields
\[
\mathrm{Regret}^{\mathrm{return}}(K)
\le
\tilde{\mathcal O}\!\left(
\frac{\bar b\,d^2}{\theta(1-\gamma)^2}
\Xi_K(\Lambda)
\right).
\]
Moreover,
\[
\mathrm{Violation}^{\mathrm{return}}_{\mathrm{soft}}(K)
\le
\tilde{\mathcal O}\!\left(
\frac{\bar b\,d^2}{(1-\gamma)^2}
\Xi_K(\Lambda)
\right).
\]
In particular, since $\Lambda_k\le \Gamma_k=|\mathcal A|/\epsilon_k$,
\[
\mathrm{Regret}^{\mathrm{return}}(K)
\le
\tilde{\mathcal O}\!\left(
\frac{\bar b\,d^2}{\theta(1-\gamma)^2}
\Xi_K(\Gamma)
\right).
\]
Likewise,
\[
\mathrm{Violation}^{\mathrm{return}}_{\mathrm{soft}}(K)
\le
\tilde{\mathcal O}\!\left(
\frac{\bar b\,d^2}{(1-\gamma)^2}
\Xi_K(\Gamma)
\right).
\]
For $\epsilon_k=k^{-p}$ with $p\in(0,1/2)$, this becomes
\[
\begin{aligned}
\mathrm{Regret}^{\mathrm{return}}(K)
&\le
\tilde{\mathcal O}\!\left(K^{1/2+p}\right),\quad
\mathrm{Violation}^{\mathrm{return}}_{\mathrm{soft}}(K)
\le
\tilde{\mathcal O}\!\left(K^{1/2+p}\right).
\end{aligned}
\]
up to problem-dependent prefactors, with the regret bound carrying the extra
$\theta^{-1}$ dependence. Moreover, under the displayed tuning of
$\lambda_k$,
\[
\frac{\alpha_k^{\mathrm{emp}}}{\sqrt{\lambda_k}}
=
\tilde{\mathcal O}(\sqrt d),
\]
so the nonbinding-cap condition is compatible with any
$\bar b=\tilde{\Omega}(\sqrt d)$ and does not introduce any additional
$K$-dependence.
\end{proposition}

\begin{proof}
Start from Lemma~\ref{lem:gen_reg_and_vio} and substitute the tunings
\[
\begin{aligned}
\alpha_k^{\mathrm{emp}}
&=
\tilde{\mathcal O}\!\left(
\sqrt{
(\Gamma_k+d^2)\gamma\ln(|\mathcal M|k/\delta)
}
\right),\quad
\lambda_k
=
\tilde{\Theta}\!\left(
\frac{\Gamma_k+d^2}{d}\,
\ln(|\mathcal M|k/\delta)
\right),
\end{aligned}
\]
together with
\[
\zeta_k
=
\tilde{\mathcal O}(1/k).
\]
Under these choices,
\(\lambda_k d=\tilde{\Theta}((\Gamma_k+d^2)\ln(|\mathcal M|k/\delta))\),
so every place in Lemma~\ref{lem:gen_reg_and_vio} where $\lambda_k d$
appears is of the same order as the dominant $(\Gamma_k+d^2)$ term already
present in $\alpha_k^{\mathrm{emp}}{}^2$. Therefore enlarging $\lambda_k$ to
make the cap nonbinding does not change the final $K$-dependence.

We also check the cap condition explicitly:
\[
\frac{\alpha_k^{\mathrm{emp}}}{\sqrt{\lambda_k}}
=
\tilde{\mathcal O}\!\left(
\sqrt{
\frac{(\Gamma_k+d^2)\ln(|\mathcal M|k/\delta)}
{((\Gamma_k+d^2)/d)\ln(|\mathcal M|k/\delta)}
}
\right)
=
\tilde{\mathcal O}(\sqrt d).
\]
Hence any choice $\bar b=\tilde{\Omega}(\sqrt d)$ is sufficient to ensure the
nonbinding-cap condition uniformly over $k\le K$.

Because $\Gamma_k$ is nondecreasing, the dominant dependence in every
episode-wise term may be upper bounded by $\Gamma_K$ inside the coefficients
that multiply $A_k$. Consequently the two $A_k$-weighted sums in
Lemma~\ref{lem:gen_reg_and_vio} are both of order
\[
\tilde{\mathcal O}\!\left(
\frac{(1+\xi)\bar b\,d^2}{(1-\gamma)^2}
\sqrt{\Gamma_K}\,\Psi_K
\right),
\qquad
\Psi_K:=\sum_{k=1}^K A_k,
\]
The three remaining sums in Lemma~\ref{lem:gen_reg_and_vio} contain no factor
$A_k$; after substituting the same choices of
$\alpha_k^{\mathrm{emp}},\lambda_k,\zeta_k$, each of them is bounded by
\[
\tilde{\mathcal O}\!\left(
\frac{(1+\xi)\bar b\,d^2}{(1-\gamma)^2}
\sum_{k=1}^K \frac{\Gamma_k}{\sqrt{k}}
\right).
\]
Now apply Lemma~\ref{lem:gen_potential}, which gives
\[
\Psi_K
\lesssim
\frac{1}{\sqrt{\lambda_1}}
+
\sqrt{
d\ln\!\left(1+\frac{K}{d\lambda_1}\right)
\sum_{k=2}^K \Lambda_k
}.
\]
After absorbing the harmless initial term $1/\sqrt{\lambda_1}$ into the
polylogarithmic factor, we obtain
\[
\sqrt{\Gamma_K}\,\Psi_K
\le
\tilde{\mathcal O}\!\left(
\sqrt{\Gamma_K\sum_{k=2}^K \Lambda_k}
\right).
\]
This is exactly the quantity $\Xi_K(\Lambda)$ defined in the proposition, so
\[
\sum_{k=1}^K \Delta_k
\le
\tilde{\mathcal O}\!\left(
\frac{(1+\xi)\bar b\,d^2}{(1-\gamma)^2}
\Xi_K(\Lambda)
\right).
\]

The projection step in
Algorithm~\ref{alg:rep-pd-teg} is identical to that of the
main algorithm, so the proof of Lemma~\ref{lem:boundY} applies verbatim and
gives, for every $Y\in[0,\xi]$,
\[
\sum_{k=1}^{K}(Y-Y_k)
\bigl(
b-V_{\hat P_k,g+\hat b_k}^{\pi_k}(s_1)
\bigr)
\le
\frac{Y^2}{2\eta}
+
\frac{\eta \bar b^2K}{2(1-\gamma)^2}.
\]
Setting $Y=0$ and combining this with Lemma~\ref{lem:gen_reg_and_vio} yields
\[
\sum_{k=1}^{K}
\bigl(
V_{P^\star,r}^{\pi^\star}(s_1)-V_{P^\star,r}^{\pi_k}(s_1)
\bigr)
\lesssim
\sum_{k=1}^K \Delta_k
+
\frac{\eta \bar b^2K}{2(1-\gamma)^2}.
\]
Substituting the bound on $\sum_k \Delta_k$ gives a right-hand side
\[
\tilde{\mathcal O}\!\left(
\frac{\bar b\,d^2}{\theta(1-\gamma)^2}
\Xi_K(\Lambda)
\right)
+
\mathcal O\!\left(
\frac{\bar b\sqrt K}{\theta(1-\gamma)}
\right),
\]
$\eta=\frac{2(1-\gamma)}{\theta\bar b\sqrt K}$. The second term is lower order
than the first whenever $\epsilon_k=k^{-p}$ with $p>0$, since then
$\Xi_K(\Lambda)=\tilde{\Omega}(K^{1/2+p})$. This proves the displayed regret
bound.

We now turn to the soft-violation term. Define
\[
\hat U_k:=V_{\hat P_k,g+\hat b_k}^{\pi_k}(s_1),
\qquad
U_k:=V_{P^\star,g}^{\pi_k}(s_1),
\qquad
S_K:=\sum_{k=1}^K (b-U_k).
\]
Because $V_{P^\star,g}^{\pi^\star}(s_1)\ge b$, we have
\[
V_{P^\star,r}^{\pi^\star}(s_1)
-
V_{P^\star,r}^{\pi_k}(s_1)
+Y_k(b-U_k)\le \Delta_k
\qquad\forall k\in[K].
\]
Hence, for every $Y\in[0,\xi]$,
\[
\begin{aligned}
\mathrm{Regret}^{\mathrm{return}}(K)+Y S_K
\le&
\sum_{k=1}^K \Delta_k
+
\sum_{k=1}^K
(Y-Y_k)(b-\hat U_k)
+
\sum_{k=1}^K
(Y-Y_k)(\hat U_k-U_k).
\end{aligned}
\]
Choose $Y=\xi$ when $S_K>0$ and $Y=0$ otherwise. Then
\[
\mathrm{Regret}^{\mathrm{return}}(K)+\xi [S_K]_+
\le
\sum_{k=1}^K \Delta_k
+
\frac{\xi^2}{2\eta}
+
\frac{\eta \bar b^2K}{2(1-\gamma)^2}
+
\xi \sum_{k=1}^K [\hat U_k-U_k]_+.
\]
Now use Lemma~\ref{lem:tri_overestimation} together with the bound on
$\sum_k\Delta_k$. This gives
\[
\mathrm{Regret}^{\mathrm{return}}(K)+\xi [S_K]_+
\le
\tilde{\mathcal O}\!\left(
\frac{\xi\bar b\,d^2}{(1-\gamma)^2}
\Xi_K(\Lambda)
\right)
+
\mathcal O\!\left(
\frac{\xi\bar b\sqrt K}{1-\gamma}
\right).
\]
By linearity of discounted occupancy measures, there exists a policy $\pi'$
whose reward and utility values equal the episode averages:
\[
V_{P^\star,r}^{\pi'}(s_1)
=
\frac{1}{K}\sum_{k=1}^K V_{P^\star,r}^{\pi_k}(s_1),
\qquad
V_{P^\star,g}^{\pi'}(s_1)
=
\frac{1}{K}\sum_{k=1}^K V_{P^\star,g}^{\pi_k}(s_1).
\]
Therefore
\[
\begin{aligned}
V_{P^\star,r}^{\pi^\star}(s_1)
-
V_{P^\star,r}^{\pi'}(s_1)
&+
\xi
\bigl[
b-V_{P^\star,g}^{\pi'}(s_1)
\bigr]_+
\le
\frac{1}{K}
\Bigl(
\mathrm{Regret}^{\mathrm{return}}(K)+\xi [S_K]_+
\Bigr).
\end{aligned}
\]
By Lemma~\ref{lem:boundY}, the optimal dual variable satisfies
$Y^\star\le 1/\theta$, so the constant hidden in
$\xi=\tilde\Theta(1/\theta)$ may be chosen so that $\xi\ge 2Y^\star$.
Applying Lemma~\ref{lem:convex} yields
\[
[S_K]_+
\le
\frac{2}{\xi}
\Bigl(
\mathrm{Regret}^{\mathrm{return}}(K)+\xi [S_K]_+
\Bigr).
\]
Since $\xi^{-1}=\tilde{\mathcal O}(\theta)$ and
$\Xi_K(\Lambda)=\tilde{\Omega}(\sqrt K)$, the lower-order
$\mathcal O(\xi\bar b\sqrt K/(1-\gamma))$ term is absorbed, and we obtain
\[
\mathrm{Violation}^{\mathrm{return}}_{\mathrm{soft}}(K)
=
[S_K]_+
\le
\tilde{\mathcal O}\!\left(
\frac{\bar b\,d^2}{(1-\gamma)^2}
\Xi_K(\Lambda)
\right).
\]
This proves the displayed soft-violation bound.
\end{proof}

\begin{proposition}[Execution policy corollary for a stationary-mixture variant]
\label{prop:exec_stationary_mix}
The guarantees in Proposition~\ref{prop:tri_theorem_level} are stated for the
return policy sequence \(\{\pi_k\}_{k=1}^K\) under the discounted-occupancy
oracle protocol. The present proposition records a \emph{separate} fully online
interpretation in which episode \(k\) executes the stationary mixed policy
\[
\tilde \pi_k(\cdot\mid s)
:=(1-\epsilon_k)\pi_k(\cdot\mid s)+\epsilon_k U(\mathcal A).
\]
This online mixture is not the probing protocol analyzed above; it is included
only to translate return policy guarantees into execution policy guarantees.

Define
\[
\begin{aligned}
\mathrm{Regret}^{\mathrm{exec}}(K)
&:=\sum_{k=1}^K
\bigl(V_r^{\pi^\star}(s_1)-V_r^{\tilde \pi_k}(s_1)\bigr),\\
\mathrm{Violation}^{\mathrm{exec}}_{\mathrm{soft}}(K)
&:=\left[
\sum_{k=1}^K
\bigl(b-V_g^{\tilde \pi_k}(s_1)\bigr)
\right]_+ .
\end{aligned}
\]
Then
\[
\begin{aligned}
\mathrm{Regret}^{\mathrm{exec}}(K)
&\le
\mathrm{Regret}^{\mathrm{return}}(K)+\frac{2}{1-\gamma}\sum_{k=1}^K\epsilon_k,\\
\mathrm{Violation}^{\mathrm{exec}}_{\mathrm{soft}}(K)
&\le
\mathrm{Violation}^{\mathrm{return}}_{\mathrm{soft}}(K)+\frac{2}{1-\gamma}\sum_{k=1}^K\epsilon_k.
\end{aligned}
\]
In particular, if \(\epsilon_k=k^{-p}\) with \(p\in [1/4,1/2)\),
then Proposition~\ref{prop:tri_theorem_level} implies
\[
\mathrm{Regret}^{\mathrm{exec}}(K),
\ \mathrm{Violation}^{\mathrm{exec}}_{\mathrm{soft}}(K)
\le
\tilde{\mathcal O}\!\left(K^{1/2+p}\right),
\]
which matches the asymptotic orders of the return policy bounds.
\end{proposition}

\begin{proof}
Fix an episode \(k\). For \(j\in\{r,g\}\), define
\(A_j^{\pi_k}(s,a):=Q_j^{\pi_k}(s,a)-V_j^{\pi_k}(s)\).
By the trajectory-level normalization introduced in the background section,
\(|A_j^{\pi_k}(s,a)|\le 1\) for both \(j=r\) and \(j=g\). The
performance-difference lemma with normalized discounted occupancy measure gives
\[
\begin{aligned}
&(1-\gamma)
\bigl(
V_j^{\tilde \pi_k}(s_1)-V_j^{\pi_k}(s_1)
\bigr)
=
\mathbb E_{s\sim d_{P^\star}^{\tilde \pi_k}}
\Bigl[
\sum_{a\in\mathcal A}
\bigl(
\tilde \pi_k(a\mid s)-\pi_k(a\mid s)
\bigr)
A_j^{\pi_k}(s,a)
\Bigr].
\end{aligned}
\]
Hence
\[
D_{k,j}(s)
:=
\sum_{a\in\mathcal A}
\bigl|
\tilde \pi_k(a\mid s)-\pi_k(a\mid s)
\bigr|
\cdot
\bigl|
A_j^{\pi_k}(s,a)
\bigr|.
\]
Therefore
\[
\begin{aligned}
\bigl|
V_j^{\tilde \pi_k}(s_1)-V_j^{\pi_k}(s_1)
\bigr|
\le&
\frac{1}{1-\gamma}
\mathbb E_{s\sim d_{P^\star}^{\tilde \pi_k}}
\bigl[
D_{k,j}(s)
\bigr]
\le
\frac{1}{1-\gamma}
\mathbb E_{s\sim d_{P^\star}^{\tilde \pi_k}}
\Bigl[
\|
\tilde \pi_k(\cdot\mid s)-\pi_k(\cdot\mid s)
\|_1
\Bigr].
\end{aligned}
\]
Now
\[
\tilde \pi_k(\cdot\mid s)-\pi_k(\cdot\mid s)
=\epsilon_k\bigl(U(\mathcal A)-\pi_k(\cdot\mid s)\bigr),
\quad
\|
\tilde \pi_k(\cdot\mid s)-\pi_k(\cdot\mid s)
\|_1
\le
\epsilon_k
\|U(\mathcal A)-\pi_k(\cdot\mid s)\|_1
\le
2\epsilon_k.
\]
Substituting yields, for both \(j=r\) and \(j=g\),
\[
\bigl|
V_j^{\tilde \pi_k}(s_1)-V_j^{\pi_k}(s_1)
\bigr|
\le
\frac{2\epsilon_k}{1-\gamma}.
\]

Applying this bound with \(j=r\), we obtain
\[
\begin{aligned}
\mathrm{Regret}^{\mathrm{exec}}(K)
=&
\sum_{k=1}^K
\bigl(
V_r^{\pi^\star}(s_1)-V_r^{\pi_k}(s_1)
\bigr)
+
\sum_{k=1}^K
\bigl(
V_r^{\pi_k}(s_1)-V_r^{\tilde \pi_k}(s_1)
\bigr)
\\
\le&
\mathrm{Regret}^{\mathrm{return}}(K)
+
\frac{2}{1-\gamma}
\sum_{k=1}^K \epsilon_k.
\end{aligned}
\]

Applying the same bound with \(j=g\), together with
\([x+y]_+\le [x]_+ + y\) for \(y\ge 0\), yields
\[
\begin{aligned}
\mathrm{Violation}^{\mathrm{exec}}_{\mathrm{soft}}(K)
=&
\left[
\sum_{k=1}^K
\bigl(
b-V_g^{\tilde \pi_k}(s_1)
\bigr)
\right]_+
\\
\le&
\left[
\sum_{k=1}^K
\bigl(
b-V_g^{\pi_k}(s_1)
\bigr)
\right]_+
+
\sum_{k=1}^K
\bigl|
V_g^{\tilde \pi_k}(s_1)-V_g^{\pi_k}(s_1)
\bigr|
\\
\le&
\mathrm{Violation}^{\mathrm{return}}_{\mathrm{soft}}(K)
+
\frac{2}{1-\gamma}
\sum_{k=1}^K \epsilon_k.
\end{aligned}
\]

Finally, if \(\epsilon_k=k^{-p}\) with \(p\in [1/4,1/2)\), then
Proposition~\ref{prop:tri_theorem_level} gives
\[
\mathrm{Regret}^{\mathrm{return}}(K)\le\tilde{\mathcal O}(K^{1/2+p}),
\qquad
\sum_{k=1}^K \epsilon_k
=\sum_{k=1}^K k^{-p}
=\mathcal O(K^{1-p}).
\]
Since \(p\ge 1/4\) is equivalent to \(1-p\le 1/2+p\),
the additional execution term is of no larger order than the returned-policy
terms. Hence
\(\mathrm{Regret}^{\mathrm{exec}}(K)\le\tilde{\mathcal O}(K^{1/2+p})\).
The same comparison applies to
\(\mathrm{Violation}^{\mathrm{exec}}_{\mathrm{soft}}(K)\), because the
returned-policy soft-violation term is already of order
\(\tilde{\mathcal O}(K^{1/2+p})\). Therefore
\(\mathrm{Violation}^{\mathrm{exec}}_{\mathrm{soft}}(K)
\le\tilde{\mathcal O}(K^{1/2+p})\).
\end{proof}

%% file: appendix/uniform_proof.tex
\section{Uniform Exploration as a Tertiary-Aware Repaired Variant}
\label{app:uniform_repaired}

This appendix specializes the decaying-\(\epsilon\) analysis in
Appendix~\ref{app:general_probing} to the constant schedule
\(\epsilon_k\equiv1\). This case is useful as a clean benchmark: all probe
actions are drawn uniformly, while the tertiary probe level and merged bonus
covariance are retained so that the repaired optimism proof still applies.

Throughout this appendix, we set
\[
q_k(\cdot\mid s)=q_k'(\cdot\mid s)=q_k''(\cdot\mid s)=U(\mathcal A),
\qquad \forall k\in[K].
\]
Equivalently, this is Algorithm~\ref{alg:rep-pd-teg} with
\(\epsilon_k\equiv1\). Therefore
\[
\Gamma_k=\Lambda_k=|\mathcal A|,
\qquad \forall k\in[K].
\]

\subsection{Uniform Specialization}

For each round $i\in[K]$, define the primary, secondary, and tertiary state
marginals by
\[
\rho_i(s):=d_{P^\star}^{\pi_{i-1}}(s),\quad
\rho_i'(s'):=\mathbb E_{\rho_i\times U(\mathcal A)}
P^\star(s'\mid s,a),\quad
\rho_i''(s''):=\mathbb E_{\rho_i'\times U(\mathcal A)}
P^\star(s''\mid s,a).
\]
The corresponding probing distributions are
\[
\begin{aligned}
\nu_i(s,a)
&:=\rho_i(s)U(a),
\quad
\nu_i'(s,a)
:=\rho_i'(s)U(a),
\quad
\nu_i''(s,a)
:=\rho_i''(s)U(a).
\end{aligned}
\]
We write
\[
\bar \nu_k^{\mathrm{pri}}
:=
\frac{1}{k}\sum_{i=1}^k \nu_i,
\qquad
\bar \nu_k^{\mathrm{sec}}
:=
\frac{1}{k}\sum_{i=1}^k \nu_i',
\qquad
\bar \nu_k^{\mathrm{ter}}
:=
\frac{1}{k}\sum_{i=1}^k \nu_i''.
\]

No separate pseudocode is needed: this is exactly
Algorithm~\ref{alg:rep-pd-teg} with the constant schedule
\(\epsilon_k\equiv1\). Relative to the general algorithm, the only change is
that all probing kernels are fixed to \(U(\mathcal A)\); the tertiary tuple,
merged covariance, model fitting step, and primal-dual update remain the same.

\subsection{Supporting Specializations}

The tuple-level MLE event, the pathwise-to-averaged bridge for the squared
MLE-mismatch class, and the merged covariance comparison from
Appendix~\ref{app:general_probing} apply verbatim to
Algorithm~\ref{alg:rep-pd-teg} with \(\epsilon_k\equiv1\). The only new point is that all
coverage coefficients collapse to explicit constants.

\begin{proposition}[Uniform coverage specialization]
\label{prop:uniform_tri_next}
Under Algorithm~\ref{alg:rep-pd-teg} with \(\epsilon_k\equiv1\), for every episode $k$ and
every policy $\pi$,
\[
\mathfrak C_k^{\mathrm{init}}(\pi)\le |\mathcal A|,
\qquad
\mathfrak C_k^{\mathrm{next}}(\pi)\le |\mathcal A|,
\]
and
\[
\mathfrak C_{k,1}^{\mathrm{tri,next}}(\pi)\le |\mathcal A|,
\qquad
\mathfrak C_{k,2}^{\mathrm{tri,next}}(\pi)\le |\mathcal A|,
\qquad
\mathfrak C_k^{\mathrm{tri,next}}(\pi)\le |\mathcal A|.
\]
\end{proposition}

\begin{proof}
Because every probe kernel is uniform,
\[
\bar \nu_k^{\mathrm{pri}}(s,a)
=
\frac{1}{|\mathcal A|}
\cdot
\frac{1}{k}\sum_{i=1}^k \rho_i(s),
\]
and every discounted occupancy satisfies
\[
\rho_i(s)=d_{P^\star}^{\pi_{i-1}}(s)\ge (1-\gamma)d_1(s).
\]
Hence
\[
\bar \nu_k^{\mathrm{pri}}(s,a)
\ge
\frac{(1-\gamma)d_1(s)}{|\mathcal A|},
\]
which implies
\[
\mathfrak C_k^{\mathrm{init}}(\pi)
\le
|\mathcal A|
\]
for every policy $\pi$.

The bound on $\mathfrak C_k^{\mathrm{next}}(\pi)$ follows exactly as in the
uniform version of the original proof: the next-step pushforward state-action
measure has density at most $|\mathcal A|$ times
$\bar \nu_k^{\mathrm{sec}}$, because the action component of
$\bar \nu_k^{\mathrm{sec}}$ is uniform.

The tertiary coefficients are the specialization of
Proposition~\ref{prop:tri_next_bound} to the constant schedule
$\epsilon_k\equiv 1$, for which $\Gamma_k=|\mathcal A|$.
\end{proof}

\begin{proposition}[Uniform potential reduction]
\label{prop:uniform_potential}
Under Algorithm~\ref{alg:rep-pd-teg} with \(\epsilon_k\equiv1\),
\[
\Psi_K:=\sum_{k=1}^K A_k
\lesssim
\frac{1}{\sqrt{\lambda_1}}
+
\sqrt{
|\mathcal A|\,dK\ln\!\left(1+\frac{K}{d\lambda_1}\right)
}.
\]
\end{proposition}

\begin{proof}
This is the specialization of Lemma~\ref{lem:gen_potential} to
$\Lambda_k=|\mathcal A|$ for all $k$:
\[
\Psi_K
\lesssim
\frac{1}{\sqrt{\lambda_1}}
+
\sqrt{
d\ln\!\left(1+\frac{K}{d\lambda_1}\right)
\sum_{k=2}^K \Lambda_k
}.
\]
Since $\sum_{k=2}^K\Lambda_k=(K-1)|\mathcal A|$, the claim follows.
\end{proof}

\subsection{Theorem-Level Derivation}

\begin{theorem}[Return policy regret and soft violation for \(p=0\)]
\label{thm:uniform_repaired}
Assume Algorithm~\ref{alg:rep-pd-teg} is run with \(\epsilon_k\equiv1\), and suppose the
high-probability events required by
Lemmas~\ref{lem:target_mix_optimism}, \ref{lem:gen_reg_and_vio}, and
\ref{lem:gen_potential} hold simultaneously. Choose
\[
\alpha_k^{\mathrm{mix}}
=
\alpha_k^{\mathrm{emp}}
=
\tilde{\mathcal O}\!\left(
\sqrt{
(|\mathcal A|+d^2)\gamma\ln(|\mathcal M|k/\delta)
}
\right),
\]
\[
\lambda_k
=
\tilde{\Theta}\!\left(
\frac{|\mathcal A|+d^2}{d}\,
\ln(|\mathcal M|k/\delta)
\right),
\]
and
\[
\eta
=
\frac{2(1-\gamma)}{\theta\bar b\sqrt K},
\qquad
\xi=\tilde{\Theta}(1/\theta),
\]
and assume the cap is nonbinding:
\[
\bar b
\ge
\max_{1\le j\le K}
\frac{\alpha_j^{\mathrm{emp}}}{\sqrt{\lambda_j}}.
\]
Then, with probability at least $1-\delta$,
\[
\mathrm{Regret}^{\mathrm{return}}(K)
\le
\tilde{\mathcal O}\!\left(
\frac{|\mathcal A|\bar b\,d^2\sqrt K}{\theta(1-\gamma)^2}
\right),\quad
\mathrm{Violation}^{\mathrm{return}}_{\mathrm{soft}}(K)
\le
\tilde{\mathcal O}\!\left(
\frac{|\mathcal A|\bar b\,d^2\sqrt K}{(1-\gamma)^2}
\right).
\]
\end{theorem}

\begin{proof}
This is a direct specialization of
Proposition~\ref{prop:tri_theorem_level} to the constant schedule
$\epsilon_k\equiv 1$. Under this choice,
\[
\Gamma_k=\Lambda_k=|\mathcal A|,
\qquad \forall k\in[K].
\]
Therefore the quantity
\[
\Xi_K(\Lambda)
:=
\sqrt{
\Gamma_K \sum_{k=2}^K \Lambda_k
}
+
\sum_{k=1}^K \frac{\Gamma_k}{\sqrt{k}}
\]
satisfies
\[
\Xi_K(\Lambda)
=
\sqrt{
|\mathcal A|\cdot (K-1)|\mathcal A|
}
+
|\mathcal A|\sum_{k=1}^K \frac{1}{\sqrt{k}}
\lesssim
|\mathcal A|\sqrt K.
\]
Substituting this into Proposition~\ref{prop:tri_theorem_level} yields
\[
\mathrm{Regret}^{\mathrm{return}}(K)
\le
\tilde{\mathcal O}\!\left(
\frac{\bar b\,d^2}{\theta(1-\gamma)^2}
\Xi_K(\Lambda)
\right)
\le
\tilde{\mathcal O}\!\left(
\frac{|\mathcal A|\bar b\,d^2\sqrt K}{\theta(1-\gamma)^2}
\right),
\]
and
\[
\mathrm{Violation}^{\mathrm{return}}_{\mathrm{soft}}(K)
\le
\tilde{\mathcal O}\!\left(
\frac{\bar b\,d^2}{(1-\gamma)^2}
\Xi_K(\Lambda)
\right)
\le
\tilde{\mathcal O}\!\left(
\frac{|\mathcal A|\bar b\,d^2\sqrt K}{(1-\gamma)^2}
\right).
\]

It remains only to note that the displayed tuning of $\lambda_k$ makes the cap
condition compatible with a $K$-independent choice of $\bar b$. Indeed,
\[
\frac{\alpha_k^{\mathrm{emp}}}{\sqrt{\lambda_k}}
=
\tilde{\mathcal O}\!\left(
\sqrt{
\frac{(|\mathcal A|+d^2)\ln(|\mathcal M|k/\delta)}
{((|\mathcal A|+d^2)/d)\ln(|\mathcal M|k/\delta)}
}
\right)
=
\tilde{\mathcal O}(\sqrt d),
\]
so any $\bar b=\tilde\Omega(\sqrt d)$ is sufficient, exactly as in
Appendix~\ref{app:general_probing}.
\end{proof}

\paragraph{Constraint discussion.}
The regret proof uses the same representation-learning and optimism ingredients
as Appendix~\ref{app:general_probing}. For soft violation, one additional
ingredient is needed: the cumulative overestimation lemma controls how much the
optimistic utility used by the dual update can exceed the true utility. This
term is summable under the repaired tertiary design, so the standard projected
dual argument yields the stated soft-violation bound.

\begin{remark}[Relation to the \(p=0\) specialization]
The uniform protocol is exactly the \(\epsilon_k\equiv 1\) specialization of
Algorithm~\ref{alg:rep-pd-teg} in Section~\ref{sec:online}. Consequently, this
appendix gives the fully explicit return-policy theorem for the \(p=0\)
schedule, with the familiar
$\tilde{\mathcal O}(|\mathcal A|d^2\sqrt K)$ dependence up to the same
problem-dependent prefactors and logarithmic terms.
\end{remark}

%% file: appendix/hard_general_probing.tex
\section{Hard Constraints Under Decaying Tertiary Probing}
\label{app:hard_general_probing}

This appendix gives the repaired hard-violation analysis.  The key distinction
is between the one-sided correction used by the algorithm and the two-sided
width used only in the proof.  In particular, we do not assume an external
certificate event and we do not require a summable radius uniformly over all
policies queried by the inner loop.

\subsection{Algorithm and Certificate}

The representation-learning front end is the same as in
Algorithm~\ref{alg:rep-pd-teg}.  At episode \(k\), let
\[
q_k=q'_k=q''_k
:=(1-\epsilon_k)\pi_{k-1}+\epsilon_kU(\mathcal A),
\qquad
\Gamma_k:=\frac{|\mathcal A|}{\epsilon_k},
\]
and let
\[
\zeta_k:=\frac{\ln(|\mathcal M|k/\delta)}{k}.
\]
After fitting \(\widehat P_k=(\widehat\mu_k,\widehat\phi_k)\) on the three
streams and forming \(\widehat\Sigma_k^{\rm mix}\), define
\[
\widehat b_k(s,a)
:=
\min\!\left\{
\alpha_k
\|\widehat\phi_k(s,a)\|_{(\widehat\Sigma_k^{\rm mix})^{-1}},
\bar b
\right\}.
\]
For a sufficiently large universal calibration constant \(C_{\rm cert}\), set
\[
\ell_k
:=
C_{\rm cert}
\sqrt{\frac{\Gamma_k\zeta_k}{1-\gamma}},
\qquad
\underline{\mathcal U}_k(\pi)
:=
\left[
V^\pi_{\widehat P_k,g-\widehat b_k}(s_1)-\ell_k
\right]_+.
\tag{G.1}
\]
The quantity \(\ell_k\) is the one-sided statistical correction used by the
algorithm.  For analysis, define the policy-dependent selected width
\[
w_k(\pi)
:=
2\ell_k+2V^\pi_{\widehat P_k,\widehat b_k}(s_1).
\tag{G.2}
\]
Unlike the correction \(\ell_k\), the width \(w_k(\pi)\) is not required as an
algorithmic input.

Let \(Y_{\max}:=K^{1/4}\), and use the capped grid
\[
\mathcal Y_K
:=
\{0,\eta_h,2\eta_h,\ldots\}\cap[0,Y_{\max}]
\ \cup\ \{Y_{\max}\}.
\]
Thus the cap itself is always queried even when \(Y_{\max}/\eta_h\) is not an
integer.

\begin{algorithm}[t]
\caption{REP-PD-H-TEG with an explicit one-sided certificate}
\label{alg:rep-pd-hard-teg}
\label{alg:hard_teg_lc}
\label{alg:hard}
\begin{algorithmic}[1]
\STATE {\bf Input:} Parameters
\(\{\lambda_k,\alpha_k,\epsilon_k\}_{k=1}^K\), confidence \(\delta\),
calibration \(C_{\rm cert}\), temperature \(\tau_h\), grid step \(\eta_h\),
model class \(\mathcal M\), and episodes \(K\).
\STATE Initialize \(\pi_0(\cdot\mid s)=U(\mathcal A)\) and
\(D_0=D'_0=D''_0=\emptyset\).
\FOR{\(k=1,\ldots,K\)}
    \STATE Set
    \(q_k=q'_k=q''_k=(1-\epsilon_k)\pi_{k-1}
    +\epsilon_kU(\mathcal A)\).
    \STATE Query one tertiary tuple and update \(D_k,D'_k,D''_k\) exactly as
    in Algorithm~\ref{alg:rep-pd-teg}.
    \STATE Fit
    \[
    \widehat P_k=(\widehat\mu_k,\widehat\phi_k)
    \in
    \arg\max_{(\mu,\phi)\in\mathcal M}
    \mathbb E_{D_k+D'_k+D''_k}
    [\log\mu(x')^\top\phi(x,u)].
    \]
    \STATE Form \(\widehat\Sigma_k^{\rm mix}\) and \(\widehat b_k\) as above,
    and set
    \[
    \ell_k=C_{\rm cert}
    \sqrt{\frac{|\mathcal A|
    \ln(|\mathcal M|k/\delta)}
    {\epsilon_k k(1-\gamma)}}.
    \]
    \STATE Set \(\mathrm{found}\leftarrow\mathrm{false}\).
    \FOR{\(Y\) in the increasing grid \(\mathcal Y_K\)}
        \STATE Compute
        \[
        \widehat\pi_{k,Y}
        \in
        \arg\max_{\pi}
        \left(
        Q^\pi_{\widehat P_k,r+\widehat b_k}
        +YQ^\pi_{\widehat P_k,g+\widehat b_k}
        \right).
        \]
        \STATE Form the statewise softmax candidate
        \[
        \pi_{k,Y}(\cdot\mid x)
        =
        \operatorname{softmax}_{\tau_h}
        \left(
        Q^{\widehat\pi_{k,Y}}_{\widehat P_k,r+\widehat b_k}(x,\cdot)
        +YQ^{\widehat\pi_{k,Y}}_{\widehat P_k,g+\widehat b_k}(x,\cdot)
        \right).
        \]
        \STATE Compute
        \[
        \underline{\mathcal U}_k(\pi_{k,Y})
        =
        \left[
        V^{\pi_{k,Y}}_{\widehat P_k,g-\widehat b_k}(s_1)
        -\ell_k
        \right]_+.
        \]
        \IF{\(\underline{\mathcal U}_k(\pi_{k,Y})\ge b\)}
            \STATE Set
            \((\pi_k,Y_k)\leftarrow(\pi_{k,Y},Y)\) and
            \(\mathrm{found}\leftarrow\mathrm{true}\).
            \STATE \textbf{break}
        \ENDIF
    \ENDFOR
    \IF{\(\mathrm{found}=\mathrm{false}\)}
        \STATE Set
        \((\pi_k,Y_k)\leftarrow
        (\pi_{k,Y_{\max}},Y_{\max})\).
    \ENDIF
\ENDFOR
\STATE {\bf Return} \(\pi_1,\ldots,\pi_K\).
\end{algorithmic}
\end{algorithm}

\subsection{Derived Certificate Event}

Let \(\mathcal E_{\rm rep}\) be the intersection of the tuple-level MLE,
pathwise bridge, and covariance-comparison events in
Lemmas~\ref{lem:gen_mle}, \ref{lem:pathwise_bridge_mle}, and
\ref{lem:gen_con}.  Replacing \(\delta\) by a constant fraction in the
individual events makes
\(\mathbb P(\mathcal E_{\rm rep})\ge1-\delta\).

\begin{lemma}[One-sided certificate and selected-policy interval]
\label{lem:hard_derived_certificate}
Assume \(\mathcal E_{\rm rep}\), the nonbinding-cap condition
\[
\bar b\ge\max_{j\le K}\frac{\alpha_j}{\sqrt{\lambda_j}},
\]
and the parameter lower bound on \(\alpha_k\) in
Lemma~\ref{lem:target_mix_optimism} with
\(\mathfrak C_k^{\rm tri,next}(\pi)\le\Gamma_k\).
For a sufficiently large \(C_{\rm cert}\), simultaneously for every
episode \(k\) and every policy \(\pi\),
\[
\underline{\mathcal U}_k(\pi)
\le
V^\pi_{P^\star,g}(s_1)
\le
\underline{\mathcal U}_k(\pi)+w_k(\pi).
\tag{G.3}
\]
\end{lemma}

\begin{proof}
We first verify that the residual in the one-step-back argument is uniform in
\(\pi\).  Since every primary occupancy contains the initial component
\((1-\gamma)d_1\), and every probing kernel up to episode \(k\) has action
floor \(\epsilon_k/|\mathcal A|\),
\[
\bar\nu_k^{\rm pri}(s,a)
\ge
(1-\gamma)d_1(s)\frac{\epsilon_k}{|\mathcal A|}.
\]
Consequently
\[
\mathfrak C_k^{\rm init}(\pi)
=
\mathsf D\!\left(
(1-\gamma)d_1\otimes\pi
\;\middle\|\;
\bar\nu_k^{\rm pri}
\right)
\le\Gamma_k
\]
for every policy \(\pi\).  Proposition~\ref{prop:tri_next_bound} likewise
gives
\(\mathfrak C_k^{\rm tri,next}(\pi)\le\Gamma_k\) uniformly in \(\pi\).

Write \(f_k(s,a)=
\|\widehat P_k(\cdot\mid s,a)-P^\star(\cdot\mid s,a)\|_1\).
The first form of the simulation lemma gives
\[
\begin{aligned}
&(1-\gamma)
\left(
V^\pi_{\widehat P_k,g-\widehat b_k}(s_1)
-V^\pi_{P^\star,g}(s_1)
\right)\\
&\quad=
\mathbb E_{d_{\widehat P_k}^{\pi}}
\left[
-\widehat b_k(s,a)
+\gamma
\left(
\mathbb E_{\widehat P_k}
-\mathbb E_{P^\star}
\right)
V^\pi_{P^\star,g}(s')
\right].
\end{aligned}
\tag{G.4}
\]
The normalized utility assumption implies
\(\|V^\pi_{P^\star,g}\|_\infty\le1\).  Applying
Lemma~\ref{lem:mix_useful2} with \(g=f_k\), using
\(\mathfrak C_k^{\rm init}(\pi),
\mathfrak C_k^{\rm tri,next}(\pi)\le\Gamma_k\), and choosing
\(\alpha_k\) to absorb the covariance-dependent term yields
\[
V^\pi_{\widehat P_k,g-\widehat b_k}(s_1)
-V^\pi_{P^\star,g}(s_1)
\le
C_{\rm cert}
\sqrt{\frac{\Gamma_k\zeta_k}{1-\gamma}}
=\ell_k.
\tag{G.5}
\]
The proof is deterministic on \(\mathcal E_{\rm rep}\), so no additional
union bound over the adaptively queried policies is required.  Since the true
utility is nonnegative, (G.5) proves the first inequality in (G.3) after
positive-part clipping.

Repeating the direct-optimism argument of
Lemma~\ref{lem:target_mix_optimism} for the normalized utility signal \(g\)
gives
\[
V^\pi_{P^\star,g}(s_1)
\le
V^\pi_{\widehat P_k,g+\widehat b_k}(s_1)+\ell_k.
\tag{G.6}
\]
For the same model and policy, discounted value is linear in the signal:
\[
V^\pi_{\widehat P_k,g+\widehat b_k}(s_1)
-V^\pi_{\widehat P_k,g-\widehat b_k}(s_1)
=
2V^\pi_{\widehat P_k,\widehat b_k}(s_1).
\tag{G.7}
\]
Combining (G.6), (G.7), and the raw lower endpoint
\(V^\pi_{\widehat P_k,g-\widehat b_k}-\ell_k\) proves the upper inequality
with \(w_k(\pi)=2\ell_k+
2V^\pi_{\widehat P_k,\widehat b_k}\).  Replacing a negative raw endpoint by
zero only tightens this upper bound because
\(V^\pi_{P^\star,g}\ge0\).
\end{proof}

For the selected policy sequence, define
\[
\mathcal B_K
:=
\widetilde O\!\left(
\frac{\bar b d^2}{(1-\gamma)^2}
\left[
\sqrt{\Gamma_K\sum_{k=2}^K\Lambda_k}
+\sum_{k=1}^K\frac{\Gamma_k}{\sqrt k}
\right]
\right),
\tag{G.8}
\]
where
\[
\Lambda_k
:=
\frac{1}{1-\epsilon_k+\epsilon_k/|\mathcal A|}.
\]

\begin{lemma}[Selected-policy width summability]
\label{lem:hard_selected_width}
On \(\mathcal E_{\rm rep}\), the policy sequence returned by
Algorithm~\ref{alg:rep-pd-hard-teg} satisfies
\[
\sum_{k=1}^K w_k(\pi_k)
\le
\mathcal B_K+2\sum_{k=1}^K\ell_k.
\tag{G.9}
\]
\end{lemma}

\begin{proof}
Condition on the history before episode \(k\).  Apply the second form of the
simulation lemma to the signal \(\widehat b_k\), with no additional bonus:
\[
\begin{aligned}
V_{\widehat P_k,\widehat b_k}^{\pi_k}(s_1)
-V_{P^\star,\widehat b_k}^{\pi_k}(s_1)
=
\frac{\gamma}{1-\gamma}
\mathbb E_{d_{P^\star}^{\pi_k}}
\left[
\left(
\mathbb E_{\widehat P_k}-\mathbb E_{P^\star}
\right)
V_{\widehat P_k,\widehat b_k}^{\pi_k}(s')
\right].
\end{aligned}
\]
Because \(0\le\widehat b_k\le\bar b\),
\(\|V_{\widehat P_k,\widehat b_k}^{\pi_k}\|_\infty
\le\bar b/(1-\gamma)\), and hence
\[
V_{\widehat P_k,\widehat b_k}^{\pi_k}(s_1)
\le
\frac{1}{1-\gamma}
\mathbb E_{d_{P^\star}^{\pi_k}}[\widehat b_k]
+
\frac{\gamma\bar b}{(1-\gamma)^2}
\mathbb E_{d_{P^\star}^{\pi_k}}[f_k].
\tag{G.10}
\]
The two expectations in (G.10) are exactly the bonus and model-error terms
bounded in the proof of Lemma~\ref{lem:gen_reg_and_vio}; their cumulative
bound follows from Lemma~\ref{lem:gen_potential} and is \(\mathcal B_K\) up
to universal constants.  Summing (G.2) and using (G.10) proves (G.9).
Importantly, this step is applied only to the returned sequence
\(\{\pi_k\}\), for which the deployed-to-probe coupling used by
Lemma~\ref{lem:gen_potential} is valid.
\end{proof}

For a polynomial schedule \(\epsilon_k=k^{-p}\),
\[
\sum_{k=1}^K\ell_k
\le
\widetilde O\!\left(
\sqrt{\frac{|\mathcal A|}{1-\gamma}}\,
K^{(1+p)/2}
\right),
\tag{G.11}
\]
while
\[
\mathcal B_K
\le
\widetilde O\!\left(
\frac{|\mathcal A|\bar b d^2}{(1-\gamma)^2}
K^{1/2+p}
\right).
\tag{G.12}
\]
Thus the correction term is no larger in \(K\) than the selected-width term.

\subsection{Hard Front End}

For the selected dual level \(Y_k\), define the Lagrangian front end
\[
\mathfrak F_K
:=
\sum_{k=1}^K
\left[
V_{P^\star,r}^{\pi^\star}(s_1)
-V_{P^\star,r}^{\pi_k}(s_1)
+Y_k
\left(
V_{P^\star,g}^{\pi^\star}(s_1)
-V_{P^\star,g}^{\pi_k}(s_1)
\right)
\right].
\tag{G.13}
\]
Let \(\sigma_k^{\rm sm}\) be the value suboptimality introduced by replacing
the exact optimistic planner at the selected dual level with its statewise
softmax candidate.  Under Assumption~\ref{ass:action_gap},
\[
\sum_{k=1}^K\sigma_k^{\rm sm}
\le
\widetilde O\!\left(
\frac{K|\mathcal A|}{(1-\gamma)^2}
\exp(-\delta_{\min}/\tau_h)
\right).
\tag{G.14}
\]

\begin{theorem}[Decaying hard front end]
\label{thm:hard_general_probing}
Assume \(\mathcal E_{\rm rep}\), the nonbinding bonus cap, and
Assumption~\ref{ass:action_gap}.  With the same
\(\alpha_k,\lambda_k\) as in the soft analysis and \(0\le Y_k\le Y_{\max}\),
\[
\mathfrak F_K
\le
\widetilde O\!\left(
\frac{\bar b d^2Y_{\max}}{(1-\gamma)^2}
\left[
\sqrt{\Gamma_K\sum_{k=2}^K\Lambda_k}
+\sum_{k=1}^K\frac{\Gamma_k}{\sqrt k}
\right]
\right)
+\sum_{k=1}^K\sigma_k^{\rm sm}.
\tag{G.15}
\]
\end{theorem}

\begin{proof}
At each selected dual level, the optimistic Lagrangian comparison is the same
as in Lemma~\ref{lem:gen_reg_and_vio}, with the utility-side statistical
terms multiplied by \(Y_k\le Y_{\max}\).  The reward-side terms are unchanged.
The returned policy is a softmax approximation of the optimistic planner, so
its value loss contributes \(\sigma_k^{\rm sm}\).  Applying
Lemma~\ref{lem:gen_potential} to the returned sequence and summing gives
(G.15).  No certificate width is used in this front-end argument.
\end{proof}

\subsection{Selected-Certificate Reduction}

Let \(\mathcal K_{\rm br}\) contain episodes that stop at a certified grid
point, and let \(\mathcal K_{\rm cap}\) contain episodes that reach
\(Y_{\max}\) without certification.  Define the certificate grid sensitivity
\[
s_{\rm grid}
:=
C_{\rm grid}
\frac{\eta_h}{\tau_h(1-\gamma)},
\tag{G.16}
\]
where \(C_{\rm grid}\) is the universal constant from the softmax-sensitivity
Lemma~\ref{lem:eta_diff}.  The positive-part clipping is \(1\)-Lipschitz, so
neighboring certificates differ by at most \(s_{\rm grid}\).

\begin{proposition}[Selected-certificate regret reduction]
\label{prop:hard_general_cert_reduction}
On \(\mathcal E_{\rm rep}\),
\[
\mathrm{Regret}^{\rm return}(K)
\le
\mathfrak F_K
+Y_{\max}
\left(
\sum_{k=1}^K w_k(\pi_k)
+Ks_{\rm grid}
\right),
\tag{G.17}
\]
and
\[
\mathrm{Violation}_{\rm hard}^{\rm return}(K)
\le
\sum_{k\in\mathcal K_{\rm cap}}
\left[
b-V_{P^\star,g}^{\pi_k}(s_1)
\right]_+.
\tag{G.18}
\]
\end{proposition}

\begin{proof}
If \(k\in\mathcal K_{\rm br}\), then
\(\underline{\mathcal U}_k(\pi_k)\ge b\), so
Lemma~\ref{lem:hard_derived_certificate} gives
\(V_{P^\star,g}^{\pi_k}(s_1)\ge b\).  Break episodes therefore contribute
zero hard violation, proving (G.18).

To convert the Lagrangian front end into reward regret, use
\(V_{P^\star,g}^{\pi^\star}(s_1)\ge b\):
\[
\mathrm{Regret}^{\rm return}(K)
\le
\mathfrak F_K+
\sum_{k=1}^K
Y_k
\left[
V_{P^\star,g}^{\pi_k}(s_1)-b
\right]_+.
\tag{G.19}
\]
If a break occurs at \(Y_k>0\), the preceding grid point failed the
certificate.  Grid sensitivity and Lemma~\ref{lem:hard_derived_certificate}
therefore imply
\[
\left[
V_{P^\star,g}^{\pi_k}(s_1)-b
\right]_+
\le
s_{\rm grid}+w_k(\pi_k).
\]
When \(Y_k=0\), the corresponding term in (G.19) is zero.  On a cap episode,
the certificate at the explicitly queried cap failed, so
\[
\left[
V_{P^\star,g}^{\pi_k}(s_1)-b
\right]_+
\le w_k(\pi_k).
\]
Using \(Y_k\le Y_{\max}\) proves (G.17).
\end{proof}

\subsection{Hard Violation at the Executed Cap Policy}

The remaining argument is performed at the cap policy that is actually
returned and subsequently enters the probing kernel; it does not use an
unexecuted fixed-\(Y_\theta\) candidate.

For \(c\in\{r,g\}\), define the selected-policy optimistic
overestimation
\[
o_{k,c}
:=
\left[
V_{\widehat P_k,c+\widehat b_k}^{\pi_k}(s_1)
-V_{P^\star,c}^{\pi_k}(s_1)
\right]_+.
\tag{G.20}
\]
The proof of Lemma~\ref{lem:tri_overestimation} applies to the returned
sequence of the hard algorithm and to either normalized signal \(r\) or \(g\);
hence
\[
\sum_{k=1}^Ko_{k,r}
\le\mathcal B_K,
\qquad
\sum_{k=1}^Ko_{k,g}
\le\mathcal B_K.
\tag{G.21}
\]

\begin{lemma}[Cap-level Slater pressure]
\label{lem:hard_cap_slater}
Let \(\bar\pi\) satisfy the Slater condition
\(V_{P^\star,g}^{\bar\pi}(s_1)\ge b+\theta\).
For any \(k\in\mathcal K_{\rm cap}\), if
\(Y_{\max}\theta\ge1\), then
\[
\left[
b-V_{P^\star,g}^{\pi_k}(s_1)
\right]_+
\le
\frac{o_{k,r}}{Y_{\max}}
+o_{k,g}
+\left(1+\frac1{Y_{\max}}\right)\ell_k
+\frac{\sigma_k^{\rm sm}}{Y_{\max}}.
\tag{G.22}
\]
\end{lemma}

\begin{proof}
At the cap, softmax planning gives
\[
\begin{aligned}
&V_{\widehat P_k,r+\widehat b_k}^{\pi_k}(s_1)
+Y_{\max}
V_{\widehat P_k,g+\widehat b_k}^{\pi_k}(s_1)\\
&\quad\ge
V_{\widehat P_k,r+\widehat b_k}^{\bar\pi}(s_1)
+Y_{\max}
V_{\widehat P_k,g+\widehat b_k}^{\bar\pi}(s_1)
-\sigma_k^{\rm sm}.
\end{aligned}
\tag{G.23}
\]
Direct optimism, with
\(\mathfrak C_k^{\rm init}(\bar\pi)\le\Gamma_k\), yields
\[
V_{\widehat P_k,c+\widehat b_k}^{\bar\pi}(s_1)
\ge
V_{P^\star,c}^{\bar\pi}(s_1)-\ell_k
\qquad(c\in\{r,g\}).
\tag{G.24}
\]
The definition of \(o_{k,c}\) upper-bounds the left-hand optimistic values
in (G.23) by the corresponding true values plus \(o_{k,c}\).
Consequently,
\[
\begin{aligned}
&V_{P^\star,r}^{\pi_k}(s_1)
+Y_{\max}V_{P^\star,g}^{\pi_k}(s_1)
+o_{k,r}+Y_{\max}o_{k,g}\\
&\quad\ge
V_{P^\star,r}^{\bar\pi}(s_1)
+Y_{\max}V_{P^\star,g}^{\bar\pi}(s_1)
-(1+Y_{\max})\ell_k-\sigma_k^{\rm sm}.
\end{aligned}
\tag{G.25}
\]
If
\(v_k=[b-V_{P^\star,g}^{\pi_k}(s_1)]_+>0\), the Slater condition and the
normalization
\(|V_{P^\star,r}^{\pi_k}-V_{P^\star,r}^{\bar\pi}|\le1\) imply
\[
Y_{\max}(\theta+v_k)
\le
1+o_{k,r}+Y_{\max}o_{k,g}
+(1+Y_{\max})\ell_k+\sigma_k^{\rm sm}.
\]
Since \(Y_{\max}\theta\ge1\), subtracting this term and dividing by
\(Y_{\max}\) proves (G.22).  The claim is trivial when \(v_k=0\).
\end{proof}

\begin{proposition}[Decaying Slater-induced hard violation]
\label{prop:hard_general_slater_violation}
Assume the hypotheses of
Proposition~\ref{prop:hard_general_cert_reduction}.  If
\(Y_{\max}\theta\ge1\), then
\[
\mathrm{Violation}_{\rm hard}^{\rm return}(K)
\le
\left(1+\frac1{Y_{\max}}\right)
\left(
2\mathcal B_K+\sum_{k=1}^K\ell_k
+\sum_{k=1}^K\sigma_k^{\rm sm}
\right)
\tag{G.26}
\]
up to universal constants.  If \(Y_{\max}\theta<1\), the trivial bound
\(\mathrm{Violation}_{\rm hard}^{\rm return}(K)\le K
\le\theta^{-4}\) applies.
\end{proposition}

\begin{proof}
Combine (G.18) with Lemma~\ref{lem:hard_cap_slater}, sum over cap episodes,
and use (G.21).  Extending each nonnegative sum from
\(\mathcal K_{\rm cap}\) to all episodes gives (G.26).  The second statement
follows from \(Y_{\max}=K^{1/4}\).
\end{proof}

\subsection{Polynomial Schedules}

\begin{corollary}[Polynomial hard rates]
\label{cor:hard_general_probing_poly}
Let \(\epsilon_k=k^{-p}\) with \(p\in[0,1/2)\), and choose
\[
\begin{aligned}
\alpha_k
&=
\widetilde\Theta\!\left(
\sqrt{(\Gamma_k+d^2)\gamma\log(|\mathcal M|k/\delta)}
\right)
\quad\text{with a sufficiently large leading constant},\\
\lambda_k
&=
\widetilde\Theta\!\left(
\frac{\Gamma_k+d^2}{d}
\log(|\mathcal M|k/\delta)
\right).
\end{aligned}
\]
\[
\tau_h
=
\Theta\!\left(
\frac{1}{(1-\gamma)\sqrt K\log|\mathcal A|}
\right),
\qquad
\eta_h
=
\Theta\!\left(
\frac{\tau_h(1-\gamma)}{K^{3/4}}
\right).
\]
Under the assumptions of Theorem~\ref{thm:hard_general_probing} and
Proposition~\ref{prop:hard_general_slater_violation}, with probability at
least \(1-\delta\),
\[
\mathrm{Regret}^{\rm return}(K)
\le
\widetilde O\!\left(
\left[
\frac{|\mathcal A|\bar b d^2}{(1-\gamma)^2}
+\sqrt{\frac{|\mathcal A|}{1-\gamma}}
\right]
K^{3/4+p}
\right),
\tag{G.27}
\]
and
\[
\mathrm{Violation}_{\rm hard}^{\rm return}(K)
\le
\widetilde O\!\left(
\frac1\theta
\left[
\frac{|\mathcal A|\bar b d^2}{(1-\gamma)^2}
+\sqrt{\frac{|\mathcal A|}{1-\gamma}}
\right]
K^{1/2+p}
\right)
+O(\theta^{-4}).
\tag{G.28}
\]
\end{corollary}

\begin{proof}
For polynomial schedules,
\[
\sqrt{\Gamma_K\sum_{k=2}^K\Lambda_k}
+\sum_{k=1}^K\frac{\Gamma_k}{\sqrt k}
\le
\widetilde O(|\mathcal A|K^{1/2+p}).
\]
Equations (G.11)--(G.12), Theorem~\ref{thm:hard_general_probing},
Proposition~\ref{prop:hard_general_cert_reduction}, and
\(Y_{\max}=K^{1/4}\) give (G.27); the grid contribution is
\[
Y_{\max}Ks_{\rm grid}
=
\widetilde O(K^{1/2}).
\]
Proposition~\ref{prop:hard_general_slater_violation} gives (G.28).
We display the harmless factor \(\theta^{-1}\) to match the theorem's
problem-dependent Slater convention; for normalized utility
\(\theta\le1\).  The exponential softmax residual under the stated
temperature is lower order.
\end{proof}

%% file: appendix/hard_tertiary_roadmap.tex
\section{The \texorpdfstring{\(p=0\)}{p=0} Hard Specialization}
\label{app:hard_tertiary_roadmap}

This appendix specializes the derived-certificate proof in
Appendix~\ref{app:hard_general_probing} to uniform tertiary probing.  It is
included separately to make the clean return-policy rate transparent.

When \(p=0\), \(\epsilon_k\equiv1\), and hence
\[
q_k=q'_k=q''_k=U(\mathcal A),
\qquad
\Gamma_k=|\mathcal A|,
\qquad
\Lambda_k=|\mathcal A|.
\]
The algorithmic lower correction becomes
\[
\ell_k
=
C_{\rm cert}
\sqrt{
\frac{
|\mathcal A|\ln(|\mathcal M|k/\delta)
}{
k(1-\gamma)
}
},
\tag{H.1}
\]
and Lemma~\ref{lem:hard_derived_certificate} gives
\[
\underline{\mathcal U}_k(\pi)
\le
V_{P^\star,g}^{\pi}(s_1)
\le
\underline{\mathcal U}_k(\pi)+w_k(\pi)
\]
for every queried policy.  The algorithm uses only the first inequality; the
second is summed only for the returned sequence.

\begin{lemma}[\(p=0\) selected correction and width]
\label{lem:hard_pzero_width}
On the representation-learning good event,
\[
\sum_{k=1}^K\ell_k
\le
\widetilde O\!\left(
\sqrt{\frac{|\mathcal A|K}{1-\gamma}}
\right),
\]
and
\[
\sum_{k=1}^Kw_k(\pi_k)
\le
\widetilde O\!\left(
\left[
\frac{|\mathcal A|\bar b d^2}{(1-\gamma)^2}
+\sqrt{\frac{|\mathcal A|}{1-\gamma}}
\right]
\sqrt K
\right).
\tag{H.2}
\]
\end{lemma}

\begin{proof}
The first display follows by summing
\(\ell_k=\widetilde O(k^{-1/2})\).  The second follows from
Lemma~\ref{lem:hard_selected_width} with
\(\Gamma_k=\Lambda_k=|\mathcal A|\).
\end{proof}

\begin{proposition}[\(p=0\) certificate reduction]
\label{prop:hard_lower_cert_reduction}
Under the hypotheses of
Proposition~\ref{prop:hard_general_cert_reduction},
\[
\mathrm{Regret}^{\rm return}(K)
\le
\widetilde O\!\left(
\left[
\frac{|\mathcal A|\bar b d^2}{(1-\gamma)^2}
+\sqrt{\frac{|\mathcal A|}{1-\gamma}}
\right]
K^{3/4}
\right),
\tag{H.3}
\]
and hard violation is supported only on cap episodes.
\end{proposition}

\begin{proof}
For \(p=0\), the hard front end in
Theorem~\ref{thm:hard_general_probing} is
\(\widetilde O(
|\mathcal A|\bar b d^2K^{3/4}/(1-\gamma)^2)\).
Apply Proposition~\ref{prop:hard_general_cert_reduction} and
Lemma~\ref{lem:hard_pzero_width}.  With
\(\eta_h=\Theta(\tau_h(1-\gamma)/K^{3/4})\), the grid term is
\[
Y_{\max}Ks_{\rm grid}
=
\widetilde O(K^{1/2}),
\]
and is lower order.
\end{proof}

\begin{theorem}[Slater-induced hard violation for \(p=0\)]
\label{thm:uniform_slater_hard_violation}
Suppose Assumptions~\ref{assum:slater}, \ref{assum:linearmdp}, and
\ref{ass:action_gap} hold,
\(|\mathcal M|<\infty,\ P^\star\in\mathcal M\), and the
representation-learning good event holds.  Then
Algorithm~\ref{alg:rep-pd-hard-teg} with \(\epsilon_k\equiv1\) satisfies
\[
\mathrm{Violation}_{\rm hard}^{\rm return}(K)
\le
\widetilde O\!\left(
\frac1\theta
\left[
\frac{|\mathcal A|\bar b d^2}{(1-\gamma)^2}
+\sqrt{\frac{|\mathcal A|}{1-\gamma}}
\right]
\sqrt K
\right)
+O(\theta^{-4}).
\tag{H.4}
\]
\end{theorem}

\begin{proof}
Break episodes are feasible by the one-sided inequality in
Lemma~\ref{lem:hard_derived_certificate}.  At a cap episode, the returned
policy is the actually queried \(Y_{\max}\) policy, so
Lemma~\ref{lem:hard_cap_slater} applies and its overestimation terms are
summable along the returned sequence.  Proposition~\ref{prop:hard_general_slater_violation}
and Lemma~\ref{lem:hard_pzero_width} give (H.4).  If
\(K^{1/4}\theta<1\), the trivial \(K\le\theta^{-4}\) bound accounts for the
last term.
\end{proof}

%% file: content/experiment.tex
\section{Experiments}
\label{sec:exp}

\subsection{Job scheduling.}
To validate the convergence of REP-PD-TEG in CMDPs, we first evaluate it in a simulated job scheduling environment. This toy domain lets us visualize how the learned representation, optimistic bonus, and primal-dual update interact under unknown dynamics.

Following \cite{ghosh2022provably,ghosh2024towards}, we design a constrained MDP with state space $\mathcal{S}=\{0,1,\dots,9\}$, where $s$ represents the number of pending jobs. The agent chooses either to withhold jobs ($a=0$) or submit them ($a=1$) over a horizon $H=10$. The transition dynamics follow:
\begin{align*}
    s_{h + 1}=\begin{cases}
        \max\{s_h-2a_h,0\}&\text{with probability }0.8\\
        \max\{s_h-a_h,0\}&\text{with probability }0.1\\
        s_h&\text{otherwise.}
    \end{cases}
\end{align*}
We assume that at time steps from 3 to 6, the reward is $1-0.9a$, while at all other time steps the reward is $1-0.2a$. This mimics a system in which job processing becomes temporarily expensive. The utility function $g(s,a,s^\prime)=\frac{s-s^\prime}{2}$ quantifies job-processing efficiency, with per-episode constraint $\sum_{h=0}^{H-1}g(s_h,a_h,s_{h+1})\geq b=2$.

We do not hand-craft transition features in this domain. Instead, the appendix implementation learns latent transition factors $\hat\phi(s,a)$ and $\hat\mu(s')$ directly from the collected transitions using one-hot encodings of the discrete states and actions. After each episode, the newly collected transitions are appended to the dataset, the representation learner is updated on the accumulated data for multiple gradient steps, and the primal-dual policy update is then executed with the induced exploration bonus.

To demonstrate the effectiveness of our constraint-aware design, we compare against REP-UCB \cite{uehara2021representation}, an unconstrained low-rank MDP approach using similar representation learning but without dual-variable updates, and a linear method \cite{ghosh2022provably} where the transition feature is known. As shown in Figure \ref{fig:plot_results}, REP-PD-TEG has higher regret than REP-UCB on this toy instance but substantially lower cumulative constraint violation, illustrating the tradeoff induced by constraint-aware learning. Despite not knowing the transition feature, REP-PD-TEG is comparable in scale to the known-feature linear method of \citet{ghosh2022provably} on both metrics. Because this illustrative plot reports a single run without uncertainty bands, we do not infer statistical significance or an empirical asymptotic rate from it.

\begin{figure*}[t]
    \centering
    \includegraphics[width=0.8\textwidth]{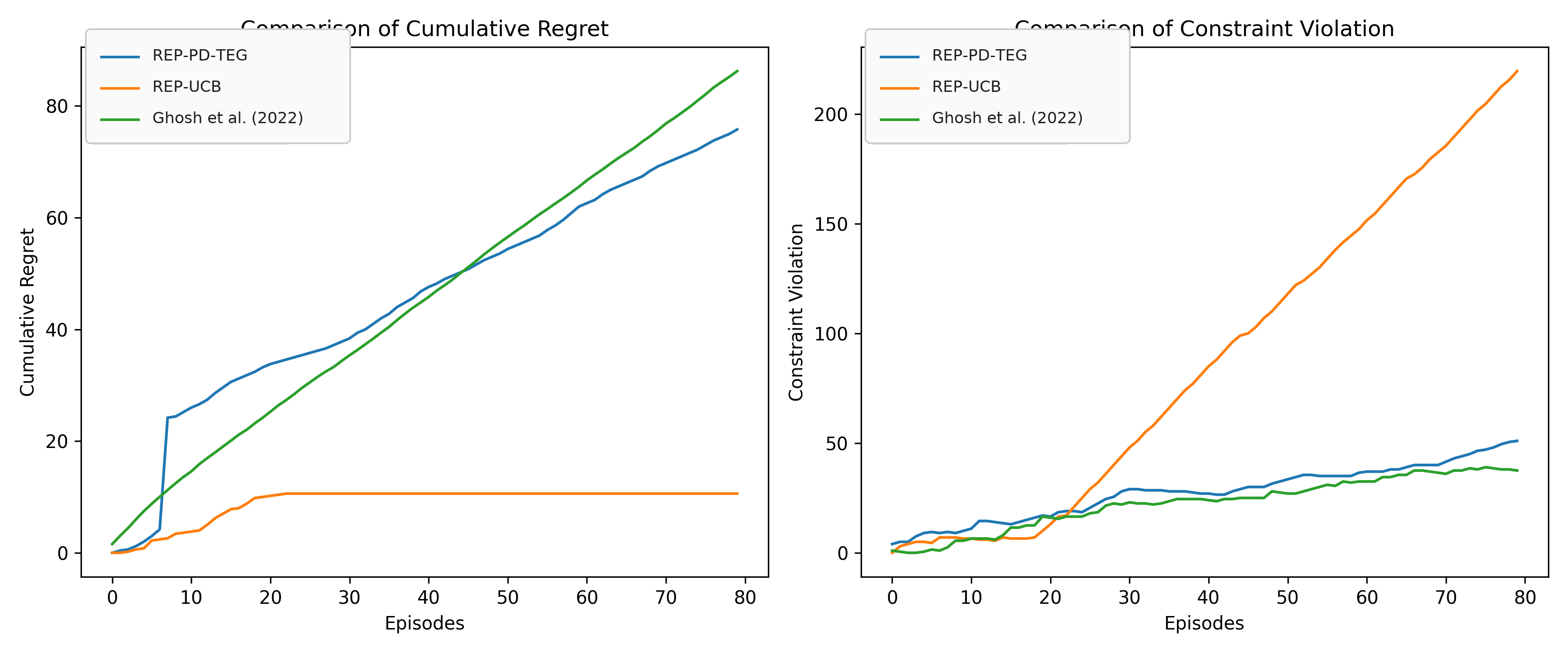}
    \caption{Illustrative single-run comparison of REP-PD-TEG, REP-UCB, and the known-feature linear method of \citet{ghosh2022provably} on the job-scheduling CMDP.}
    \label{fig:plot_results}
\end{figure*}

\subsection{Practical Neural Implementation}
\label{app:rep-pd-nn}

\paragraph{Neural Safety Gym implementation.}
To probe scalability beyond the toy domain, we instantiate a practical neural
variant, denoted \textbf{REP-PD-NN} for neural REP-PD-TEG. The theoretical
algorithm uses a per-round full-batch MLE representation-learning step, which
is difficult to scale directly to continuous-control domains. REP-PD-NN
therefore replaces this step with a mini-batch SPEDER-style stochastic
representation objective \cite{ren2022spectral}, while preserving the
primal-dual safe RL structure. The learned state-action representation is used
as the input to neural reward and cost critics, and the policy is trained with
a Lagrangian-style constrained policy update. We evaluate this variant on Safety Gym tasks
\cite{ray2019benchmarking}.

Let \(\phi_\eta(s,a)\) denote the learned state-action representation, and $J_{\mathrm{rep}}(\eta)$ denote the objective of representation learning. The actor is trained with the Lagrangian objective
\[
\mathcal J_{\pi}(\theta)
=
\mathbb E\!\left[
\alpha \log \pi_\theta(a|s)
-
\min_i Q^r_{\psi_i}(\phi_\eta(s,a))
+
\lambda \min_i Q^c_{\omega_i}(\phi_\eta(s,a))
\right],
\]
where \(\lambda\) is the Lagrange multiplier for the cost constraint.

\paragraph{Raw-state baseline.}
For the raw-state baseline, we use the same Lagrangian training skeleton
but remove the representation update. The critics and actor receive raw
state-action pairs \((s,a)\) instead of \(\phi_\eta(s,a)\). Thus, any
performance difference is attributable to the learned representation rather
than to a different safe RL backbone.

\begin{algorithm}[t]
\caption{REP-PD-NN}
\label{alg:rep-pd-nn}
\begin{algorithmic}[1]
\STATE {\bf Input:} CMDP environment \(\mathcal E\), cost budget \(d_{\mathrm{cost}}\), replay buffer \(\mathcal D\), warm-up steps \(W\), gradient steps per update \(G\), dual stepsize \(\rho\).
\STATE {\bf Parameters:} representation \(\eta\), actor \(\theta\), reward critics \(\{\psi_i\}_{i=1}^2\), cost critics \(\{\omega_i\}_{i=1}^2\), multiplier \(\lambda\).
\STATE Initialize target networks \(\bar\eta\leftarrow\eta\), \(\bar\psi_i\leftarrow\psi_i\), and \(\bar\omega_i\leftarrow\omega_i\).
\STATE Collect \(W\) transitions with a random or stochastic policy and store them in \(\mathcal D\).
\FOR{\(t=1,\ldots,T\)}
    \STATE Sample \(a_t\sim\pi_\theta(\cdot|s_t)\), execute \(a_t\) in \(\mathcal E\), and store \((s_t,a_t,r_t,c_t,s_{t+1},m_t)\) in \(\mathcal D\).
    \IF{an episode terminates with cumulative cost \(\widehat J_c\)}
        \STATE Update \(\lambda\leftarrow[\lambda+\rho(\widehat J_c-d_{\mathrm{cost}})]_+\).
    \ENDIF
    \FOR{\(g=1,\ldots,G\)}
        \STATE Sample a mini-batch \(\mathcal B\) from \(\mathcal D\).
        \STATE Update the representation network by minimizing \(\mathcal J_{\mathrm{rep}}(\eta)\) on \(\mathcal B\).
        \STATE Encode current pairs with \(\phi_\eta(s,a)\) and target pairs with \(\phi_{\bar\eta}(s',a')\).
        \STATE Update reward critics by minimizing the SAC reward Bellman loss.
        \STATE Update cost critics by minimizing the cost Bellman loss.
        \STATE Update the actor by minimizing \(\mathcal J_{\pi}(\theta)\) with the current multiplier \(\lambda\).
        \STATE Optionally update the entropy temperature \(\alpha\).
        \STATE Soft-update target networks \(\bar\eta\), \(\{\bar\psi_i\}_{i=1}^2\), and \(\{\bar\omega_i\}_{i=1}^2\).
    \ENDFOR
\ENDFOR
\STATE {\bf Return} policy \(\pi_\theta\) and representation \(\phi_\eta\).
\end{algorithmic}
\end{algorithm}

Table \ref{tab:safetygym_results} reports the completed positive results after $1$M interaction steps and $3$ random seeds, using the final deterministic evaluation checkpoint (\texttt{eval\_return} / \texttt{eval\_cost}). Across the three reported tasks, \textbf{REP-PD-NN} substantially improves final evaluation return and also reduces final evaluation cost relative to the raw baseline. These results indicate that a learned representation can improve both task performance and safety on a meaningful subset of continuous-control safe RL tasks.

\begin{table*}[t]
\centering
\small
\setlength{\tabcolsep}{4pt}
\caption{Safety Gym results for the practical neural implementation \textbf{REP-PD-NN}. Metrics are the mean and standard deviation over $3$ seeds at the final deterministic evaluation checkpoint after $1$M environment steps.}
\label{tab:safetygym_results}
\begin{tabular}{lcccc}
\toprule
\textbf{Environment} & \textbf{Raw Ret.} & \textbf{Raw Cost} & \textbf{REP-PD-NN Ret.} & \textbf{REP-PD-NN Cost} \\
\midrule
PointGoal1 & $-1.38 \pm 0.58$ & $27.85 \pm 18.99$ & $\mathbf{-0.14 \pm 0.03}$ & $\mathbf{16.45 \pm 12.69}$ \\
PointPush1 & $-1.40 \pm 0.60$ & $32.17 \pm 14.66$ & $\mathbf{-0.17 \pm 0.12}$ & $\mathbf{12.13 \pm 6.95}$ \\
PointPush2 & $-1.40 \pm 0.71$ & $13.00 \pm 8.17$ & $\mathbf{-0.03 \pm 0.01}$ & $\mathbf{10.90 \pm 3.85}$ \\
\bottomrule
\end{tabular}
\end{table*}

\subsection{Computational Resources}
\label{app:compute}

All experiments were run on a local institutional server equipped with eight
NVIDIA RTX 4090 GPUs, each with 24GB of memory, and standard x86\_64 CPUs.
The tabular job-scheduling experiment is lightweight and can be reproduced on
a CPU-only machine. Each Safety Gym training run uses one GPU; the eight-GPU
server was used only to run independent seeds and methods in parallel, and is
not required for reproducing any single run. The implementation uses Python
with PyTorch, CUDA-enabled GPU execution, MuJoCo, and Safety Gym.

Table~\ref{tab:compute_resources} summarizes the approximate compute required
to reproduce the reported experiments. For Safety Gym, each environment is
evaluated with two methods, \texttt{raw} and \texttt{REP-PD-NN}, and three
random seeds, resulting in six runs per environment and eighteen final
training runs in total.

\begin{table}[t]
\centering
\small
\setlength{\tabcolsep}{4pt}
\caption{Approximate compute required to reproduce the reported experiments.
Runtime denotes wall-clock time for one run. Total compute is over the reported
methods and seeds.}
\label{tab:compute_resources}
\begin{tabular}{lccccc}
\toprule
\textbf{Experiment} & \textbf{Worker} & \textbf{Steps} & \textbf{Runs} &
\textbf{Time/run} & \textbf{Total compute} \\
\midrule
Job scheduling & CPU ($<1$GB RAM) & -- & 1 & $<1$ min & $<1$ CPU-hour \\
Safety Gym PointGoal1 & 1 RTX 4090 (24GB) & $10^6$ & 6 & 8--20 h & 48--120 GPU-hours \\
Safety Gym PointPush1 & 1 RTX 4090 (24GB) & $10^6$ & 6 & 8--20 h & 48--120 GPU-hours \\
Safety Gym PointPush2 & 1 RTX 4090 (24GB) & $10^6$ & 6 & 8--20 h & 48--120 GPU-hours \\
\bottomrule
\end{tabular}
\end{table}

Thus, the reported Safety Gym results require approximately 144--360
GPU-hours in total. When all eight GPUs are used in parallel, the final runs
can be completed in roughly 18--45 wall-clock hours. During method
development, we also ran preliminary hyperparameter sweeps and diagnostic
experiments that are not required to reproduce the reported tables. These
preliminary runs used additional GPU time of the same order as the final
reported experiment set; failed or preliminary runs were excluded from the
reported results.